\documentclass[final,onefignum,onetabnum]{siamonline220329}

\usepackage{amssymb}
\usepackage{amsfonts}
\usepackage{mathrsfs}
\usepackage{amsmath}
\usepackage{color}
\definecolor{blue}{rgb}{0.0, 0.0, 1.0}

\definecolor{jgGreen}{rgb}{0.0, 0.5, 0.0}

\definecolor{pink}{rgb}{0.5, 0.0, 0.25}

\newcommand{\ignore}[1]{}

\definecolor{darkorange}{RGB}{255, 140, 0}

\newcommand{\cprm}{R^\e}
\newcommand{\cdl}{\bar R}

\newcommand{\cB}{{\mathcal B}}

\newcommand{\cL}{\mathcal{L}}
\newcommand{\cM}{\mathcal{M}}

\newcommand{\cS}{\mathcal{S}}

\newcommand{\sU}{{\mathscr U}}

\newcommand{\ba}{{\mathbf a}}
\newcommand{\bb}{{\mathbf b}}

\newcommand{\bh}{{\mathbf h}}

\newcommand{\bu}{{\mathbf u}}

\newcommand{\bx}{{\mathbf x}}
\newcommand{\by}{{\mathbf y}}
\newcommand{\bz}{{\mathbf z}}

\DeclareMathOperator*{\argmin}{argmin}

\DeclareMathOperator*{\esssup}{ess\,sup}

\DeclareMathOperator{\interior}{int}
\DeclareMathOperator{\supp}{supp}

\DeclareMathOperator{\comp}{comp}

\def\Rset{\mathbb{R}}
\def\Nset{\mathbb{N}}

\newcommand{\PP}{{\mathbb P}}
\newcommand{\QQ}{{\mathbb Q}}

\newcommand{\one}{{\mathbf{1}}}
\newcommand{\zero}{{\mathbf{0}}}
\newcommand{\Wball}[1]{{\cB^\infty_{#1}}}
\newcommand{\UWball}[1]{{\wtd \cB^\infty_{#1}}}

\newcommand{\ov}{\overline}

\newcommand{\td}{\tilde}
\newcommand{\wtd}{\widetilde}
\newcommand{\e}{\epsilon}

\newcommand{\sidebysidesubequations}[3]{%

    \begin{subequations}\label{#1}

        \begin{minipage}{0.5\linewidth}
         \begin{equation}
              #2 \vphantom{#3}  \label{#1_a}
        \end{equation}
         \end{minipage}
         \begin{minipage}{0.5\linewidth}
        \begin{equation}
              #3 \vphantom{#2} \label{#1_b}
         \end{equation}
         \end{minipage}
    \end{subequations}

 }

\newsiamthm{assumption}{Assumption}
\newsiamthm{informaltheorem}{Informal Theorem}
\newsiamremark{example}{Example}

\usepackage{lipsum}
\usepackage{amsfonts}
\usepackage{graphicx}
\usepackage{epstopdf}
\usepackage{algorithmic}
\ifpdf
  \DeclareGraphicsExtensions{.eps,.pdf,.png,.jpg}
\else
  \DeclareGraphicsExtensions{.eps}
\fi

\usepackage{enumitem}
\setlist[enumerate]{leftmargin=.5in}
\setlist[itemize]{leftmargin=.5in}

\newsiamremark{remark}{Remark}
\newsiamremark{hypothesis}{Hypothesis}
\crefname{hypothesis}{Hypothesis}{Hypotheses}
\newsiamthm{claim}{Claim}

\headers{Uniqueness for the Adversarial Bayes Classifier}{N.S. Frank}

\title{A Notion of Uniqueness for the Adversarial Bayes Classifier\thanks{Submitted to the editors DATE.
\funding{Natalie Frank was supported in part by the Research Training Group in Modeling and Simulation
funded by the National Science Foundation via grant RTG/DMS – 1646339 and grants DMS-2210583, CCF-1535987, IIS-1618662.}}}

\author{Natalie S. Frank\thanks{Courant Institute, New York, NY 
  (\email{nf1066@nyu.edu}, \url{https://natalie-frank.github.io/}).}}

\usepackage{amsopn}

\usepackage{ifthen}

\newboolean{appendixmode}
\setboolean{appendixmode}{true} 
\newcommand{\tocbreak}{\protect\\\hspace*{2em}} 

\ifpdf
\hypersetup{
  pdftitle={A Notion of Uniqueness for the Adversarial Bayes Classifier},
  pdfauthor={N. S. Frank}
}
\fi

\usepackage{multicol}
\usepackage{tabularx}

\usepackage{caption}
\usepackage{subcaption}

\begin{document}

\maketitle
\begin{abstract}
We propose a new notion of uniqueness for the adversarial Bayes classifier in the setting of binary classification. Analyzing this concept produces a simple procedure for computing all adversarial Bayes classifiers for a well-motivated family of one dimensional data distributions. This characterization is then leveraged to show that as the perturbation radius increases, certain notions of regularity for the adversarial Bayes classifiers improve. Furthermore, these results provide tools for understanding relationships between the Bayes and adversarial Bayes classifiers in one dimension.
\end{abstract}




\section{Introduction}
A crucial reliability concern for machine learning models is their susceptibility to adversarial attacks. Neural nets are particularly sensitive to small perturbations to data. For instance, \cite{biggio2013evasion, szegedy2013intriguing} show that perturbations imperceptible to the human eye can cause a neural net to misclassify an image. In order to reduce the susceptibility of neural nets to such attacks, several methods have been proposed to minimize the \emph{adversarial classification risk}, which incurs a penalty when a data point can be perturbed into the opposite class. However, state-of-the-art methods for minimizing this risk still achieve significantly lower accuracy than standard neural net training on simple datasets, even for small perturbations. For example, on the CIFAR10 dataset, \cite{peng2023robust} achieves 71\% robust accuracy for $\ell_\infty$ perturbations size  $8/255$ while \cite{dosovitskiy2021image} achieves over 99\% accuracy without an adversary.

In the setting of standard (non-adversarial) classification, a \emph{Bayes classifier} is defined as a minimizer of the classification risk. This classifier simply predicts the most probable class at each point. When multiple classes are equally probably, the Bayes classifier may not be unique. The Bayes classifier has been a helpful tool in the development of machine learning classification algorithms \cite[Chapter~2.4]{HastieTibshiraniFriedman2009}. On the other hand, in the adversarial setting, computing minimizers of the adversarial classification risk in terms of the data distribution is a challenging problem. These minimizers are referred to as \emph{adversarial Bayes classifiers}. Prior work \cite{AikenBrewerMurray2022,BhagojiCullinaMittalji2019lower,PydiJog2019} calculates these classifiers by first proving a minimax principle relating the adversarial risk with a dual problem, and then showing that the adversarial risk of a proposed set matches the dual risk of an element of the dual space.

In this paper, we propose a new notion of `equivalence' and `uniqueness' for adversarial Bayes classifiers in the setting of binary classification under the evasion attack. In the non-adversarial setting, two classifiers are \emph{equivalent} if they are equal a.e. with respect to the data distribution, and one can show that any two equivalent classifiers have the same classification risk. The Bayes classifier is \emph{unique} if any two minimizers of the classification risk are equivalent. However, under this equivalence relation, two equivalent sets can have different adversarial classification risks. This discrepancy necessitates a new definition of equivalence for adversarial Bayes classifiers.
Further analyzing this new equivalence relation in one dimension results in a method for calculating all possible adversarial Bayes classifiers for a well-motivated family of distributions using optimality conditions from calculus. In contrast, prior work on calculating adversarial Bayes classifier involved verifying a strong duality principle from optimal transport \cite{PydiJog2019} (see \cref{sec:related_works} for a further discussion of related works). 

We apply our new techniques to show that certain regularity properties of adversarial Bayes classifiers improve as $\e$ increases. Subsequent examples demonstrate that different adversarial Bayes classifiers can achieve varying levels of standard classification risk. Together, these examples suggest that in certain cases, the accuracy-robustness tradeoff could be mitigated by a careful selection of an adversarial Bayes classifier (see \cite{ZhangYuetal19principled} for a further discussion of this phenomenon). Followup work \cite{Frank2024consistency} demonstrates that the concepts presented in this paper have algorithmic implications--- when the data distribution is absolutely continuous with respect to Lebesgue measure, adversarial training with a convex loss is adversarially consistent iff the adversarial Bayes classifier is unique, according to the new notion of uniqueness proposed in this paper. Hopefully, a better understanding of adversarial Bayes classifiers will aid the design of algorithms for robust classification.
\section{Background}\label{sec:Background}
\subsection{Bayes classifiers}\label{sec:background_bayes_classifiers}
We study binary classification on the space $\Rset^d$ with labels $\{-1,+1\}$. The measure $\PP_0$ describes the probability of data with label $-1$ occurring in regions of $\Rset^d$ while the measure $\PP_1$ describes the probability of data with label $+1$ occurring in regions of $\Rset^d$. Most of our results will assume that $\PP_0$ and $\PP_1$ are absolutely continuous with respect to the Lebesgue measure $\mu$.
Vectors in $\Rset^d$ will be denoted in boldface ($\bx$). Many of the results in this paper focus on the case $d=1$ for which we will use non-bold letters ($x$).  The functions $p_0$ and $p_1$ will denote the densities of $\PP_0$, $\PP_1$ respectively. A classifier is represented by a set $A$ corresponding to points labeled $+1$. The \emph{classification risk} of the set $A$ is then the proportion of incorrectly classified data:
\begin{equation}\label{eq:classification_risk_set}
R(A)=\int \one_{A^C} d\PP_1+\int \one_Ad\PP_0. 
\end{equation}

A minimizer of the classification risk is called a \emph{Bayes classifier}. Analytically finding the minimal classification risk and Bayes classifiers is a straightforward calculation: Let $\PP=\PP_0+\PP_1$, representing the total probability of a region, and let $\eta$ be the the Radon-Nikodym derivative $\eta=d\PP_1/d\PP$, the conditional probability of the label $+1$ at a point $\bx$. Thus one can re-write the classification risk as
\begin{equation}\label{eq:classification_risk_eta}
R(A)=\int C(\eta(\bx),\one_A(\bx)) d\PP(\bx).
\end{equation}
and the minimum classification risk as $\inf_f R(f)=\int C^*(\eta)d\PP$ with

\vspace{-15pt}\sidebysidesubequations{eq:C_def}{C(\eta,\alpha)=\eta \one_{\alpha \leq 0}+(1-\eta) \one_{\alpha >0}}{C^*(\eta)=\inf_{\alpha\in\{0,1\}} C(\eta,\alpha)=\min(\eta,1-\eta)}
    \vspace{5pt}

A set $B$ is a Bayes classifier iff $\one_B$ minimizes \cref{eq:C_def_a} pointwise a.e. Accordingly, a Bayes classifier $B$ must include $\{p_1(x)>p_0(x)\}$ and exclude $\{p_1(x)<p_0(x)\}$ up to $\PP$-null sets. Rephrased in terms of $\PP$ and $\eta$, any Bayes classifier must include $\{\eta(x)>1/2\}$ and exclude $\{1-\eta(x)<1/2\}$ up to measure zero sets. The set of points with $\eta(\bx)=1/2$ can be arbitrarily split between $B$ and $B^C$. Two Bayes classifiers are \emph{equivalent} if they split this ambiguous set in the same way: 
\begin{definition}\label{def:equivalence_Bayes}
    The Bayes classifiers $B_1$ and $B_2$ are \emph{equivalent} if $\one_{B_1}=\one_{B_2}$ $\PP$-a.e. 
\end{definition}
This notion defines an equivalence relation because equality $\PP$-a.e. satisfies the properties of an equivalence relation. 
The Bayes classifier is \emph{unique} if there is a single equivalence class.
\begin{proposition}\label{prop:Bayes_equivalencies}
    The following are equivalent:
    \begin{enumerate}[label=\Roman*)]
        \item\label{it:Bayes_equivalencies_unique} The Bayes classifier is unique
        \item \label{it:Bayes_equivalencies_values}Amongst all Bayes classifiers $B$, either the value of $\PP_0(B)$ is unique or the value of $\PP_1(B^C)$ is unique
        \item \label{it:Bayes_equivalencies_eta}$\PP(\eta=1/2)=0$
    \end{enumerate}
\end{proposition}
See \cref{app:Bayes_equivalencies} for a proof. This paper will develop the counterpart of this result for adversarial Bayes classifiers.
 Next, we focus on the simplified scenario of one dimensional classifiers and $p_0, p_1$ continuous. In particular, by the preceding discussion, the set $B=\{p_1(x)>p_0(x)\}$ is a Bayes classifier. This set is open can therefore be written as a countable union of disjoint open intervals:
 \begin{equation}\label{eq:Bayes_as_union}
    B=\bigcup_{i=n}^N (c_i,d_i),    
 \end{equation}
  where $-\infty\leq n\leq N-1\leq +\infty$. The option $n=N-1$ specifically accounts for the empty set.
 
As $p_0$ and $p_1$ are continuous, points in the boundary of this Bayes classifier must satisfy

\vspace{-15pt}\sidebysidesubequations{eq:Bayes_necessary}{p_0(c_i)=p_1(c_i)}{p_0(d_i)=p_1(d_i)}
    \vspace{5pt}


The points $c_i$ mark transitions from $\{p_1(x)>p_0(x)\}$ to $\{p_1(x)\leq p_0(x)\}$ while the points $d_i$ mark transitions from $\{p_1(x)\leq p_0(x)\}$ to $\{p_1(x)>p_0(x)\}$. Consequently, if $p_0$ and $p_1$ are differentiable at $c_i$ or $d_i$, these boundary points $c_i, d_i$ must also satisfy

\vspace{-15pt}\sidebysidesubequations{eq:Bayes_second_order_necessary}{p_1'(c_i)-p_0'(c_i)\geq 0}{p_0'(d_i)-p_1'(d_i)\geq 0}
    \vspace{5pt}
respectively.
This analysis suggests a procedure for finding all possible Bayes classifiers up to equivalence:
\begin{enumerate}[label=\roman*)]
      \item\label{it:Bayes_procedure_first} Let $\mathfrak{c}$, $\mathfrak{d}$ be the set of points that satisfy the necessary conditions \cref{eq:Bayes_necessary,eq:Bayes_second_order_necessary} for $c_i$, $d_i$ respectively
      \item Construct a Bayes classifier $B=\bigcup_{i=n}^N (c_i,d_i)$ with $c_i\in \mathfrak{c}$ and $d_i\in \mathfrak{d}$. 
      \item \label{it:Bayes_procedure_last} Find the set $S=\{x: p_1(x)=p_0(x)>0\}$. For any two disjoint $S_1,S_2$ with $S_1\sqcup S_2=S$, the set $B\cup S_1-S_2$ is also a Bayes classifier. The equivalence classes of Bayes classifiers are enumerated by subsets $S_1\subset S$ that differ by a set of positive $\PP$-measure.
\end{enumerate}
This paper will extend notions of equivalence and uniqueness for classifiers in the adversarial setting. Furthermore, our results adapt \cref{prop:Bayes_equivalencies}, \cref{eq:Bayes_necessary,eq:Bayes_second_order_necessary}, and the procedure \ref{it:Bayes_procedure_first}-\ref{it:Bayes_procedure_last} to the adversarial context.

\subsection{Adversarial Bayes classifiers}\label{sec:background_adv_bayes_classifiers}

In the adversarial scenario, an adversary tries to perturb the data point $\bx$ into the opposite class. We assume that perturbations are in a closed $\e$-ball $\ov{B_\e(\zero)}$ in some norm $\|\cdot\|$. The proportion of incorrectly classified data under an adversarial attack is the \emph{adversarial classification risk},\footnote{In order to define the adversarial classification risk, one must show that $S_\e(\one_A)$ is measurable for measurable $A$. A full discussion of this issue is delayed to \cref{sec:fundamental_regularity}.}
\begin{equation}\label{eq:adv_classification_risk_set}
R^\e(A)= \int S_\e(\one_{A^C}) d\PP_1 +\int S_\e(\one_{A})d\PP_0
\end{equation}
where the $S_\e$ operation on a function $g$ is defined according to
\begin{equation}\label{eq:S_e_def}
    S_\e(g)(\bx)=\sup_{\|\bh\|\leq \e} g(\bx+\bh).
\end{equation}

Under this model, a set $A$ incurs a penalty wherever $\bx\in A \oplus \ov{B_\e(\zero)}$, and thus we define the \emph{$\e$-expansion} of a set $A$ as 
\[A^\e=A\oplus \ov{B_\e(\zero)}.\]
The supremum of an indicator function equals the indicator function of an $\e$-expansion:
\begin{equation}\label{eq:indicator_supremum}
    S_\e(\one_A)=\one_{A^\e}.
\end{equation}
Hence the adversarial risk can also be written as
\[R^\e(A)= \int \one_{(A^C)^\e}d\PP_1+\int \one_{A^\e} d\PP_0.\]
Prior work has established the existence of minimizers to \cref{eq:adv_classification_risk_set}, referred to as \emph{adversarial Bayes classifiers} \cite{AwasthiFrankMohri2021,BhagojiCullinaMittalji2019lower,FrankNilesWeed23consistency,PydiJog2021}, see \cite[Theorem~1]{FrankNilesWeed23consistency} for an existence theorem that matches the setup of this paper. However, identifying minimizers of \cref{eq:adv_classification_risk_set} is substantially more difficult than in the standard classification problem, because unlike the standard classification problem, as the integrand cannot be minimized in a pointwise manner. 

In addition to existence, prior research has investigated the structure of the set of minimizers to $\cprm$. In particular, \cite{BungertGarciaMurray2021} establishes the following result:
        \begin{lemma}\label{lemma:union_intersection_adv_bayes}
        If $A_1,A_2$ are two adversarial Bayes classifiers, then so are $A_1\cup A_2$ and $A_2\cap A_1$.
    \end{lemma}
See \cref{app:union_intersection_adv_bayes} for a proof.

Next, we focus on classifiers in one dimension as this case is simple to analyze yet still yields non-trivial behavior.
    Prior work shows that when $\PP_0,\PP_1\ll \mu$ and $p_0,p_1$ are continuous, if the adversarial Bayes classifier is sufficiently `regular', one can find necessary conditions describing the boundary of the adversarial Bayes classifier \cite{trillosMurray2022}. Assume that an adversarial Bayes classifier $A$ can be expressed as a union of disjoint intervals $A=\bigcup_{i=m}^M (a_i,b_i)$ with $a_i<b_i<a_{i+1}$, where the $m,M,a_i$, and $b_i$ can be $\pm\infty$. Notice that one can arbitrarily include/exclude the endpoints $\{a_i\}$, $\{b_i\}$ without changing the value of the adversarial risk $R^\e$. If $b_i-a_i> 2\e$ and $a_{i+1}-b_i> 2\e$, the adversarial classification risk can then be expressed as:
    \begin{equation}\label{eq:risk_on_intervals}
        R^\e(A)= \cdots+\int_{b_{i-1}-\e}^{a_i+\e} p_1(x) dx+\int_{a_i-\e} ^{b_i+\e} p_0(x)dx+\int_{b_i-\e}^{a_{i+1}+\e} p_1(x)dx+\cdots    
    \end{equation}
   When the densities $p_0$ and $p_1$ are continuous, differentiating this expression in $a_i$ and $b_i$ produces necessary conditions 
    
    \vspace{-15pt}\sidebysidesubequations{eq:first_order_necessary}{p_1(a_i+\e)-p_0(a_i-\e)=0}{p_0(b_i+\e)-p_1(b_i-\e)=0}
    \vspace{5pt}
    which describe the boundary of the adversarial Bayes classifier. When $\e=0$, these equations reduce to the condition describing the boundary of the Bayes classifier in \cref{eq:Bayes_necessary}. Prior work shows that when $p_0,p_1$ are well-behaved, the necessary conditions hold for sufficiently small $\e$.

    \begin{theorem}[\cite{trillosMurray2022}]\label{thm:sufficient_prior_work}
        Assume that $p_0,p_1$ are $C^1$, the relation $p_0(x)=p_1(x)$ is satisfied at finitely many points $x\in \supp \PP$, and that at these points, $p_0'(x)\neq p_1'(x)$. Then for sufficiently small $\e$, there exists an adversarial Bayes classifier for which the $a_i$ and $b_i$ satisfy the necessary conditions \cref{eq:first_order_necessary}.            
        \end{theorem}
    For a proof, see the discussion of Equation (4.1) and Theorem~5.4 in \cite{trillosMurray2022}.

    Differentiating \cref{eq:risk_on_intervals} twice results in the second order necessary conditions
        \vspace{-15pt} 
    \sidebysidesubequations{eq:second_order_necessary}{p_1'(a_i+\e)-p_0'(a_i-\e)\geq 0}{p_0'(b_i+\e)-p_1'(b_i-\e)\geq 0}
    \vspace{5pt}
    When $\e=0$, these relations reduce to the conditions \cref{eq:Bayes_second_order_necessary} describing the boundary of the Bayes classifier.

    A central goal of this paper is producing necessary conditions analogous to \cref{eq:first_order_necessary,eq:second_order_necessary} that hold for all $\e$.

        \subsection{Minimax theorems for the adversarial classification risk}
            We analyze the properties of  adversarial Bayes classifiers
            by expressing the minimal $R^\e$ risk in a `pointwise' manner analogous to \cref{eq:classification_risk_eta}.  The Wasserstein-$\infty$ metric from optimal transport and the minimax theorems in \cite{FrankNilesWeed23consistency,PydiJog2021} are essential tools for expressing $R^\e$ in this manner.

            Informally, the measure $\QQ'$ is in the Wasserstein-$\infty$ ball of radius $\e$ around the positive finite measure $\QQ$ if one can produce the measure $\QQ'$ by moving points in $\Rset^d$ by at most $\e$ under the measure $\QQ$. Formally, the $W_\infty$ metric is defined in terms the set of couplings $\Pi(\QQ,\QQ')$ between two finite positive measures $\QQ,\QQ'$:
            \[\Pi(\QQ,\QQ')=\{\gamma \text{ positive measure on }\Rset^d\times \Rset^d: \gamma(A\times \Rset^d)=\QQ(A), \gamma(\Rset^d\times A)=\QQ'(A)\}.\]
            The Wasserstein-$\infty$ distance between two finite positive measures $\QQ'$ and $\QQ$ with $\QQ(\Rset^d)=\QQ'(\Rset^d)$ is then defined as 
            \begin{equation*}
                W_\infty(\QQ,\QQ')=\inf_{\gamma \in \Pi(\QQ,\QQ')} \esssup_{(\bx,\by)\sim \gamma} \|\bx-\by\|.    
            \end{equation*}
            
            The $W_\infty$ metric is in fact a metric on the space of measures, as it is a limit of the Wasserstein-$p$ metrics as $p\to \infty$, see \cite{ChampionDePascaleJuutinen07,Jylha15} for details. We denote the $\e$-ball in the $W_\infty$ metric around a measure $\QQ$ by
            \[\Wball \e(\QQ)=\{\QQ':\QQ' \text{ Borel}, W_\infty(\QQ,\QQ')\leq \e\}\]

 Prior work \cite{PydiJog2021,TrillosJacobsKim22} applies properties of the $W_\infty$ metric to find a dual problem to the minimization of $R^\e$: let $\PP_0',\PP_1'$ be finite Borel measures and define
        \begin{equation}\label{eq:cdl_def}
            \cdl(\PP_0',\PP_1')=\int C^*\left(\frac{d \PP_1'}{d(\PP_0'+\PP_1')}\right)d(\PP_0'+\PP_1')
        \end{equation}
        where $C^*$ is defined by \eqref{eq:C_def}. Prior results \cite{FrankNilesWeed23consistency,PydiJog2021} relate this risk to $R^\e$.
    \begin{theorem}\label{thm:minimax_classification} \label{thm:existence_adv_classification}
        Let $\cdl$ be defined by \eqref{eq:cdl_def}. Then
	        \begin{equation}\label{eq:minimax_classification}
	           \inf_{\substack{A\text{ Borel}}}\cprm(A)=\sup_{\substack{\PP_0'\in\Wball \e (\PP_0)\\ \PP_1'\in \Wball \e (\PP_1)}}\cdl(\PP_0',\PP_1')
	        \end{equation}
	        and furthermore equality is attained for some Borel measurable $A$ and 
         $\PP_1^*, \PP_0^*$ with $W_\infty(\PP_0^*,\PP_0)\leq \e$ and $W_\infty(\PP_1^*,\PP_1)\leq \e$.   
    \end{theorem}
        See Theorem~1 of \cite{FrankNilesWeed23consistency} for the above result.
       This minimax theorem then implies complementary slackness conditions that characterize optimal $A$ and $\PP_0^*,\PP_1^*$, see \cref{app:duality_proofs} for a proof.

        \begin{theorem}\label{thm:complementary_slackness_classification}
        The set $A$ is a minimizer of $R^\e$ and $(\PP_0^*,\PP_1^*)$ is a maximizer of $\bar R$ over $\Wball \e(\PP_0)\times \Wball \e(\PP_1)$ iff $W_\infty(\PP_0^*,\PP_0)\leq \e$, $W_\infty(\PP_1^*,\PP_1)\leq \e$, and 
    \begin{enumerate}[label=\arabic*)]\item\label{it:sup_assumed_precondition_classification}
        \begin{minipage}{\linewidth}
        \vspace{-10pt}
            \begin{equation}\label{eq:sup_comp_slack_classification}
            \int S_\e(\one_{A^C})d\PP_1=\int \one_{A^C} d\PP_1^*\quad \text{and} \quad \int S_\e(\one_{A}) d\PP_0=\int \one_{ A} d\PP_0^* 
        \end{equation}    
        \end{minipage}
        
        \item \label{it:comp_slack_equation_classification} If we define $\PP^*=\PP_0^*+\PP_1^*$ and $\eta^*=d\PP_1^*/d\PP^*$, then 
    \begin{equation}\label{eq:complementary_slackness_necessary_classification}
        \eta^*(\by)\one_{A^C}(\by)+(1-\eta^*(\by))\one_{A}(\by)=C^*(\eta^*(\by))\quad \PP^*\text{-a.e.}    
    \end{equation}
    \end{enumerate}

    \end{theorem}

\section{Main results}\label{sec:main_results}

\subsection*{Definitions}
As discussed in \cref{sec:background_bayes_classifiers}, a central goal of this paper is describing the regularity of adversarial Bayes classifiers and finding necessary conditions that hold for every $\e$ in one dimension. 

As an example of non-regularity, consider a data distribution defined by $p_0(x)=1/5$,  for $|x|\leq 1/4$ and $p_0(x)=0$ elsewhere; and $p_1(x)=3/5$ for $1\geq |x|> 1/4$ and $p_1(x)=0$ elsewhere (see \cref{fig:degenerate} for a depiction of $p_0$ and $p_1$). If $\e=1/8$, an adversarial Bayes classifier is $A=\Rset$. However, \emph{any} subset $S$ of $[-1/4+\e,1/4-\e]$ satisfies $R^\e(S^C)=R^\e(\Rset)$, and thus $S^C$ is an adversarial Bayes classifier as well. (These claims are rigorously justified in \cref{ex:degenerate}.) Consequently there are many adversarial Bayes classifiers lacking regularity, but they are all morally equivalent to the regular set $A=\Rset$. The notion of \emph{equivalence up to degeneracy} encapsulates this behavior.

\begin{definition}\label{def:equivalence_up_to_degeneracy}
    Two adversarial Bayes classifiers $A_1$ and $A_2$ are \emph{equivalent up to degeneracy} if for any Borel set $E$ with $A_1\cap A_2\subset E\subset A_1\cup A_2$, the set $E$ is also an adversarial Bayes classifier. The adversarial Bayes classifier is \emph{unique up to degeneracy} if any two adversarial Bayes classifiers are equivalent up to degeneracy.
\end{definition} 

Due to \cref{lemma:union_intersection_adv_bayes}, to verify that an adversarial Bayes classifier is unique up to degeneracy, it suffices to show that if $A_1$ and $A_3$ are any two adversarial Bayes classifiers with $A_1\subset A_3$, then any set satisfying $A_1\subset E\subset A_3$ is an adversarial Bayes classifier as well. In the example distribution discussed above (\cref{ex:degenerate}), the non-regular portion of the adversarial Bayes classifier could only be some subset of $D=[-1/4+\e,1/4-\e]$. The notion of `degenerate sets' formalizes this behavior.
\begin{definition}
    
    A set $D$ is \emph{degenerate} for an adversarial Bayes classifier $A$ if for every Borel set $E$ satisfying $A-D\subset E\subset A\cup D$, the set $E$ is also an adversarial Bayes classifier. 
\end{definition}

Equivalently, a set $D$ is degenerate for $A$ if for all disjoint subsets $D_1,D_2\subset D$, the set $(A\cup D_1)-D_2$ is also an adversarial Bayes classifier.
In terms of this definition: the adversarial Bayes classifiers $A_1$ and $A_2$ are equivalent up to degeneracy iff the set $A_1\triangle A_2$ is degenerate for either $A_1$ or $A_2$.

This paper first studies properties of these new notions, and then uses the resulting insights to characterize adversarial Bayes classifiers in one dimension. To start, we show that when $\PP\ll \mu$, equivalence up to degeneracy is in fact an equivalence relation (\cref{thm:equivalence_up_to_degeneracy}) and furthermore, every adversarial Bayes classifier has a `regular' representative when $d=1$ (\cref{thm:adv_bayes_and_degenerate}). The differentiation argument in \cref{sec:background_adv_bayes_classifiers} then produces necessary conditions characterizing regular adversarial Bayes classifiers in one dimension (\cref{thm:exists_regular}). These conditions provide a tool for understanding how the adversarial Bayes classifier depends on $\e$; see \cref{thm:adv_bayes_increasing_e} and \crefrange{prop:within_e_to_risk_bound}{prop:eta_0_1}. 
Identifying all adversarial Bayes classifiers then requires characterizing degenerate sets, and we provide such a criterion under specific assumptions. Lastly, \cref{thm:TFAE_equiv} provides alternative criteria for equivalence up to degeneracy.

\subsection*{Theorem statements}
First, equivalence up to degeneracy is in fact an equivalence relation for many common distributions.
\begin{theorem}
\label{thm:equivalence_up_to_degeneracy}
    If $\PP\ll \mu$, then equivalence up to degeneracy is an equivalence relation. 
\end{theorem}
\Cref{thm:equivalence_up_to_degeneracy} also implies that two equivalent adversarial Bayes classifiers must have the same degenerate sets.
\Cref{ex:non_equiv} demonstrates that the assumption $\PP \ll\mu$ is necessary. 
Additionally, uniqueness up to degeneracy generalizes certain notions of uniqueness for the Bayes classifier.
    \begin{theorem}\label{thm:TFAE_equiv}
        Assume that $\PP\ll \mu$ and $\e>0$. Then the following are equivalent:
        \begin{enumerate}[label=\Alph*)]
            \item \label{it:unique_under_deg} The adversarial Bayes classifier is unique up to degeneracy
            \item \label{it:S_e_unique} 
            Amongst all adversarial Bayes classifiers $A$, either the value of $\PP_0(A^\e)$ is unique or the value of  $\PP_1((A^C)^\e)$ is unique 
            \item \label{it:eta_*_meas zero} There are maximizers $\PP_0^*,\PP_1^*$ of $\cdl$ for which $\PP^*(\eta^*=1/2)=0$, where $\PP^*=\PP_0^*+\PP_1^*$ and $\eta^*=d\PP_1^*/d\PP^*$ 
        \end{enumerate}
    \end{theorem}

When $\e=0$,  \cref{it:S_e_unique} and \cref{it:eta_*_meas zero} are equivalent notions of uniqueness of the Bayes classifier (see \cref{prop:Bayes_equivalencies}). However, if $B_1$ and $B_2$ are Bayes classifiers, any set $E$ satisfying $B_1\cap B_2\subset E\subset B_1\cup B_2$ is always a Bayes classifier. Thus \cref{it:unique_under_deg} is not necessarily equivalent to \cref{it:S_e_unique,it:eta_*_meas zero} when $\e=0$.
When $\PP \not \ll \mu$, \cref{thm:equivalence_up_to_degeneracy} is false although \cref{it:S_e_unique} and \cref{it:eta_*_meas zero} are still equivalent (see \cref{ex:non_equiv} and \cref{lemma:non_abs_cont}). This equivalence  suggests a different notion of uniqueness for such distributions, see \cref{sec:den_equivalence_relation} for more details.

A central result of this paper is that degenerate sets are the only form of non-regularity possible in the adversarial Bayes classifier in one dimension. 
\begin{theorem} \label{thm:adv_bayes_and_degenerate}
    Assume that $d=1$, $\e>0$, and $\PP_0,\PP_1\ll \mu$. Then any adversarial Bayes classifier is equivalent up to degeneracy to an adversarial Bayes classifier
   $ A'= \bigsqcup_{i=m}^M (a_i,b_i)$ with $b_i-a_i>2\e$, $a_{i+1}-b_i>2\e$, and $-\infty\leq m\leq M-1\leq +\infty$.

\end{theorem}
Again, the alternative $m=M-1$ accounts for $A'=\emptyset$.
This result motivates the definition of \emph{regularity} in one dimension.
\begin{definition}\label{def:regularity}
    We say $E\subset \Rset$ is a \emph{regular set of radius $\e$} if one can write both $E$ and $E^C$ as a disjoint union of intervals of length strictly greater than $2\e$. 
\end{definition}
We will drop `of radius $\e$' when clear from the context. 
\Cref{thm:adv_bayes_and_degenerate} states that any adversarial Bayes classifier is equivalent to a regular set. However, as demonstrated by the example below, this claim does not extend to Bayes classifiers. 

\begin{example} \label{ex:counterexample_Bayes}
Define a distribution by
\begin{equation}\label{eq:counterexample_Bayes_def}
    p(x)=\begin{cases}
    0 &\text{if }x<-2\\
    \frac 13 (x+ 2) &\text{if }-2\leq x<-1\\
    \frac 13 &\text{if }-1\leq x\leq 1\\
    -\frac 13 (x-2)&\text{if } 1< x\leq 2\\
    0&\text{if }x>2
    \end{cases} \quad\quad\quad \eta(x)=\frac 12 +\frac{x}4 \cos\left(\frac \pi {x}  \right)
\end{equation}
Examining \cref{eq:Bayes_necessary}, one can conclude that the Bayes classifier $B=\{\eta(x)>1/2\}$ can be expressed as 
\begin{equation}\label{eq:Bayes_soln}
    B= \bigcup_{n=-1}^{-\infty} \left(\frac 2 {4n+3}, \frac 2 {4n+1}\right) \cup \bigcup_{n=0}^\infty \left(\frac 2 {4n+3}, \frac 2 {4n+1}\right)
\end{equation}
The point $0$ is an isolated point of $B^C$, and consequently $B$ is not a regular set. Furthermore, $B$ is not equivalent to any regular set, with equivalence defined as in \cref{def:equivalence_Bayes}. See \cref{app:counterexample_Bayes} for detailed proofs of these claims.
\end{example}
In fact, this example reveals that one cannot assume that $d_i<c_{i+1}$ in \cref{eq:Bayes_as_union}, even when $p_0$ and $p_1$ are continuous.
When $p_0,p_1$ are continuous, the necessary conditions \cref{eq:first_order_necessary} always hold for a regular adversarial Bayes classifier.
\begin{theorem}\label{thm:exists_regular}
    Let $d=1$ and assume that $\PP\ll \mu$. Let $ A=\bigcup_{i=m}^M(a_i,b_i)$ be a regular adversarial Bayes classifier. 

    If $p_0$ is continuous at $a_i-\e$ (resp. $b_i+\e$) and $p_1$ is continuous at $a_i+\e$ (resp. $b_i-\e$), then $a_i$ (resp. $b_i$) must satisfy the first order necessary conditions \cref{eq:first_order_necessary_a} (resp. \cref{eq:first_order_necessary_b}). Similarly, if $p_0$ is differentiable at $a_i-\e$ (resp. $b_i+\e$) and $p_1$ is differentiable at $a_i+\e$ (resp. $b_i-\e$), then $a_i$ (resp. $b_i$) must satisfy the second order necessary conditions \cref{eq:second_order_necessary_a} (resp. \cref{eq:second_order_necessary_b}).
\end{theorem}

\Cref{thm:exists_regular} provides a method for identifying a representative of every equivalence class of adversarial Bayes classifiers under equivalence up to degeneracy. 

\begin{enumerate}[label=\arabic*)]
    \item \label{it:procedure_first}Let $\mathfrak{a}$, $\mathfrak{b}$ be the set of points that satisfy the necessary conditions for $a_i$, $b_i$ respectively
    \item Form all possible open regular sets $\bigcup_{i=m}^M (a_i,b_i)$ with $a_i\in \mathfrak{a}$ and $b_i\in \mathfrak b$. 
    \item Identify which of these sets would be be equivalent up to degeneracy, if they were adversarial Bayes classifiers.
    \item \label{it:procedure_last}Compare the risks of all non-equivalent sets from step 2) to identify which are adversarial Bayes classifiers.
\end{enumerate}

 One only need to consider open sets in step 2) because the boundary of a regular adversarial Bayes classifier is always a degenerate set when $\PP\ll\mu$, as noted in \cref{sec:background_adv_bayes_classifiers} (see \cref{lemma:max_degenerate_set} for a formal proof). 
 
 \Cref{sec:examples} applies the procedure above to several example distributions, see \cref{ex:gaussians_equal_variances} for a crisp demonstration. The analysis in \cref{sec:examples} reveals some interesting patterns. First, boundary points of the adversarial Bayes classifier are frequently within $\e$ of boundary points of the Bayes classifier. \Cref{prop:uniform_within_e} and \cref{prop:eta_0_1} prove that this phenomenon occurs when either $\PP$ is a uniform distribution on an interval or $\eta\in\{0,1\}$, and \cref{prop:within_e_to_risk_bound} shows that this occurrence can reduce the accuracy-robustness tradeoff. Second, uniqueness up to degeneracy often fails only for a small number of values of $\e$ when $\PP_0(\Rset)\neq \PP_1(\Rset)$. Understanding both of these occurrences in more detail is an open problem.
 
 As illustrated by the examples throughout this paper, the procedure \crefrange{it:procedure_first}{it:procedure_last} is fairly brute force. Reducing the number of alternatives considered in this process is an open problem. In contrast, the procedure \ref{it:Bayes_procedure_first}-\ref{it:Bayes_procedure_last} provides a method for constructing a Bayes classifier directly from the sets $\mathfrak c$, $\mathfrak d$. Analyzing the set $S$ in step \ref{it:Bayes_procedure_last} serves solely to enumerate \emph{all} equivalence classes of Bayes classifiers.

\Cref{thm:exists_regular} is a tool for identifying a representative of each equivalence class of adversarial Bayes classifiers under equivalence up to degeneracy. Can one characterize all the members of a specific equivalence class? Answering this question requires understanding properties of degenerate sets.

\begin{theorem}\label{thm:1d_degenerate}
    Assume that $d=1$, $\PP\ll\mu$, $\e>0$ and let $A$ be an adversarial Bayes classifier.
    \begin{itemize}
        \item If some interval $I$ is degenerate for $A$ and $I$ is contained in $\supp \PP$, then $|I|\leq 2\e$.
        \item Conversely, the connected components of $A$ and $A^C$ of length less than or equal to $2\e$ are contained in a degenerate set.
        \item A countable union of degenerate sets is degenerate.
        \item Assume that $\supp \PP$ is an interval and $\PP(\eta\in\{0,1\})=0$. If $D$ is a degenerate set for $A$, then $D$ must be contained in the degenerate set $\ov{(\supp \PP^\e)^C}\cup \partial A$. 
    \end{itemize}
\end{theorem}
The first two bullets state that within the support of $\PP$, degenerate intervals must have length at most $2\e$, and conversely a component of size at most $2\e$ must be degenerate. The last bullet implies that when $\supp \PP$ is an interval and $\PP(\eta\in \{0,1\})=0$, the equivalence class of an adversarial Bayes classifier $A$ consists of all Borel sets that differ from $A$ by a measurable subset of $\ov{(\supp \PP^\e)^C}\cup \partial A$. Specifically, under these conditions, $A$ cannot have a degenerate interval contained in $\supp \PP^\e$, see \cref{lemma:degenerate_1d_eta_0_1}--- a helpful observation for identifying sets which are equivalent up to degeneracy in step 3) of the procedure above. In contrast, Bayes classifiers with pathological behavior on $\supp \PP$ always exist. If $B=\bigcup_{i=m}^M (c_i,d_i)$ is a Bayes classifier, then at least one of $B\cup \QQ$ and $B-\QQ$ is a Bayes classifier that fails to satisfy the definition of regularity in \cref{def:regularity} for $\e=0$. Both of the assumptions present in the fourth bullet are necessary--- \cref{ex:degenerate} presents a counterexample  where $\supp \PP$ is an interval and $\PP(\eta\in \{0,1\})>0$  while \cref{ex:deg_eta_0_1_counterexample} presents a counterexample for which $\PP(\eta\in \{0,1\})=0$ but $\supp \PP$ is not an interval.

Prior work \cite{AwasthiFrankMohri2021,BungertGarciaMurray2021} shows that a certain form of regularity for adversarial Bayes classifiers improves as $\e$ increases. \Cref{thm:adv_bayes_and_degenerate} is an expression of this principle: this theorem states that each adversarial Bayes classifier $A$ is equivalent to a regular set of radius $\e$, and thus the regularity guarantee improves as $\e$ increases. Another form of regularity also improves as $\e$ increases---the number of components of $A$ and $A^C$ must decrease for well-behaved distributions.
Let $\comp(A)\in \Nset\cup \{\infty\}$ be the number of connected components of a set $A$.
\begin{theorem}\label{thm:adv_bayes_increasing_e}
    Assume that $d=1$, $\PP\ll \mu$, $\supp \PP$ is an interval $I$, and $\PP(\eta\in \{0,1\})=0$. Let $\e_2>\e_1\geq 0$ and let $A_1$, $A_2$ be adversarial Bayes classifiers corresponding to perturbation radiuses $\e_1$ and $\e_2$ respectively. Then either $\emptyset$ and $\Rset$ minimize both $R^{\e_1}$ and $R^{\e_2}$ or $\comp(A_1\cap \interior( I^{\e_1}))\geq \comp(A_2\cap \interior(I^{\e_2}))$ and $\comp(A_1^C\cap \interior(I^{\e_1}))\geq \comp(A_2^C\cap \interior(I^{\e_2}))$.
\end{theorem}
    The case where $\Rset$ and $\emptyset$ are both minimizers of both $R^{\e_1}$, $R^{\e_2}$ arises, for instance, when $\PP_0$ and $\PP_1$ are identical.  Due to the fourth bullet of \cref{thm:1d_degenerate}, the assumptions of \cref{thm:adv_bayes_increasing_e} imply that there is no degenerate interval within $\interior(\supp \PP^\e)$ when $\e>0$, and hence every adversarial Bayes classifier matches a regular adversarial Bayes classifier on $\supp \PP^\e$.
    The intersections with $\interior(I^{\e_1})$ and $\interior(I^{\e_2})$ are necessary in  the theorem above because the complements $\interior(I^{\e_1})^C$ and $\interior(I^{\e_2})^C$ are degenerate sets for $A_1$ and $A_2$, respectively, by the fourth bullet of \cref{thm:1d_degenerate}. \Cref{sec:increasing_e} actually proves a stronger statement: typically, no component of $A_1\cap I^{\e_1}$ can contain a connected component of $A_2^C$ and no component of $A_1^C\cap I^{\e_1}$ can contain a connected component of $A_2$.

    When computing adversarial Bayes classifiers, \cref{thm:adv_bayes_increasing_e} and the stronger version in \cref{sec:increasing_e} are useful tools in ruling out some of the sets in step 2) of the procedure above without explicitly computing their risk.

    The crux of the one-dimensional characterization discussed above is proving that every adversarial Bayes classifier is equivalent up to degeneracy to a regular adversarial Bayes classifier. Unfortunately, this result is false in high dimensions: See Figure~3 of \cite{BungertGarciaMurray2021} for a counterexample\footnote{This example discusses a singular distribution, but the result could be extended to non-singular distributions by replacing the $\delta$-functions with a uniform distribution on a small ball.} and Figure~5 of \cite{BungertGarciaMurray2021} for a depiction of a key obstacle in two dimensions. However, it is possible to establish one-sided regularity results:
    \begin{theorem}\label{thm:high_dim_uniqueness regularity}
    Let $A$ be an adversarial Bayes classifier. Then $A$ is equivalent up to degeneracy to a classifier $A_1$ for which $A_1=C^\e$ and a classifier $A_2$ for which $A_2^C=E^\e$, for some sets $C$, $E$. 
    \end{theorem}
    This result has appeared in a different form in prior work-- see for instance Lemma~3.29 of \cite{BungertGarciaMurray2021} or Lemma~15 of \cite{AwasthiFrankMohri2021}. 
    Further understanding uniqueness up to degeneracy in higher dimension is an open problem.

\subsection*{Paper Outline} \Cref{sec:examples} applies the tools presented above to compute adversarial Bayes classifiers for a variety of distributions. Subsequently, \cref{sec:equiv_degeneracy} presents properties of equivalence up to degeneracy, including proofs of \cref{thm:equivalence_up_to_degeneracy,thm:TFAE_equiv,thm:high_dim_uniqueness regularity}. \Cref{sec:fundamental_regularity,sec:degenerate} further develop properties of degenerate sets, and these results are subsequently applied in \cref{sec:regular_adv_bayes} to prove \cref{thm:adv_bayes_and_degenerate,thm:exists_regular}. \Cref{sec:degenerate_1d} focuses specifically on distributions in one dimension to prove \cref{thm:1d_degenerate}. Subsequently, \cref{sec:increasing_e} proves \cref{thm:adv_bayes_increasing_e}. Lastly, \Cref{sec:related_works} compares our results with related works. Technical proofs and calculations appear in the \ifthenelse{\boolean{appendixmode}}{appendix}{supplementary materials}.

\section{Examples}\label{sec:examples}
 The examples below calculate the equivalence classes under equivalence up to degeneracy for any $\e>0$. \Cref{sec:elementary_examples} presents elementary examples illustrating our procedure for finding adversarial Bayes classifiers. \Cref{sec:anomalous_examples} discusses two distributions with pathological behavior: \cref{ex:non_uniqueness_all} presents a distribution with a unique Bayes classifier but non-unique adversarial Bayes classifiers for all $\e>0$ while \cref{ex:degenerate} presents an example with a degenerate set of positive Lebesgue measure. Remarkably, in these examples, for sufficiently small $\e$, there always exists a set that simultaneously simultaneously serves as both a Bayes and adversarial Bayes classifier. For such distributions, a deliberate selection of the adversarial Bayes classifier would mitigate the tradeoff between robustness and accuracy.

 Moreover, all of the examples below except \cref{ex:gaussians_equal_means} exhibit a curious occurrence--- the boundary of the adversarial Bayes classifier lies within $\e$ of the boundary of the Bayes classifier. \Cref{sec:comparing_bayes_and_adv_bayes} explores this pattern in detail--- \Cref{prop:uniform_within_e,prop:eta_0_1} provide conditions under which this behavior occur while \cref{prop:within_e_to_risk_bound} examines the robustness-accuracy tradeoff in light of this phenomenon.

Three of the examples in this section involve gaussian mixtures with varying parameters. Understanding how uniqueness depends on the parameters of the gaussian mixture remains an open problem. More broadly, understanding how uniqueness is influenced by the parameters of a broader parametric family of distributions remains an open question.

\subsection{Elementary examples}\label{sec:elementary_examples}
The first two examples study Gaussian mixtures: $p_0=(1-\lambda) g_{\mu_0,\sigma_0}(x)$, $p_1=\lambda g_{\mu_1,\sigma_1}(x)$, 
where $\lambda \in (0,1)$ and $g_{\mu,\sigma}$ is the density of a gaussian with mean $\mu$ and variance $\sigma^2$. Prior work \cite{PydiJog2019} 
calculates a single adversarial Bayes classifier for $\lambda=1/2$ and any value of $\mu_i$ and $\sigma_i$. Below, our goal is to find \emph{all} adversarial Bayes classifiers.

\begin{figure}
     \centering
     \begin{subfigure}[b]{0.30\textwidth}
         \centering
         \includegraphics[width=\textwidth]{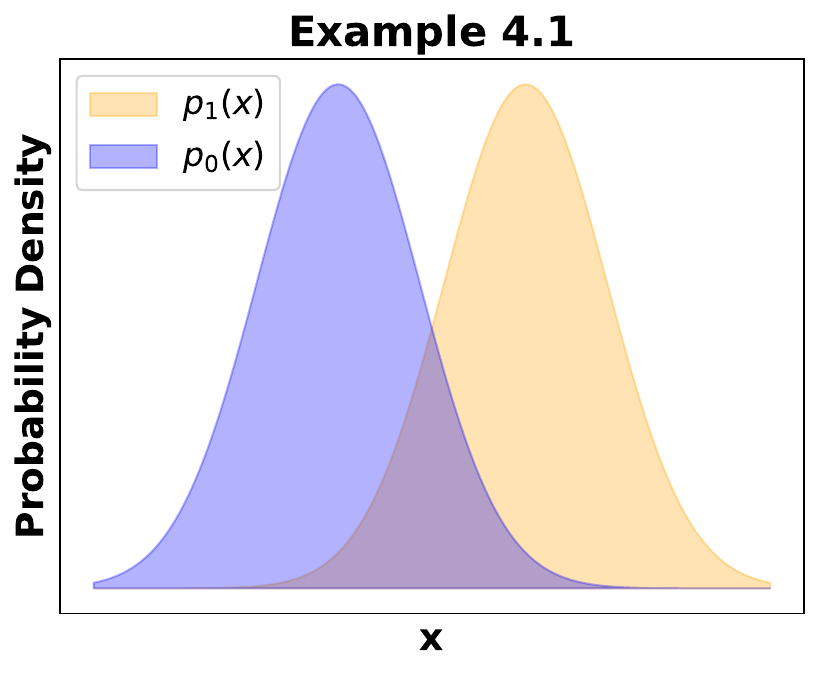}
         \caption{}
         \label{fig:gaussian_equal_var}
     \end{subfigure}
     \hfill
     \begin{subfigure}[b]{0.30\textwidth}
         \centering
         \includegraphics[width=\textwidth]{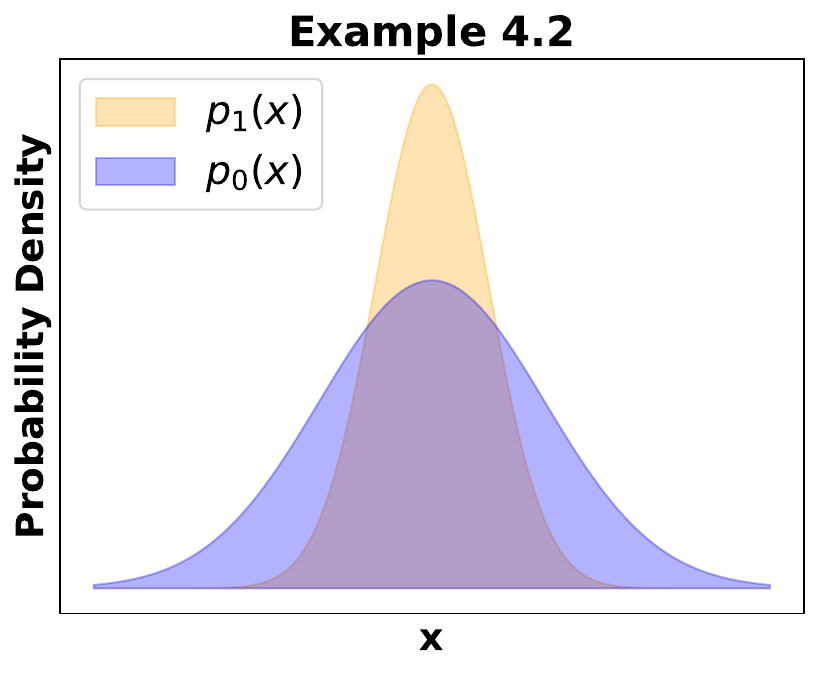}
         \caption{}
         \label{fig:gaussian_equal_means}
     \end{subfigure}
      \hfill
     \begin{subfigure}[b]{0.30\textwidth}
         \centering
        \includegraphics[width=\textwidth]{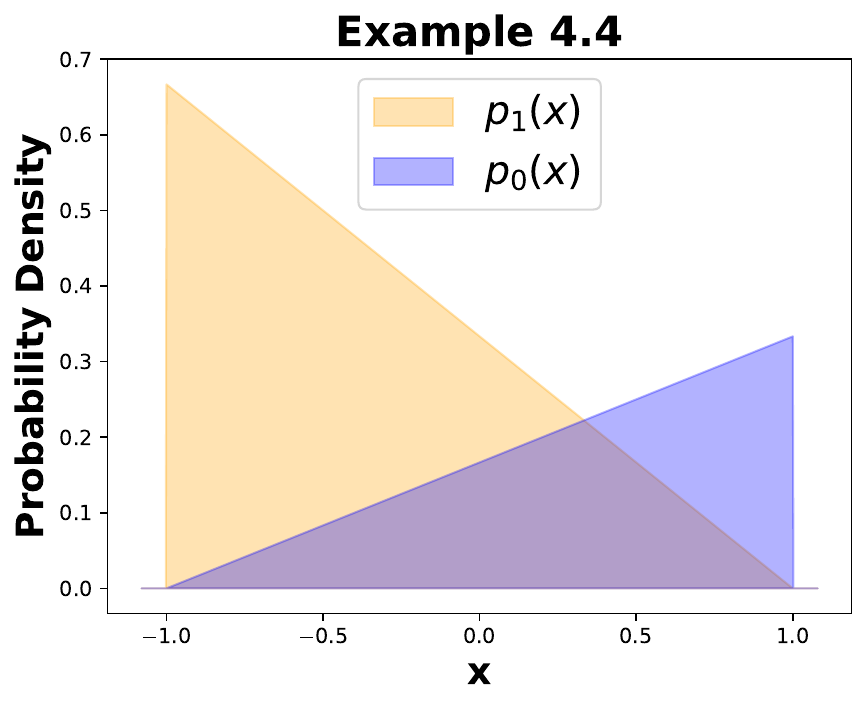}
         \caption{}
         \label{fig:non_uniqueness_single}
     \end{subfigure}\caption{(a) Gaussian mixture with equal means and unequal variances as a in \cref{ex:gaussians_equal_means}. (b) Gaussian mixture with equal weights, unequal means, and equal variances as in \cref{ex:gaussians_equal_variances}. (c) The distribution of \cref{ex:non_uniqueness_single}.}
        \label{fig:gaussians}
\end{figure}

\begin{example}[Gaussian mixtures--- equal variances, equal weights]\label{ex:gaussians_equal_variances}
    Consider a gaussian mixture with $p_0(x)=\frac 1 2 \cdot\frac{1}{\sqrt{2\pi}\sigma} e^{-(x-\mu_0)^2/2\sigma^2}$, $p_1(x)=\frac 12 \cdot \frac 12 \frac 1 {\sqrt{2\pi}\sigma} e^{-(x-\mu_1)^2/\sigma^2}$, and $\mu_1>\mu_0$, depicted in \cref{fig:gaussian_equal_var}. The solutions to the first order necessary conditions $p_1(b-\e)-p_0(b+\e)=0$ and $p_1(a+\e)-p_0(a-\e)=0$ from \cref{eq:first_order_necessary} are 
    \begin{equation*}
           a(\e)=b(\e)=\frac{\mu_0+\mu_1}2
    \end{equation*}
    However, the point $b(\e)$ does not satisfy the second order necessary condition \cref{eq:second_order_necessary_b} (see \cref{app:gaussians_equal_variances}).
    Thus the candidate sets for the Bayes classifier are $\Rset$, $\emptyset$, and $(a(\e),+\infty)$. The fourth bullet of \cref{thm:1d_degenerate} implies that none of these sets could be equivalent up to degeneracy. By comparing the adversarial risks of these three sets, one can show that the set $(a(\e),+\infty)$ is an adversarial Bayes classifier iff $\e\leq \frac{\mu_1-\mu_0} 2$ and $\Rset$, $\emptyset$ are adversarial Bayes classifiers iff $\e\geq \frac{\mu_1-\mu_0} 2$ (see \cref{app:gaussians_equal_variances} for details).
    Thus the adversarial Bayes classifier is unique up to degeneracy only when $\e<\frac{\mu_1-\mu_0} 2$.
\end{example}

When $\e\leq \frac{\mu_1-\mu_0}2$, the set $(a(\e),+\infty)$ is both a Bayes classifier and an adversarial Bayes classifier, and thus there is no accuracy-robustness tradeoff. In this example, uniqueness up to degeneracy fails for all sufficiently large $\e$. In contrast, the example below demonstrates a distribution for which the adversarial Bayes classifier is unique up to degeneracy for all $\e$.

\begin{example}[Gaussian mixtures--- equal means]\label{ex:gaussians_equal_means}
    Consider a gaussian mixture with $p_0(x)=\frac{1-\lambda}{\sqrt{2\pi}\sigma_0} e^{-x^2/2\sigma_0^2}$, $p_1(x)=\frac{\lambda}{\sqrt{2\pi}\sigma_1}e^{-x^2/2\sigma_1^2}$, and $\lambda\geq 1/2$. Assume that $p_0$ has a larger variance than $p_1$ but that the peak of $p_0$ is below the peak of $p_1$, or other words, $\sigma_0>\sigma_1$ but $\frac \lambda{\sigma_1}>\frac{1-\lambda}{\sigma_0}$, see \cref{fig:gaussian_equal_means} for a depiction. Calculations similar to \cref{ex:gaussians_equal_variances} show that the adversarial Bayes classifier is unique up to degeneracy for every $\e$, and is given by $(-b(\e),b(\e))$ where
    \begin{equation}\label{eq:b(e)_def}
        b(\e)=\frac{\e\left( \frac 1 {\sigma_1^2}+\frac 1 {\sigma_0^2}\right)+\sqrt{\frac{4\e^2}{\sigma_0^2\sigma_1^2}+2\left(\frac 1 {\sigma_1^2}-\frac 1 {\sigma_0^2} \right)\ln \frac{\lambda \sigma_0}{(1-\lambda)\sigma_1}}}{\frac 1 {\sigma_1^2}-\frac 1 {\sigma_0^2}}.
    \end{equation}
    The computational details are similar to those of \cref{ex:gaussians_equal_variances}, and thus are delayed to \cref{app:gaussians_equal_means_details}.
\end{example}
    Unlike \cref{ex:gaussians_equal_variances},  the Bayes and adversarial Bayes classifiers for this distribution can differ substantially.

    The next three examples are distributions for which $\supp \PP$ is a finite interval. In such situations, it is often helpful to assume that $a_i,b_i$ are not near $\partial \supp \PP$.

\begin{lemma}\label{lemma:eps_in_interval}
    Consider a distribution for which $\PP\ll \mu$ and $\supp \PP$ is an interval. Then every adversarial Bayes classifier is equivalent up to degeneracy to an open regular adversarial Bayes classifier $A=\bigcup_{i=m}^M(a_i,b_i)$ for which the finite $a_i$, $b_i$ are contained in $\interior(\supp \PP ^{-\e})$.
\end{lemma}

See \cref{app:eps_in_interval} for a proof. The following example presents a setting in which $\supp \PP$ is an interval but $p_0$, $p_1$ are discontinuous at $\partial \supp \PP$. This example illustrates how \cref{lemma:eps_in_interval} facilitates the analysis of such distributions.

\begin{example}[Uniqueness fails for a single value of $\e$]\label{ex:non_uniqueness_single}
    Consider a distribution for which
    \[p_0(x)=\begin{cases}
        \frac 1 6 (1+x)&\text{if } |x|\leq 1\\
        0&\text{otherwise}    
    \end{cases}
   \quad  p_1(x)=\begin{cases}
       \frac 13 (1-x)&\text{if }|x|\leq 1\\
       0&\text{otherwise}
   \end{cases}\]
   The only solutions to the first order necessary conditions $p_1(a+\e)-p_0(a-\e)=0$ and $p_0(b+\e)-p_1(b-\e)=0$ within $\supp \PP^\e$ are 
   \[a(\e)=\frac 13(1-\e)\quad \text{and} \quad b(\e)=\frac 13(1+\e)\]
   We first consider $\e$ small enough so that both of these points lie in $\interior(\supp \PP^{-\e})$, or in other words, $\e< 1/2$. Then $p_0'(a(\e)-\e)=p_0'(b(\e)+\e)=1/6$ and $p_1'(a(\e)+\e)=p_1'(b(\e)-\e)=-1/3$. Consequently, the point $a(\e)$ fails to satisfy the second order necessary condition \cref{eq:second_order_necessary_a}. To identify all adversarial Bayes classifiers under uniqueness up to degeneracy for $\e<1/2$, \cref{lemma:eps_in_interval} implies it remains to compare the adversarial risks of $\emptyset$, $\Rset$, and $(-\infty, b(\e))$. \Cref{thm:1d_degenerate} implies that none of these classifiers could be equivalent up to degeneracy. The risks of these sets compute to $R^\e(\emptyset)=2/3$, $R^\e(\Rset)=1/3$, and $R^\e((-\infty,b(\e)))=\frac 29\left( 1+\e\right)^2$. Therefore, for all $\e<1/2$, the set $(-\infty,b(\e))$ is an adversarial Bayes classifier iff $\e\leq \sqrt{3/2}-1$ while $\Rset$ is an adversarial Bayes classifier iff $1/2>\e\geq \sqrt{3/2}-1$. \Cref{thm:adv_bayes_increasing_e} then implies that this last statement holds without the restriction $\e<1/2$.
\end{example}

\begin{figure}
	\centering
	\begin{subfigure}[b]{0.30\textwidth}
		\centering
		\includegraphics[width=\textwidth]{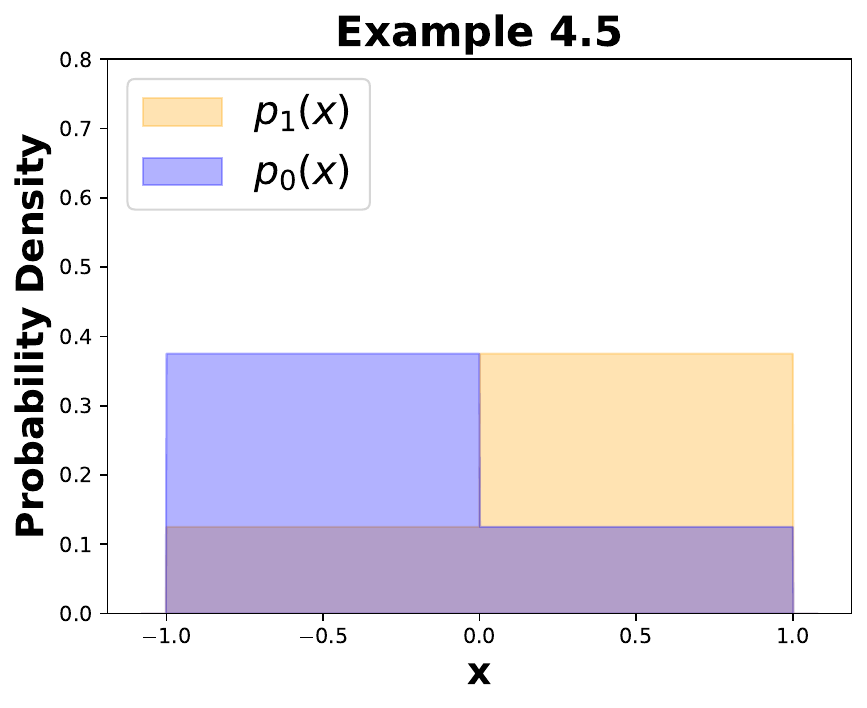}
		\caption{}
        \label{fig:non_uniqueness_all}
	\end{subfigure}
	\hfill
	\begin{subfigure}[b]{0.30\textwidth}
		\centering
		\includegraphics[width=\textwidth]{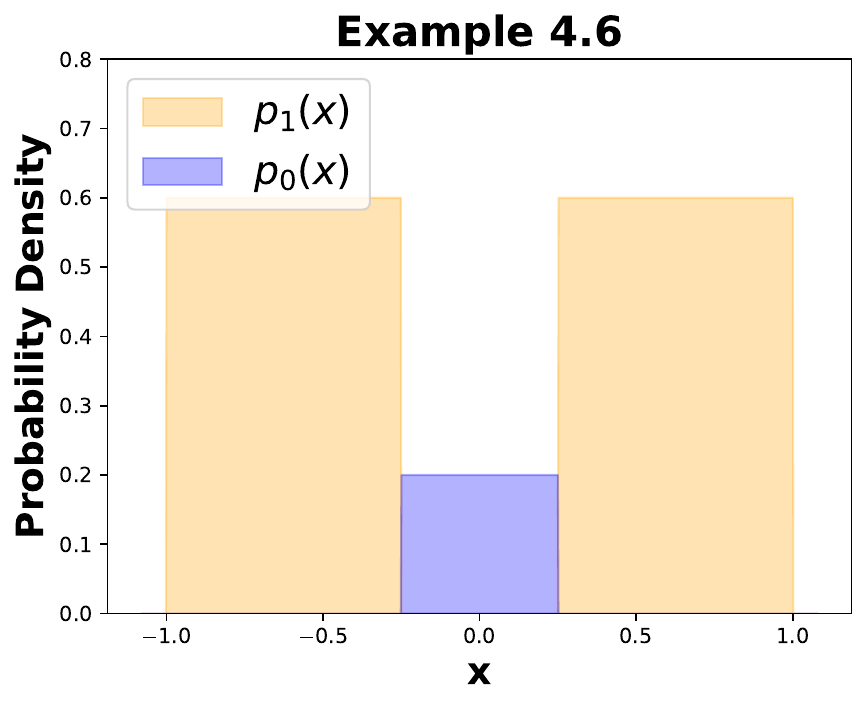}
		\caption{}
		\label{fig:degenerate}
	\end{subfigure}
	\hfill
	\begin{subfigure}[b]{0.30\textwidth}
		\centering
		\includegraphics[width=\textwidth]{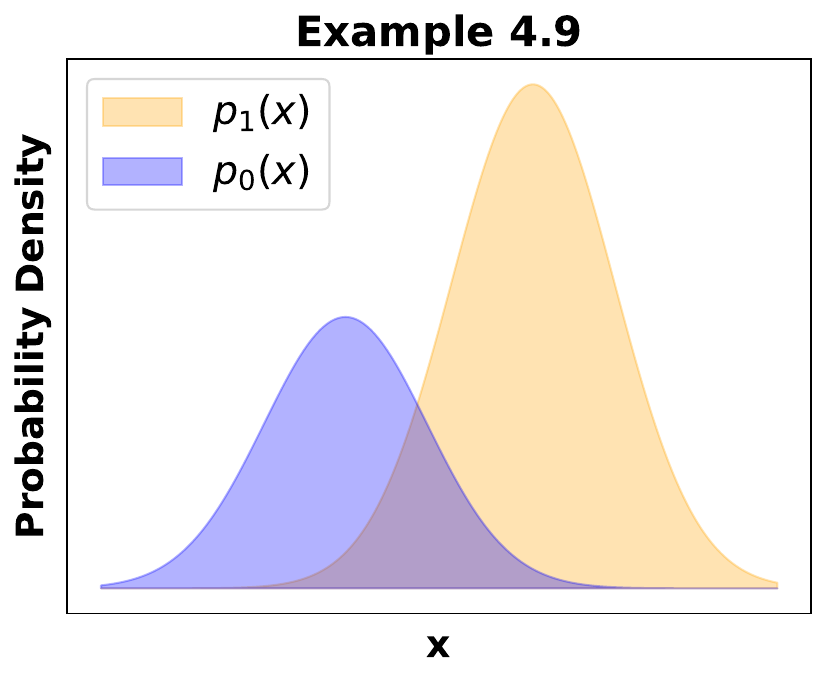}
		\caption{}
        \label{fig:gaussian_equal_var_unequal_weights}
	\end{subfigure}\caption{(a) The distribution of \cref{ex:non_uniqueness_all}. (b) The distribution of \cref{ex:degenerate}. (c) Gaussian mixture with unequal weights, unequal means, and equal variances as in \cref{ex:gaussians_unequal_weights}.}
	\label{fig:three_ex}
\end{figure}

\subsection{Anomalous examples and implications for the accuracy-robustness tradeoff}\label{sec:anomalous_examples}
In the previous examples, uniqueness up to degeneracy fails for only a single value of $\e$. In contrast, uniqueness up to degeneracy fails for all $\e$ in the distribution below. 

\begin{example}[Non-uniqueness for all $\e>0$]\label{ex:non_uniqueness_all}
    Let $p$ be the uniform density on the interval $[-1,1]$ and let 
    \[\eta(x)=
    \begin{cases}
        \frac 14 &\text{if } x\leq 0\\
        \frac 34 &\text{if } x>0
    \end{cases}\]
    
    Calculations for this example are similar to those in \cref{ex:non_uniqueness_single}, so we delay the details to \cref{app:non_uniqueness_all}.  For this distribution, the set $(y,\infty)$ is an adversarial Bayes classifier for any $y\in [-\e,\e]$ iff $\e\leq 1/3$ and $\emptyset,\Rset$ are adversarial Bayes classifiers iff $\e\geq 1/3$. \Cref{thm:1d_degenerate} implies that none of these sets could be equivalent up to degeneracy.
    Therefore, the adversarial Bayes classifier is not unique up to degeneracy for all $\e>0$ even though the Bayes classifier is unique.
\end{example}
Again, the adversarial Bayes classifier $(0,\infty)$ is also a Bayes classifier when $\e\leq 1/3$, and thus there is no accuracy-robustness tradeoff for this distribution.

A distribution is said to satisfy \emph{Massart's noise condition} if $|\eta(\bx)-1/2|\geq \delta$ $\PP$-a.e. for some $\delta>0$. Prior work \cite{MassartNedelec06} relates this condition to the sample complexity of learning from a function class. For the example above, \cref{thm:TFAE_equiv} implies that Massart's noise condition cannot hold for any maximizer of $\cdl$ even though $|\eta-1/2|\geq 1/4$ $\PP$-a.e.

In light of the prior examples, one would expect that if the Bayes classifier $B$ is unique and $p_0$, $p_1$ are sufficiently regular near $\partial B$, then the adversarial Bayes classifier would be unique. Proving such a statement remains an open question.

The next example exhibits a degenerate set that has positive measure under $\PP$.
\begin{example}[Example of a degenerate set]\label{ex:degenerate}
Consider a distribution on $[-1,1]$ with
\[p_0(x)=\begin{cases}
    \frac 15&\text{if }|x|\leq 1/4\\
    0&\text{otherwise } 
\end{cases} \quad p_1(x)=\begin{cases}
   \frac 3 5 &\text{if } 1\geq |x|>1/4\\
   0 &\text{otherwise }
\end{cases}\]
\Cref{thm:exists_regular} and \cref{lemma:eps_in_interval} imply that one only need consider $a_i,b_i\in\{-\frac 14\pm \e,\frac 14\pm\e\} $ when identifying a regular representative of each equivalence class of adversarial Bayes classifiers. 
By comparing the adversarial risks of the regular sets satisfying this criterion, one can show that when $\e\leq 1/8$, every adversarial Bayes classifier is equivalent up to degeneracy to the regular set $(-\infty, -1/4+\e)\cup (1/4-\e,\infty)$ but when $\e\geq 1/8$ then every adversarial Bayes classifier is equivalent up to degeneracy to the regular set $\Rset$ (see \cref{app:ex_degenerate_details} for details.)

Next consider $\e \in [1/8,1/4]$. If $S$ is an arbitrary subset of $[-1/4+\e,1/4-\e]$, then $R^\e(\Rset)=R^\e(S^C)$. Thus the interval $[-1/4+\e,1/4-\e]$ is a degenerate set.
\end{example}

When $\e\in [1/8,1/4]$, both $\Rset$ and $(-\infty, -1/4+\e)\cup (1/4-\e,+\infty)$ are adversarial Bayes classifiers, but their Bayes risks differ by $\frac 25 (1-4\e)$, which is close to $1/5$ for $\e$ near $1/8$. Thus a careful selection of the adversarial Bayes classifier can mitigate the accuracy-robustness tradeoff.

In each of the prior example distributions except \cref{ex:gaussians_equal_means}, the boundary of the adversarial Bayes classifier is always within $\e$ of the boundary of the Bayes classifier.
If in addition the Bayes and adversarial Bayes classifiers have the same number of components, one can bound the minimal adversarial Bayes error in terms of the Bayes error rate and $\e$. 

\begin{proposition}\label{prop:within_e_to_risk_bound}
    Let $B=\bigcup_{i=1}^M (c_i,d_i)$, $A=\bigcup_{i=1}^M (a_i,b_i)$ be Bayes and adversarial Bayes classifiers respectively. Assume that $p_0$, $p_1$ are bounded above by $K$. Then if $|a_i-c_i|\leq \e$ and $|b_i-d_i|\leq \e$, then $R(A)-R(B)\leq 2\e MK$.
\end{proposition}

Thus the robustness-accuracy tradeoff will be minimal so long as $\e\ll 1/MK$.
\begin{proof} 
    First, observe that $|p_1(x)-p_0(x)|\leq K$ because densities are always non-negative. Consequently:
    \[R(A)-R(B)\leq \sum_{i=1}^M \int_{\min(a_i, c_i)}^{\max(a_i, c_i)} |p_1(x)-p_0(x)| dx +\int_{\min(b_i,d_i)}^{\max(b_i,d_i)} |p_1(x)-p_0(x)| dx\leq 2\e MK.\] 
\end{proof}

\subsection{Comparing Bayes and adversarial Bayes classifiers}\label{sec:comparing_bayes_and_adv_bayes}
 In each of the preceding examples except \cref{ex:gaussians_equal_means}, the boundary of the adversarial Bayes classifier always lies within $\e$ of the boundary of the Bayes classifier. This section presents three propositions exploring this phenomenon. In contrast, understanding when the Bayes and adversarial Bayes classifiers must differ substantially is an open problem.

The next proposition stipulates a widely applicable criterion under which there is always a solution to the necessary conditions $p_1(a+\e)-p_0(a-\e)=0$ and $p_1(b-\e)-p_0(b+\e)=0$ within $\e$ of solutions to $p_1(x)=p_0(x)$ (which specifies the boundary of the Bayes classifier).
\begin{proposition}\label{prop:soln_within_e}
    Let $z$ be a point with $p_1(z)-p_0(z)=0$ and assume that $p_0$ and $p_1$ are continuous on $[z-r,z+r]$ for some $r>0$. Furthermore, assume that one of $p_0,p_1$ is non-increasing and the other is non-decreasing on $[z-r,z+r]$. Then for all $\e\leq r/2$ there is a solution to the first order necessary conditions \cref{eq:first_order_necessary_a} and \cref{eq:first_order_necessary_b} within $\e$ of $z$.
\end{proposition}
Notice that if $p_1$, $p_0$ are differentiable, the assumptions of this proposition imply that for all $x,y\in (z-r,z+r)$, either $p_1'(x)\geq 0$ and $p_0'(y)\leq 0$ or $p_1'(x)\leq 0$ and $p_0'(y)\geq 0$. If these inequalities are in fact strict, the second order necessary conditions \cref{eq:second_order_necessary} will distinguish between left and right endpoints $a_i, b_i$ of the adversarial Bayes classifier. However, this proposition does not guarantee that the solution to the necessary conditions within $\e$ of $z$ \emph{must} be a boundary point of the adversarial Bayes classifier.

Prior work \cite{trillosMurray2022} proves the existence of solutions to \cref{eq:first_order_necessary} for sufficiently small $\e$ by assuming that $p_1'(z)-p_0'(z)\neq 0$ at points on the boundary of the Bayes classifier (see \cref{thm:sufficient_prior_work}). This result and \cref{prop:soln_within_e} are complimentary--- 
for instance, \cref{thm:sufficient_prior_work} applies to \cref{ex:gaussians_equal_means} while \cref{prop:soln_within_e} does not. 
In contrast, \cref{prop:soln_within_e} always produces quantitative bound relating the solutions of \cref{eq:Bayes_necessary} and \cref{eq:first_order_necessary}, and still applies even when $p_0'(x)-p_1'(x)=0$.
\begin{proof}[Proof of \cref{prop:soln_within_e}]
    Without loss of generality, assume that $p_1$ is non-increasing and $p_0$ is non-decreasing on $[z-r,z+r]$. The applying the relation $p_1(z)=p_0(z)$, one can conclude that 
    \[p_1((z-\e)+\e)-p_0( (z-\e)-\e)=p_1(z)-p_0(z-2\e)=p_0(z)-p_0(z-2\e)\geq 0.\]
    An analogous argument shows that $p_1( (z+\e)+\e)-p_0((z+\e)-\e)\leq 0$. Thus the intermediate value theorem implies that there is a solution to \cref{eq:first_order_necessary_a} within $\e$ of $z$. Analogous reasoning shows that there is a solution to \cref{eq:first_order_necessary_b} within $\e$ of $z$.
\end{proof}
 To illustrate the utility of this result, we apply it to a gaussian mixture below. Instead of directly solving the first order necessary conditions \cref{eq:first_order_necessary}, we identify the Bayes classifier and then apply \cref{prop:soln_within_e}, yielding an estimate of the location of solutions to \cref{eq:first_order_necessary} for sufficiently small $\e$. 
 \begin{example}[Gaussian mixtures--- equal variances]\label{ex:gaussians_unequal_weights}
     Consider a gaussian mixture with 
     \[p_1(x)=\frac \lambda {\sqrt{2\pi}\sigma} e^{-\frac{(x-\mu_1)^2}{2\sigma^2}}, \quad p_0(x)=\frac {1-\lambda} {\sqrt{2\pi}\sigma} e^{-\frac{(x-\mu_0)^2}{2\sigma^2}}\] 
     for which $p_1(\mu_1)>p_0(\mu_1)$ and $p_0(\mu_0)>p_1(\mu_0)$, see \cref{fig:gaussian_equal_var_unequal_weights}) for an illustration. This condition reduces to the restriction $\lambda\in (k/(1+k),1/(1+k))$, with $k=\exp{(-\frac{(\mu_1-\mu_0)^2}{2\sigma^2})}$. 
     Solving the conditions \cref{eq:Bayes_necessary,eq:Bayes_second_order_necessary} yields the Bayes classifier yields $B=(c,+\infty)$, with
     \[c=\frac{\mu_1+\mu_0}2+\frac{\sigma^2}{\mu_1-\mu_0}\ln\left(\frac{1-\lambda} \lambda \right).\]

     The endpoint $c$ lies closer to $\mu_1$ when $\lambda <1/2$ and closer to $\mu_0$ when $\lambda>1/2$. Consequently, the density $p_0$ is strictly decreasing and $p_1$ is strictly increasing on $[z-r,z+r]$ with 
     \[r=\frac{\mu_1-\mu_0}2-\frac{\sigma^2}{\mu_1-\mu_0}\left| \ln\left( \frac{1-\lambda}\lambda \right)\right|.\]
        Furthermore, as the density $p_0$ is strictly decreasing and $p_1$ is strictly increasing, we conclude that $p_0'(x)<0$ and $p_1'(y)>0$ for all $x,y\in (z-r,z+r)$. Therefore, by \cref{prop:soln_within_e} and \cref{eq:second_order_necessary}, there always exists a solution to \cref{eq:first_order_necessary_a,eq:second_order_necessary_a} and no solution to \cref{eq:first_order_necessary_b,eq:second_order_necessary_b} within $[c-r,c+r]$ when $\e<r/2$. As in \cref{ex:gaussians_equal_variances}, the necessary condition \cref{eq:first_order_necessary_a} reduces to a linear equation, implying that there is at most one point $a(\e)$ that satisfies \cref{eq:first_order_necessary_a}.
 \end{example}

Next, if $\PP$ is the uniform distribution on an interval, the solutions to the first order necessary conditions \cref{eq:first_order_necessary} are within $\e$ of solutions to $p_0(z)=p_1(z)$.

\begin{proposition}\label{prop:uniform_within_e}
    Assume that $\PP$ is the uniform distribution on an finite interval, $\eta$ is continuous on $\supp \PP$, and $\eta(x)=1/2$ only at finitely many points within $\supp \PP$. Then any adversarial Bayes classifier is equivalent up to degeneracy to an adversarial Bayes classifier $A=\bigcup_{i=m}^M(a_i,b_i)$ for which each $a_i,b_i$ is at most $\e$ from some point $z$ satisfying $\eta(z)=1/2$.
\end{proposition}
If $\PP$ is the uniform distribution, the first order necessary conditions \cref{eq:first_order_necessary} reduce to 
\vspace{-12pt}\sidebysidesubequations{eq:uniform_necessary_reduce_2}{\eta(a_i+\e)=1-\eta(a_i-\e)}{\eta(b_i-\e)=1-\eta(b_i+\e)}\vspace{5pt}
If $\eta$ is continuous, these equations can only have a solution within $\e$ of a point where $\eta(z)=1/2$. This conclusion, which follows from an argument based on the intermediate value theorem, is similar to the proof of \cref{prop:soln_within_e}
The full argument is provided in\cref{app:uniform_within_e}.

Finally, under fairly general conditions, when $\eta \in \{0,1\}$, the boundary of the adversarial Bayes classifier must be within $\e$ of the boundary of the Bayes classifier. 
\begin{proposition}\label{prop:eta_0_1}
    Assume that $\supp \PP$ is an interval, $\PP\ll\mu$, $\eta\in \{0,1\}$, and $p$ is continuous on $\supp \PP$ and strictly positive. Then any adversarial Bayes classifier is equivalent up to degeneracy to an open regular adversarial Bayes classifier $A=\bigcup_{i=m}^M(a_i,b_i)$ with $m\leq M+1$ for which each $a_i,b_i$ is at most $\e$ from $\partial \{\eta=1\}$.
\end{proposition}
Again, the proof is very similar to that of \cref{prop:soln_within_e,prop:uniform_within_e}, see \cref{app:eta_0_1} for details.

\section{Equivalence up to degeneracy}\label{sec:equiv_degeneracy}

\subsection{Equivalence up to degeneracy as an equivalence relation}\label{sec:den_equivalence_relation}

When $\PP\ll\mu$, there are several useful characterizations of equivalence up to degeneracy.

\begin{proposition}\label{prop:abs_cont_equivalencies}
    Let $\PP\ll \mu$ and $\e>0$. Let $(\PP_0^*,\PP_1^*)$ be a maximizer of $\cdl$, and set $\PP^*=\PP_0^*+\PP_1^*$. Let $A_1$ and $A_2$ be adversarial Bayes classifiers. Then the following are equivalent:
    \begin{enumerate}[label=\arabic*)]
        \item\label{it:degeneracy_equiv} The adversarial Bayes classifiers $A_1$ and $A_2$ are equivalent up to degeneracy 
        \item \label{it:S_1}\label{it:S_0} Either $S_\e(\one_{A_1})=S_\e(\one_{A_2})$-$\PP_0$-a.e. or $S_\e(\one_{A_2^C})=S_\e(\one_{A_1^C})$-$\PP_1$-a.e. 
        \item \label{it:PP^*} $\PP^*(A_2\triangle A_1)=0$
    \end{enumerate}
\end{proposition}
\Cref{it:S_1} states that $A_1, A_2$ are equivalent up to degeneracy if the `attacked' classifiers $A_1^\e,A_2^\e$ are equal $\PP_0$-a.e. 
\Cref{it:PP^*} further states that the adversarial Bayes classifiers $A_1$, $A_2$ are unique up to degeneracy if they are equal almost everywhere under the measure of optimal adversarial attacks. Notice that this proposition with $\e=0$ reads exactly like \cref{prop:Bayes_equivalencies}, despite the fact that uniqueness is defined differently for Bayes classifiers. We discuss the proof in the next subsection.

\Cref{prop:abs_cont_equivalencies} has several useful consequences for understanding degenerate sets, which we explore in \cref{sec:degenerate}. Furthermore, when $\PP\ll \mu$, equivalence up to degeneracy is in fact an equivalence relation.
\begin{proof}[Proof of \cref{thm:equivalence_up_to_degeneracy}]
    \Cref{it:PP^*} of \cref{prop:abs_cont_equivalencies} states that two adversarial Bayes classifiers $A_1$, $A_2$ are equivalent up to degeneracy iff $\one_{A_1}=\one_{A_2}$ $\PP^*$-a.e.  Since equality $\PP^*$-almost everywhere is an equivalence relation, this condition defines an equivalence  relation on adversarial Bayes classifiers. 
\end{proof}

\Cref{thm:equivalence_up_to_degeneracy} is false when $\PP$ is not absolutely continuous with respect to $\mu$:
\begin{example}\label{ex:non_equiv}
    Consider a distribution defined by $\PP_0=\frac 12 \delta_{-\e}$ and $\PP_1=\frac 12 \delta_\e$. If $0\in A$ then $S_\e(\one_A)(\e)=1$ and if $0\not \in A$ then $S_\e(\one_{A^C})(-\e)=1$. In either case, $R^\e(A)\geq 1/2$. The classifier $A=\Rset$ achieves adversarial classification error $1/2$ and therefore the adversarial Bayes risk is $R^\e_*=1/2$. The sets $\Rset^{\geq 0}$ and $\Rset^{>0}$ also achieve error 1/2 and thus are also adversarial Bayes classifiers. These two classifiers are equivalent up to degeneracy because they differ by one point. Furthermore, the classifiers $\Rset$ and $\Rset^{\geq 0}$ are equivalent up to degeneracy: if $D\subset \Rset^{<0}$, then $S_\e(\one_{\Rset^{\geq 0} \cup D})(-\e)=1$ while $S_\e(\one_{(\Rset^{\geq 0}\cup D)^C})(\e)=0$ and hence $R^\e(\Rset^{\geq 0}\cup D)=1/2$. However, if $D\subset (-2\e,0)$ then $R^\e(\Rset^{>0}\cup D)=1$ and thus $\Rset$ and $\Rset^{>0}$ cannot be equivalent up to degeneracy.

    In short--- the classifiers $\Rset^{>0}$ and $\Rset^{\geq 0}$ are equivalent up to degeneracy, the classifiers $\Rset^{\geq 0}$ and $\Rset$ are equivalent up to degeneracy, but $\Rset^{>0}$ and $\Rset$ are not equivalent up to degeneracy. Thus equivalence up to degeneracy cannot be an equivalence relation for this distribution.
\end{example}
However, \cref{it:S_0} and \cref{it:PP^*} of \cref{prop:abs_cont_equivalencies} are still equivalent when $\PP\not \ll \mu$, as are \cref{it:S_e_unique} and \cref{it:eta_*_meas zero} of \cref{thm:TFAE_equiv} (see \cref{lemma:non_abs_cont} in \cref{app:abs_cont_equivalencies_main} and \cref{prop:equiv_non_abs_cont} in \cref{app:TFAE_equiv_proof}). As the proof of \cref{thm:equivalence_up_to_degeneracy} relies only on \cref{it:PP^*}, one could use \cref{it:S_0} and \cref{it:PP^*} to define a notion of equivalence for adversarial Bayes classifiers even when $\PP\not \ll \mu$.

\subsection{Proof of \cref{prop:abs_cont_equivalencies}} \label{app:abs_cont_equivalencies_main}

First, the implication \cref{it:S_1} $\Rightarrow$ \cref{it:degeneracy_equiv} in \cref{prop:abs_cont_equivalencies} is a consequence of properties of supremums and \cref{lemma:union_intersection_adv_bayes}.

\begin{lemma}\label{lemma:intersection_union_a.e.}
    Let $A_1$ and $A_2$ be adversarial Bayes classifiers for which either $S_\e(\one_{A_1})=S_\e(\one_{A_2})$ $\PP_0$-a.e. or $S_\e(\one_{A_1^C})=S_\e(\one_{A_2^C})$-$\PP_1$-a.e.
    Then 
    \begin{equation}\label{eq:0_cup_cap}
        S_\e(\one_{A_1})=S_\e(\one_{A_2})=S_\e(\one_{A_1\cap A_2})= S_\e(\one_{A_1\cup A_2})  \quad \PP_0\text{-a.e.}  
    \end{equation}
    and
    \begin{equation}\label{eq:1_cup_cap}
        S_\e(\one_{A_1^C})=S_\e(\one_{A_2^C})=S_\e(\one_{(A_1\cap A_2)^C})= S_\e(\one_{(A_1\cup A_2)^C})\quad \PP_1\text{-a.e.}    
    \end{equation}
\end{lemma}
See \cref{app:intersection_union_a.e.} for a proof.

Next, we observe that \cref{it:S_1} of \cref{prop:abs_cont_equivalencies} is equivalent to a more convenient condition: 

\begin{lemma}\label{cor:0_1_a.e._equivalence}
    Let $A_1$ and $A_2$ be two adversarial Bayes classifiers. Then \cref{it:S_1} of \cref{prop:abs_cont_equivalencies} is equivalent to
    \begin{enumerate}[label=\arabic*')]
        \setcounter{enumi}{1}
        \item Both $S_\e(\one_{A_1})=S_\e(\one_{A_2})$-$\PP_0$-a.e. and $S_\e(\one_{A_2^C})=S_\e(\one_{A_1^C})$-$\PP_1$-a.e. 
    \end{enumerate}
\end{lemma}
 See \cref{app:intersection_union_a.e.} for a proof. 
Observe this lemma does not assume $\e>0$ or $\PP\ll\mu$.
This result proves that \cref{it:S_1} and \cref{it:PP^*} of \cref{prop:abs_cont_equivalencies} are always equivalent, even when $\PP\not\ll \mu$ or $\e=0$:

\begin{lemma}\label{lemma:non_abs_cont}
    Let $\e>0$, and let $(\PP_0^*,\PP_1^*)$ be any maximizer of $\cdl$. Define $\PP^*=\PP_0^*+\PP_1^*$. Then \cref{it:S_1} and \cref{it:PP^*} of \cref{prop:abs_cont_equivalencies} are equivalent for any two adversarial Bayes classifiers $A_1$ and $A_2$.
\end{lemma}

\begin{proof}
          Assume that $A_1$ and $A_2$ are both adversarial Bayes classifiers. 
  \Cref{lemma:union_intersection_adv_bayes} then implies that $A_1\cup A_2$, $A_1\cap A_2$ are both adversarial Bayes classifiers. \Cref{eq:sup_comp_slack_classification} of \cref{thm:complementary_slackness_classification} implies that
    \[\int S_\e(\one_{A_1\cup A_2})d\PP_0= \int \one_{A_1\cup A_2} d\PP_0^*=\int \one_{A_1\cap A_2} d\PP_0^*+\PP_0^*(A_1\triangle A_2)=\int S_\e(\one_{A_1\cap A_2}) d\PP_0 +\PP_0^*(A_1\triangle A_2)\]
    Since $S_\e(\one_{A_1\cap A_2})\leq  S_\e(\one_{A_1\cup A_2})$, the condition $\PP_0^*(A_1\triangle A_2)=0$ is equivalent to $S_\e(\one_{A_1\cap A_2})=  S_\e(\one_{A_1\cup A_2})$ $\PP_0$-a.e. Moreover, by \cref{lemma:intersection_union_a.e.}, $S_\e(\one_{A_1\cap A_2})=  S_\e(\one_{A_1\cup A_2})$ $\PP_0$-a.e. is equivalent to $S_\e(\one_{A_1})=S_\e(\one_{A_2})$ $\PP_0$-a.e. Therefore, by \cref{cor:0_1_a.e._equivalence}, $\PP_0^*(A_1\triangle A_2)=0$ is equivalent to \cref{it:S_0}.

    The same argument implies that $\PP_1^*(A_1\triangle A_2)=0$ is equivalent to \cref{it:S_0}. Lastly, $\PP^*(A_1\triangle A_2)=0$ is equivalent to $\PP_0^*(A_1\triangle A_2)=0$ and $\PP_1^*(A_1\triangle A_2)=0$.
\end{proof}


Lastly, to show that \cref{it:degeneracy_equiv} implies \cref{it:PP^*}, we apply the complementary slackness condition \cref{eq:sup_comp_slack_classification} of \cref{thm:complementary_slackness_classification} to argue that $D=A_1\triangle A_2$ has $\PP^*$-measure zero. We separately establish that $\PP^*(\interior (D))=0$ and $\PP^*(D\cap \partial D)=0$. 
For the latter, we prove that the boundary $\partial A$ is always a degenerate set for an adversarial Bayes classifier $A$ when $\PP\ll\mu$ and $\e>0$. Consequently:

    \begin{lemma}\label{lemma:closure_interior_adversarial_Bayes}
        Let $A$ be an adversarial Bayes classifier. If $\PP\ll \mu$ and $\e>0$, 
        then $A$, $\ov A$, and $\interior A$ are adversarial Bayes classifiers that are equivalent up to degeneracy.
    \end{lemma}

    Similarly, for $A$ optimal, the sets $(\interior A)^\e$, $A^\e$ and ${\ov A}^\e$ all have equal $\PP_0$ measure while $( (\interior A)^C)^\e$, $(A^C)^\e$, and $(\ov A^C)^\e$ have equal $\PP_1$-measure. 
    \begin{lemma}\label{lemma:boundary_difference}
        Assume $\PP\ll\mu$. If $A$ is any adversarial Bayes classifier and $\e>0$, then $\PP_0(A^\e)=\PP_0({\ov A}^\e)=\PP_0((\interior A)^\e)$ and $\PP_1((A^C)^\e)=\PP_1( ((\interior A)^C)^\e)=\PP_1( ({\ov A}^C)^\e)$.
    \end{lemma}
    
    See \cref{app:closure_interior_adversarial_Bayes} for a proofs of these results. The conditions $\e>0$ and $\PP\ll \mu$ are essential for both \cref{lemma:closure_interior_adversarial_Bayes,lemma:boundary_difference}.

    To show $\PP^*(\interior D)=0$ we study the interaction between open sets and the $\empty^{\e}$ operation.

    \begin{lemma}\label{lemma:open^e}
        Let $U$ be an open set and let $\QQ$ be the set of rational numbers. Further assume $\e>0$. Then $U^\e=(U\cap \QQ^d)^\e=(U\cap (\QQ^d)^C)^\e$.
    \end{lemma}

 See \cref{app:open^e} for a proof. The assumption $\e>0$ is essential for this argument. 
    
     Note that in the proof of \cref{prop:abs_cont_equivalencies} below the absolute continuity assumption is only used in the implication \cref{it:degeneracy_equiv} $\Rightarrow$ \cref{it:PP^*}; specifically, \cref{lemma:boundary_difference} is invoked to argue that the boundary of the adversarial Bayes classifier is a degenerate set.

    \begin{proof}[Proof of \cref{prop:abs_cont_equivalencies}]
  \Cref{lemma:non_abs_cont} states that \cref{it:PP^*} implies \cref{it:S_0}. It remains to show \cref{it:S_0} implies \cref{it:degeneracy_equiv} and \cref{it:degeneracy_equiv} implies \cref{it:PP^*}.

    
    

    \textbf{\cref{it:S_0} $\Rightarrow$  \cref{it:degeneracy_equiv}:}
    Assume that \cref{it:S_0} holds; then \cref{cor:0_1_a.e._equivalence} implies that both $S_\e(\one_{A_1})=S_\e(\one_{A_2})$ $\PP_0$-a.e. and $S_\e(\one_{A_1^C})=S_\e(\one_{A_2^C})$ $\PP_1$-a.e. \Cref{lemma:intersection_union_a.e.} implies than any set $A$ with $A_1\cap A_2\subset A\subset A_1\cup A_2$ satisfies 
    $S_\e(\one_{A_1})=S_\e(\one_A)$   $\PP_0$-a.e. and  
    $S_\e(\one_{A_1^C})=S_\e(\one_{A^C})$  $\PP_1$-a.e. Therefore $R^\e(A)=R^\e(A_1)$ so $A$ is also an adversarial Bayes classifier.

    \textbf{\Cref{it:degeneracy_equiv} $\Rightarrow$ \Cref{it:PP^*}:}
       Assume that for all $A$ satisfying $A_1\cap A_2\subset A\subset A_1\cup A_2$, the set $A$ is an adversarial Bayes classifier.
        Define $A_3=\interior(A_1\cap A_2)$, $A_4=\ov{A_1\cup A_2}$, and $D=A_3\triangle A_4$. \Cref{lemma:closure_interior_adversarial_Bayes} implies that $A_3$ and $A_4$ are adversarial Bayes classifiers equivalent to $A_1$ and $A_2$.  As $A_3\sqcup D\sqcup A_4^C=\Rset^d$, the boundary $\partial D$ is included in $\partial A_3\cup \partial A_4$.

        We split $D$ into four disjoint sets, $D_1=(\interior D) \cap \QQ^d$, $D_2=(\interior D )\cap (\QQ^d)^C$
        , 
        $D_3=D\cap \partial D\cap \partial A_3$, 
        and $D_4=D\cap \partial D\cap  \partial A_4-D_3$. Notice that these four sets satisfy $D=D_1\sqcup D_2\sqcup D_3\sqcup D_4$. Next, we will prove that each for these four sets has $\PP^*$-measure zero.
        
        Because $D$ is a degenerate set, the sets  $A_3 \cup D_1$, $A_3 \cup D_2 $, and $A_3\cup \interior D$ are all adversarial Bayes classifiers.
        However, \cref{lemma:open^e} implies that $D_1^\e=D_2^\e=(\interior D)^\e$ and therefore \cref{eq:indicator_supremum} implies that $ S_\e(\one_{A_3\cup D_1})=S_\e(\one_{A_3\cup \interior D})=S_\e(\one_{A_3\cup D_2})$. Because each of these sets is an adversarial Bayes classifier, \Cref{eq:sup_comp_slack_classification} of \cref{thm:complementary_slackness_classification} implies that $\PP_0^*(A_3\cup D_1)=\PP_0^*(A_3\cup \interior D)=\PP_0^*(A_3\cup D_2)$. As $D_1$ and $D_2$ are disjoint sets whose union is $\interior D$, it follows that $\PP_0^*(\interior D)=0$. Analogously, comparing $S_\e(\one_{(A_4 - D_1)^C}),S_\e(\one_{(A_4 - D_2)^C})$, and $S_\e(\one_{(A_4-\interior D)^C})$ results in $\PP_1^*(\interior D)=0$.

        Next we argue that $\PP^*(D_3)=0$. \cref{lemma:boundary_difference} implies $S_\e(\one_{A_3\cup D_3})=S_\e(\one_{A_3})$ $\PP_0$-a.e., and \cref{eq:sup_comp_slack_classification} of \cref{thm:complementary_slackness_classification} then implies that $\PP_0^*(A_3\cup D_3)=\PP_0^*(A_3)$. Thus $\PP_0^*(D_3)=0$ because $A_3$ and $D_3$ are disjoint. Similarly, \cref{lemma:boundary_difference} implies that $S_\e(\one_{A_3^C})=S_\e(\one_{(A_3\cup D_3)^C})=S_\e(\one_{A_3^C - D_3})$ $\PP_1$-a.e., and \cref{eq:sup_comp_slack_classification} of \cref{thm:complementary_slackness_classification} then implies that $\PP_1^*(A_3^C- D_3)=\PP_1^*(A_3^C)$. Thus $\PP_1^*(D_3)=0$ because $D_3\subset A_3^C$ are disjoint.

        Similarly, one can conclude that $\PP^*(D_4)=0$ by comparing $A_4$ and $A_4-D_4$.

\end{proof}

\subsection{Proof of \cref{thm:TFAE_equiv}}\label{app:TFAE_equiv}
In the following discussion, we assume  the adversarial Bayes classifier is unique up to degeneracy and discuss the relation between \cref{prop:abs_cont_equivalencies} and \cref{thm:TFAE_equiv}. First, \cref{it:S_0} of \cref{prop:abs_cont_equivalencies} implies \cref{it:S_e_unique} of \cref{thm:TFAE_equiv}, and \cref{lemma:union_intersection_adv_bayes} together with \cref{it:S_e_unique} imply \cref{it:S_0}. 
Next, we discuss the equivalence of \cref{it:PP^*} of \cref{prop:abs_cont_equivalencies} and to \cref{it:eta_*_meas zero} of \cref{thm:TFAE_equiv}.
First, to show \cref{it:eta_*_meas zero} $\Rightarrow$ \cref{it:PP^*}, notice that the complementary slackness condition in \cref{eq:complementary_slackness_necessary_classification} implies that 
\begin{equation}\label{eq:adv_eta^*_compare}
    \one_{\eta^*>1/2}\leq \one_A\leq \one_{\eta^*\geq 1/2} \quad \PP^*\text{-a.e.}    
\end{equation}

for any adversarial Bayes classifier $A$ and any maximizer $(\PP_0^*,\PP_1^*)$ of $\cdl$. Thus, if $\PP^*(\eta^*=1/2)=0$ then every adversarial Bayes classifier must satisfy $\one_A=\one_{\eta^*>1/2}$ $\PP^*$-a.e. and consequently, $\PP^*(A_1\triangle A_2)=0$ for any two adversarial Bayes classifiers $A_1$ and $A_2$.

It remains to prove that \cref{it:PP^*} implies \cref{it:eta_*_meas zero}.
To relate the quantities $\PP^*(A_1\triangle A_2)$ and $\PP^*(\eta^*=1/2)$, we show that there are adversarial Bayes classifiers $\hat A_1$, $\hat A_2$ which match $\{\eta^*>1/2\}$ and $\{\eta^*\geq 1/2\}$ $\PP^*$-a.e.
\begin{lemma}\label{lemma:hat_eta_adv_Bayes}
    There exists $\PP_0^*\in \Wball \e(\PP_0)$, $\PP_1^*\in \Wball \e(\PP_1)$ which maximize $\cdl$ and adversarial Bayes classifiers $\hat A_1$, $\hat A_2$ for which $\one_{\eta^*>1/2}=\one_{\hat A_1}$ $\PP^*$-a.e. and $\one_{\eta^*\geq 1/2}=\one_{\hat A_2}$ $\PP^*$-a.e., where $\PP^*=\PP_0^*+\PP_1^*$ and $\eta^*=d\PP_1^*/d\PP^*$.
\end{lemma}
See \cref{app:TFAE_equiv_lemmas} for a proof.
\Cref{it:S_e_unique} in conjunction with this lemma implies that $\PP^*(\hat A_2\triangle \hat A_1)=\PP^*(\eta^*=1/2)=0$ for the $\PP_0^*$, $\PP_1^*$ in \cref{lemma:hat_eta_adv_Bayes}. See \cref{app:TFAE_equiv} for a proof of \cref{thm:TFAE_equiv} and \cref{app:hat_eta_adv_Bayes} for a proof of \cref{lemma:hat_eta_adv_Bayes}. 

Next, we prove a stronger version of \cref{lemma:non_abs_cont} when there is a single equivalence class under equivalence up to degeneracy. Specifically, we show that \cref{it:S_e_unique,it:eta_*_meas zero} of \cref{thm:TFAE_equiv} are always equivalent, even without the assumptions $\e>0$ and $\PP\ll \mu$.
\begin{proposition}\label{prop:equiv_non_abs_cont}
Let $\e\geq 0$ and $\PP_0$, $\PP_1$ be any positive finite measures. Then \cref{it:S_e_unique,it:eta_*_meas zero} of \cref{thm:TFAE_equiv} are equivalent.
    
\end{proposition}

\begin{proof}
     \textbf{ \cref{it:S_e_unique} $\Rightarrow$ \cref{it:eta_*_meas zero}:} Let $\PP_0^*,\PP_1^*,\hat A_1,\hat A_2$ the the measures and adversarial Bayes classifiers described by \cref{lemma:hat_eta_adv_Bayes}.
Assume that $\PP_0(A_1^\e)=\PP_0(A_2^\e)$ for any two adversarial Bayes classifiers. Specifically, this relation holds for $\hat A_1$ and $\hat A_2$, and \cref{lemma:union_intersection_adv_bayes} further implies that $\PP_0((\hat A_1\cup \hat A_2)^\e)=\PP_0((\hat A_1\cap \hat A_2)^\e)$. Consequently, $\one_{(\hat A_1\cup \hat A_2)^\e}=\one_{(\hat A_1\cap \hat A_2)^\e}$ $\PP_0$-a.e. because $(\hat A_1\cap \hat A_2)^\e\subset (\hat A_1\cup \hat A_2)^\e$. As $\hat A_1^\e$ and $\hat A_2^\e$ are strictly between $(\hat A_1\cap \hat A_2)^\e$ and $(\hat A_1\cup \hat A_2)^\e$, one can conclude 
\[S_\e(\one_{\hat A_1})=\one_{\hat A_1^\e}=\one_{\hat A_2^\e}=S_\e(\one_{\hat A_2})\quad \PP_0\text{-a.e.}\] from \cref{eq:indicator_supremum}. 
Therefore, \cref{it:S_e_unique} implies \cref{it:S_0} of \cref{prop:abs_cont_equivalencies}. Consequently, \cref{lemma:non_abs_cont,lemma:hat_eta_adv_Bayes} imply that $\PP^*(\hat A_1\triangle \hat A_2)=\PP^*(\eta^*=1/2)=0$.

\textbf{ \cref{it:eta_*_meas zero} $\Rightarrow$ \cref{it:S_e_unique}:}  Assume there is a maximizer $(\PP_0^*,\PP_1^*)$ of $\cdl$ for which $\PP^*(\eta^*=1/2)=0$, where $\PP^*=\PP_0^*+\PP_1^*$ and $\eta^*=d\PP_1^*/d\PP^*$. Then \cref{eq:adv_eta^*_compare} and $\PP^*(\eta^*=1/2)=0$ imply that $\one_A=\one_{\eta^*>1/2}$ $\PP^*$-a.e. for any adversarial Bayes classifier $A$. Consequently, $\one_{A_1}=\one_{A_2}$ $\PP^*$-a.e. for any two adversarial Bayes classifiers $A_1, A_2$ or equivalently, $\PP^*(A_1\triangle A_2)=0$. \Cref{cor:0_1_a.e._equivalence,lemma:non_abs_cont} imply that $S_\e(\one_{A_1})=S_\e(\one_{A_2})$ $\PP_0$-a.e. and $S_\e(\one_{A_1^C})=S_\e(\one_{A_2^C})$ $\PP_1$-a.e., which implies \cref{it:S_e_unique}. 
\end{proof}

Finally, this result together with \cref{prop:abs_cont_equivalencies} proves \cref{thm:TFAE_equiv}.
 In the proof below, the assumption \(\PP \ll \mu\) is used in two places: in the implication \cref{it:degeneracy_equiv} \(\Rightarrow\) \cref{it:S_0}  of \cref{prop:abs_cont_equivalencies} (to establish \cref{it:unique_under_deg} \(\Rightarrow\) \cref{it:S_e_unique}), and in the implication \cref{it:PP^*} \(\Rightarrow\) \cref{it:degeneracy_equiv}  of \cref{prop:abs_cont_equivalencies} (to establish \cref{it:eta_*_meas zero} \(\Rightarrow\) \cref{it:unique_under_deg}).

\begin{proof}[Proof of \cref{thm:TFAE_equiv}]
\Cref{prop:equiv_non_abs_cont} states that \cref{it:S_e_unique} implies \cref{it:eta_*_meas zero}. It remains to show \cref{it:unique_under_deg} implies \cref{it:S_e_unique} and \cref{it:eta_*_meas zero} implies \cref{it:unique_under_deg}.

\textbf{ \cref{it:unique_under_deg}$\Rightarrow$\cref{it:S_e_unique}:} Assume that the the adversarial Bayes classifier is unique up to degeneracy. Then \cref{it:S_0} of \cref{prop:abs_cont_equivalencies} implies that $\PP_1(A_1^\e)=\PP_1(A_2^\e)$ for any two adversarial Bayes classifiers $A_1$ and $A_2$.



\textbf{ \cref{it:eta_*_meas zero} $\Rightarrow$ \cref{it:unique_under_deg}:}
   Assume that $\PP^*(\eta^*=1/2)=0$ for some $(\PP_0^*,\PP_1^*)$ that maximize $\cdl$, where $\PP^*=\PP_0^*+\PP_1^*$ and $\eta^*=d\PP_1^*/d\PP^*$. Then \cref{eq:adv_eta^*_compare} implies that $\one_{\eta^*>1/2}=\one_A$ $\PP^*$-a.e. for any adversarial Bayes classifier $A$. Thus if $\PP^*(\eta^*=1/2)=0$ then $\one_{A_1}=\one_{A_2}$ $\PP^*$-a.e. for any two adversarial Bayes classifiers $A_1$, $A_2$, or in other words, $\PP^*(A_1\triangle A_2)=0$. \Cref{it:PP^*} of \cref{prop:abs_cont_equivalencies} then implies that $A_1$ and $A_2$ are equivalent up to degeneracy. As these adversarial Bayes classifiers were arbitrary, the adversarial Bayes classifier is unique up to degeneracy.
\end{proof}

\subsection{The universal $\sigma$-algebra, measurability, and fundamental regularity results}\label{sec:fundamental_regularity}

We introduce another piece of notation to state our regularity results: define $A^{-\e}=((A^C)^\e)^C$. The set $A^\e$ consists of all points in $\Rset^d$ that can be moved into $A$ by a perturbation of size at most $\e$, while $A^{-\e}$ consists of points in $A$ that cannot be perturbed outside of $A$: 

         \begin{minipage}{0.5\linewidth}
         \begin{equation}\label{eq:e_interpretation}
        A^\e=\{\bx\colon \ov{B_\e(\bx)} \text{ intersects }A\}
         \end{equation}
         \end{minipage}
         \begin{minipage}{0.5\linewidth}
         \begin{equation}
              \label{eq:-e_interpretation}
        A^{-\e}=\{\bx: \ov{B_\e(\bx)}\subset A\}
         \end{equation}
    \end{minipage}
    \vspace{5pt}

See \cref{app:e_-e} for a proof.
Prior works \cite{AwasthiFrankMohri2021, BungertGarciaMurray2021} note that applying the $\e$, $-\e$ operations in succession can improve the regularity of an adversarial Bayes classifier and reduce the adversarial Bayes risk. Specifically: 
\begin{lemma}\label{lemma:e_-e_decrease_R^e}
    For any set $A$, $R^\e((A^{-\e})^\e)\leq R^\e(A)$ and $R^\e((A^\e)^{-\e})\leq R^\e(A)$.
\end{lemma}
See \cref{app:e_-e} for a proof.
Thus applying the $\empty^\e$ and $\empty^{-\e}$ operations in succession can only reduce the adversarial risk of a set.
In order to perform these regularizing operations, one must minimize $R^\e$ over a $\sigma$-algebra $\Sigma$ that is preserved by the $\empty^\e$ operation: in other words, one would wish that $A\in \Sigma$ implies $A^\e\in \Sigma$.

To illustrate this concern, \cite{PydiJog2021} demonstrate a Borel set $C$ for which $C^\e$ is not Borel measurable.
However, prior work shows that if $A$ is Borel, then $A^\e$ is measurable with respect to a larger $\sigma$-algebra called the \emph{universal $\sigma$-algebra} $\sU(\Rset^d)$. A set in the universal $\sigma$-algebra is referred to as \emph{universally measurable}.
Theorem~29 of \cite{FrankNilesWeed23minimax} states that 
\begin{theorem}\label{thm:univ_meas}
    If $A$ is universally measurable, then $A^\e$ is as well.
\end{theorem}

See \cref{app:measurability} for the definition of the universal $\sigma$-algebra $\sU(\Rset^d)$.

Thus, in order to guarantee the existence of minimizers to $R^\e$ with improved regularity properties, one could minimize $R^\e$ over the universal $\sigma$-algebra $\sU(\Rset^d)$. However, many prior papers such as \cite{FrankNilesWeed23consistency,PydiJog2019,PydiJog2021} study this minimization problem over the Borel $\sigma$-algebra. We show that these two approaches are equivalent:
\begin{theorem}\label{thm:borel_univ_equivalent}
    Let $\cB(\Rset^d)$ denote the Borel $\sigma$ algebra on $\Rset^d$. Then
    \[\inf_{A\in \cB(\Rset^d)} R^\e(A)=\inf_{A\in \sU(\Rset^d)} R^\e(A)\]
\end{theorem}
See \cref{app:measurability} for a proof. Furthermore, every set in $\sU(\Rset^d)$ differs from a Borel set by a subset of a null set in $\cB(\Rset^d)$, see \cite[Lemma~7.26]{BertsekasShreve96} (stated as \cref{lemma:univeral_to_borel_B.S.} in the appendix). Due to these results, in the remainder of the paper, we treat the minimization of $R^\e$ over $\sU(\Rset^d)$ and $\cB(\Rset^d)$ as interchangable. 

\subsection{Describing degenerate sets and proof of \cref{thm:high_dim_uniqueness regularity}}\label{sec:degenerate}
\Cref{prop:abs_cont_equivalencies} and fundamental properties of the $\empty^\e$ and $\empty^{-\e}$ operations imply several results on degenerate sets.

First, \cref{lemma:union_intersection_adv_bayes} implies that countable unions and intersections of adversarial Bayes classifiers are adversarial Bayes classifiers. \Cref{it:PP^*} of \cref{prop:abs_cont_equivalencies} then necessitates that countable unions and intersections preserve equivalence up to degeneracy. As a result: 
\begin{lemma}\label{lemma:union_degenerate}
    Let $\PP\ll \mu$. Then a countable union of degenerate sets is degenerate.
\end{lemma}
See \cref{app:union_degenerate} for a formal proof.

Next, using the regularizing $\empty^\e$ and $\empty^{-\e}$ operations, we study the relation between uniqueness up to degeneracy and regular adversarial Bayes classifiers.
First notice that $(A^{-\e})^\e$ is smaller than $A$ while $(A^{\e})^{-\e}$ is larger than $A$: 

\begin{lemma}\label{lemma:A_e_-e_inclusion}
    Let $A$ be any set. Then $(A^{-\e})^\e\subset A\subset (A^\e)^{-\e}$.
\end{lemma}

Furthermore, one can compare $S_\e(\one_A)$ and $S_\e(\one_{A^C})$ with $S_\e(\one_{(A^{-\e})^\e})$ and $S_\e(\one_{(A^\e)^{-\e}})$:

\begin{lemma}\label{lemma:A_e_-e_containments}
    For any set $A\subset \Rset^d$, the following set relations hold: $((A^\e)^{-\e})^\e=A^\e$, $((A^\e)^{-\e})^{-\e}\supset A^{-\e}$, $((A^{-\e})^\e)^{-\e}=A^{-\e}$, $((A^{-\e})^\e)^\e\subset A^\e$.
\end{lemma}

See \cref{app:e_-e} for proofs of \cref{lemma:A_e_-e_inclusion} and \cref{lemma:A_e_-e_containments}.
\Cref{lemma:A_e_-e_containments} then implies:

\begin{corollary}\label{cor:A_e_diff_degenerate}
    Assume $\PP\ll\mu$ and let $A$ be an adversarial Bayes classifier. Then $A, (A^\e)^{-\e}, and (A^{-\e})^\e$ are all equivalent up to degeneracy.
\end{corollary}
\begin{proof}
    \Cref{lemma:A_e_-e_containments} implies that $(A^{-\e})^\e$, $(A^\e)^{-\e}$ are both adversarial Bayes classifiers satisfying $S_\e(\one_A)=S_\e(\one_{(A^\e)^{-\e}})$ and $S_\e(\one_{A^C})=S_\e(\one_{ ((A^{-\e})^\e)^C})$. Therefore, when $\PP\ll \mu$, \cref{it:S_0} of \cref{prop:abs_cont_equivalencies} implies that $A$, $(A^{-\e})^\e)$, and $(A^\e)^{-\e}$ are all equivalent up to degeneracy.
\end{proof}

\Cref{thm:high_dim_uniqueness regularity} is an immediate consequence of \cref{cor:A_e_diff_degenerate}. Furthermore, \cref{cor:A_e_diff_degenerate} implies that ``small" components of $A$ and $A^C$ are degenerate sets: specifically, if $C$ is a component with $C^{-\e}=\emptyset$, then $C$ is contained in $(A^{\e})^{-\e}-(A^{-\e})^\e$.

\begin{proposition}\label{prop:degenerate_connected_component}
    Let $A$ be an adversarial Bayes classifier and let $C$ be a connected component of $A$ or $A^C$ with $C^{-\e}=\emptyset$. Then $C$ is contained in $(A^{\e})^{-\e}-(A^{-\e})^\e$, and thus the set 
    \begin{equation}\label{eq:components_union_set}
        \bigcup \bigg\{{C:\text{connected components} \text{ of $A$ or $A^C$ with } \text{$C^{-\e}=\emptyset$}} \bigg\}    
    \end{equation}
    is contained in a degenerate set of $A$.
\end{proposition}
See \cref{app:degenerate_connected_component} for a proof. This result has a sort of converse: A degenerate set $D$ must satisfy $\one_{D^{-\e}}=\one_{\emptyset}$ $\PP$-a.e:

\begin{lemma}\label{lemma:degenerate_set_-e}
    Assume that $\PP\ll \mu$ and let $D$ be a degenerate set for an adversarial Bayes classifier $A$. Then $\PP(D^{-\e})=0$.
\end{lemma}

See \cref{app:degenerate_set_-e} for a proof.
The adversarial classification risk heavily penalizes the boundary of a classifier. This observation suggests that if two connected components of a degenerate set are close together, then they must actually be included in a larger degenerate set. The $\empty^\e$ and $\empty^{-\e}$ operations combine to form this enlarging operation.

\begin{lemma}\label{lemma:enlarge_degenerate}
    Assume that $\PP\ll \mu$. If $D$ is a degenerate set for an adversarial Bayes classifier $A$, then $(D^\e)^{-\e}$ is as well.
\end{lemma}
\begin{proof}
    Let $A_2=A \cup (D^\e)^{-\e}$. Then $S_\e(\one_{A^C})\geq S_\e(\one_{A_2^C})$. We will show that $S_\e(\one_{A_2})=S_\e(\one_{A})$ $\PP_0$-a.e., which will then imply that $A_2$ is an adversarial Bayes classifier, and furthermore $A$ and $A_2$ are equivalent up to degeneracy by \cref{prop:abs_cont_equivalencies}. Notice that the set $A_2$ satisfies
        \[A\subset A_2\subset ((A\cup D)^\e)^{-\e}\]
    by \cref{lemma:non_abs_cont} and then \cref{lemma:A_e_-e_containments} implies that $A^\e \subset (A\cup (D^{\e})^{-\e})^\e \subset (A\cup D)^\e$. Because $D$ is a degenerate set, $A_3=A\cup D$ is an adversarial Bayes classifier equivalent to $A$ and thus 
    \cref{prop:abs_cont_equivalencies} implies that $\one_{A^\e}= \one_{(A\cup D)^\e}$-$\PP_0$-a.e. which in turn implies $\one_{A^\e}= \one_{A_2^\e}$-$\PP_0$-a.e.
\end{proof}

\section{The adversarial Bayes classifier in one dimension}\label{sec:necessary}
In this section, we assume that $d=1$. Recall that connected subsets of $\Rset$ are either intervals or single point sets. The length of an interval $I$ will be denoted $|I|$. When expressing a regular set as a union $\bigcup_{i=m}^M (a_i,b_i)$, we impose the convention that the indices $m,M$ satisfy $-\infty \leq m\leq M-1\leq +\infty$, where the case $m=M-1$ specifically accounts for the empty set. 

The arguments in this section frequently compare $b_i$ with $a_{i+1}$ or $a_i$ with $b_{i+1}$. To facilitate these comparisons, we adopt the following additional notation: when $\bigcup_{i=m}^M (a_i,b_i)$ represents a regular adversarial Bayes classifier and $M$ is finite, we define $a_{M+1}=+\infty$; similarly, if $m$ is finite, we define $b_{m-1}$ as $-\infty$.

\subsection{Regular adversarial Bayes classifiers---Proof of \cref{thm:adv_bayes_and_degenerate,thm:exists_regular}}\label{sec:regular_adv_bayes}

To start, every open regular set with $\e>0$ can be expressed in the desired form.
\begin{lemma}\label{lemma:^e_to_order}
    Let $A$ be a set that is the union of open intervals length at least $2\e$, for some $\e>0$. Then $A$ can be expressed as 
    \[A=\bigcup_{i=m}^M (a_i,b_i)\]
    with $a_i<b_i<a_{i+1}$ and $-\infty\leq m\leq M-1\leq +\infty$. 
\end{lemma}

See \cref{app:1d_necessary} for a proof. Next, notice that if $I$ is a connected component of $A$ or $A^C$ and $|I|<2\e$, then $I^{-\e}=\emptyset$. Thus the set of connected components of $A$ or $A^C$ of length strictly less than $2\e$ is contained in a degenerate set by \cref{prop:degenerate_connected_component}. 

However, if $|I|=2\e$, then $I^{-\e}$ could contain a single point: if $I=[x-\e,x+\e]$ then $I^{-\e}=\{x\}$ while $I^{-\e}=\emptyset$ if $I$ is not closed. Due to this observation, the set of connected components of $A$ and $A^C$ of length exactly $2\e$ is actually degenerate set as well. This observation enables the construction of two particularly useful adversarial Bayes classifiers that are equivalent to $A$ up to degeneracy.

    \begin{lemma}\label{lemma:Rset_split}
        Let $\e>0$ and $\PP_0,\PP_1\ll \mu$. Suppose $A$ is an adversarial Bayes classifier. Then there are adversarial Bayes classifiers $\td A_1,\td A_2$ satisfying $\td A_1\subset A\subset \td A_2$ which are equivalent to $A$ up to degeneracy and 
        \begin{equation}\label{eq:td_A_1_A_2}
            \td A_1=\bigcup_{i=m}^M (\td a_i,\td b_i),\quad \td A_2^C=\bigcup_{j=n}^N (\td e_j,\td f_j)    
        \end{equation}

        where the intervals $(\td a_i,\td b_i)$, $(\td e_i,\td f_i)$ satisfy $\td b_i-\td a_i>2\e$, $\td a_i<\td b_i<\td a_{i+1}$, $\td f_i-\td e_i>2\e$, and $\td e_i<\td f_i<\td e_{i+1}$. 
    \end{lemma}

    \begin{proof}[Proof of \cref{lemma:Rset_split}]
        \Cref{lemma:closure_interior_adversarial_Bayes} implies that $\interior A$ and $\ov A$ are both adversarial Bayes classifiers equivalent to $A$, and thus \cref{cor:A_e_diff_degenerate} implies that $D_1=((\interior A)^\e)^{-\e}- ((\interior A)^{-\e})^\e$ and $D_2=((\ov A)^\e)^{-\e}- ((\ov A)^{-\e})^\e$ are degenerate sets.  
        Consequently, the sets 
        $\td A_1=\interior A-D_1$, $\td A_2=\ov A\cup D_2$, and $A$ are all equivalent up to degeneracy. 
        
        The adversarial Bayes classifier $\interior A$ is an open set, and thus every connected component of $\interior A$ is open. Therefore, if $I$ is a connected component of $\interior A$ of length less than or equal $2\e$, then $I^{-\e}=\emptyset$ and \cref{prop:degenerate_connected_component} implies that $I\subset D_1$. Hence every connected component of $\td A_1$ has length strictly larger than $2\e$. 
        
        As $(\ov A)^C$ is an open set and $\td A_2^C=(\ov A)^C- D_2$, the same argument implies that every connected component of $\td A_2^C$ has length strictly larger than $2\e$. The claim on the ordering of the $a_i$, $b_i$ and $e_i$, $f_i$ is then a consequence of \cref{lemma:^e_to_order}.
    \end{proof}
    
    This result suffices to establish \cref{thm:adv_bayes_and_degenerate}; a brief proof sketch follows. The classifiers $\td A_1$ and $\td A_2$ have ``one-sided" regularity--- the connected components of $\td A_1$ and $\td A_2^C$ have length strictly greater than $2\e$. Next, we use these classifiers to construct a classifier $A'$ for which both $A'$ and $(A')^C$ have components larger than $2\e$.

    As $\td A_1\subset \td A_2$, the sets $\td A_1$ and $\td A_2^C$ are disjoint. Therefore, $\Rset$ equals the disjoint union
        \[\Rset= \td A_1\sqcup \td A_2^C \sqcup D.\] 
        Both $\td A_1$ and $\td A_2^C$ are a disjoint union of intervals of length greater than $2\e$, and thus the degenerate set $D=\td A_1^C\cap \td A_2$ must be a disjoint union of countably many intervals and isolated points. Notice that because $D$ is degenerate, the union of $\td A_1$ and an arbitrary measurable portion of $D$ is an adversarial Bayes classifier as well. To construct a regular adversarial Bayes classifier, we let $D_1$ be the connected components of $D$ that are adjacent to some open interval of $\td A_1$. The remaining components of $D$, $D_2=D-D_1$, must be adjacent to $\td A_2$. Therefore, if $A'=\td A_1\cup D_1$, the connected components of $A'$ and $(A')^C=\td A_2\cup D_2$ must have length strictly greater than $2\e$.

\begin{proof}[Proof of \cref{thm:adv_bayes_and_degenerate}]
        Let  $\td A_1\subset \td A_2$ be the adversarial Bayes classifiers defined in \cref{eq:td_A_1_A_2} of \cref{lemma:Rset_split}.
        Then $D=\td A_2-\td A_1$ is a degenerate set and thus
                \begin{equation}\label{eq:Rset_split}
            \Rset =D \sqcup \bigcup_{i=m}^M (\td a_i,\td b_i)\sqcup \bigcup_{i=n}^N (\td e_i,\td f_i)
        \end{equation}
        
        For each $i$, define 
        \[\hat a_i=\inf\{x: (x,\td b_i)\text{ does not intersect }\td A_2^C\}\]
        \[\hat b_i=\sup\{x: (\td a_i,x)\text{ does not intersect }\td A_2^C\}\]
        and let 
        \[ \hat A=\bigcup_{i=m}^M (\hat a_i,\hat b_i)\]
        Notice that $(\hat a_i,\hat b_i)\supset (\td a_i,\td b_i)$ so that $\hat b_i-\hat a_i>2\e$. Similarly, by the definition of the $\hat a_i$ and $\hat b_i$, every interval $(\hat b_i,\hat a_{i+1})$ with $i,i+1\in [m,M]$ must include some $(\td e_j,\td f_j)$ and thus $\hat a_{i+1}-\hat b_i>2\e$.
            As $\hat A\triangle A\subset D$, the set $\hat A$ is an adversarial Bayes classifier equivalent to $A$. 
            
        Next, we will show that any two intervals $(\hat a_k,\hat b_k)$, $(\hat a_p, \hat b_p)$ are either disjoint or equal. Assume that $(\hat a_k,\hat b_k)$ and $(\hat a_p, \hat b_p)$ are not disjoint, and thus intersect at a point $x$. By the definition of $\hat a_k$ and $\hat b_k$, the interval $(x,\hat b_k)$ does not intersect $\td A_2^C$. Thus the definition of $\hat b_p$ implies that $\hat b_p\geq \hat b_k$. Reversing the roles of $\hat b_p$ and $\hat b_k$, one can then conclude that $\hat b_p=\hat b_k$. One can show that $\hat a_p=\hat a_k$ via a similar argument. Thus we can choose $(a_i, b_i)$ to be unique disjoint intervals for which 
        \[\bigsqcup_{i=k}^K ( a_i, b_i)=\bigcup_{i=m}^M (\hat a_i,\hat b_i).\]

        Finally, \cref{lemma:^e_to_order} guarantees that the intervals can be ordered so that $a_i<b_i<a_{i+1}$ for all $i$.
    \end{proof}

Next, \cref{thm:exists_regular} is a consequence of the fact that the adversarial risk of $A=\bigcup_{i=m}^M (a_i,b_i)$ equals \eqref{eq:risk_on_intervals} when $A$ is regular.

\begin{proof}[Proof of \cref{thm:exists_regular}]

    Because $b_i-a_i>2\e$ and $a_i-b_{i-1}>2\e$, we can treat $R^\e(A)$ as a differentiable function of $a_i$ on a small open interval around $a_i$ as described by \cref{eq:risk_on_intervals}. The first order necessary conditions for a minimizer then imply the first relation of \eqref{eq:first_order_necessary} and the second order necessary conditions for a minimizer then imply the first relation of \eqref{eq:second_order_necessary}. The argument for $b_i$ is analogous.
\end{proof}

\subsection{Degenerate sets in one dimension---proof of \cref{thm:1d_degenerate}}\label{sec:degenerate_1d}
First, every component of $A$ or $A^C$ with length less than or equal to $2\e$ must be degenerate. In comparison, notice that this statement is strictly stronger than \cref{prop:degenerate_connected_component}.

\begin{lemma}\label{lemma:small_components_degenerate_1d}
    If a connected component $C$ of $A$ or $A^C$ has length less than or equal to $2\e$, then $C$ is degenerate.
\end{lemma}

\begin{proof}
       Let $A$ be an adversarial Bayes classifier and let $\td A_1$ and $\td A_2$ be the two equivalent adversarial Bayes classifiers of \cref{lemma:Rset_split}. Because every connected component of component of $\td A_1$ has length strictly larger than $2\e$, the connected components of $A$ of length less than or equal to $2\e$ must be included in the degenerate set $A-\td A_1$. Similarly, the connected components of $A^C$ of length less than or equal to $2\e$ are included in $\td A_2^C-A^C$, which is a degenerate set.
\end{proof}
Conversely, the length of a degenerate interval contained in $\supp \PP$ is at most $2\e$.
\begin{corollary}\label{cor:degenerate_size}
    Let $\PP\ll \mu$. Assume that $I\subset \supp \PP^\e$ is a degenerate interval for an adversarial Bayes classifier $A$. Then $|I|\leq 2\e$.
\end{corollary}

\begin{proof}
    \Cref{lemma:degenerate_set_-e} implies that if $I$ is a degenerate interval then $\PP(I^{-\e})=0$. Because $I$ is an interval, the set $I^{-\e}$ is either empty, a single point, or an interval. Next, observe that any interval contained in $(\supp \PP)^\e$ satisfies $I^{-\e}\subset \supp \PP$. As $I\subset \supp \PP$ and every interval larger than a single point has positive measure under $\mu$, it follows that $I^{-\e}$ is at most a single point and thus $|I|\leq 2\e$.
\end{proof}

This result is then sufficient to prove the fourth bullet of \cref{thm:1d_degenerate}.
To start:

\begin{proposition}
\label{lemma:degenerate_1d_eta_0_1}
    Let $\PP\ll \mu$ and let $A$ be an adversarial Bayes classifier. If $\supp \PP$ is an interval and the adversarial Bayes classifier $A$ has a degenerate interval $I$ contained in $\supp \PP^\e$, then $\eta(x)\in \{0,1\}$ on a set of positive measure.
\end{proposition}
A formal proof is provided in \cref{app:degenerate_1d_eta_0_1}, we sketch the main ideas below, assuming that $I$ is a degenerate interval contained in $\interior (\supp \PP^{-\e})$.
One can then find a `maximal' degenerate interval $J=[d_3,d_4]$ containing $I$ inside $\supp \PP$, in the sense that if $J'\supset I$ is a degenerate interval and $J\subset J'$ then $J'=J$. \Cref{cor:degenerate_size} implies that $|J|\leq 2\e$ and \cref{lemma:enlarge_degenerate} implies that $J$ is of distance strictly more than $2\e$ from any other degenerate set. Thus the intervals $[d_3-\e,d_3)$, $(d_4,d_4+\e]\subset \supp \PP$ do not intersect a degenerate subset of $A$, and these intervals must be entirely contained in $A$ or $A^C$ due to \cref{lemma:small_components_degenerate_1d}. Thus one can compute the difference $R^\e(A\cup J)-R^\e(A-J)$ under four cases: 1) $[d_3-\e,d_3)\subset A$, $(d_4,d_4+\e]\subset A$; 2) $[d_3-\e,d_3)\subset A$, $(d_4,d_4+\e]\subset A^C$; 3) $[d_3-\e,d_3)\subset A^C$, $(d_4,d_4+\e]\subset A$; 4) $[d_3-\e,d_3)\subset A^C$, $(d_4,d_4+\e]\subset A^C$.

In each case, this difference results in $\int_{I'} p_1(x)dx=0$ or $\int_{I'} p_0(x)dx=0$ on some interval $I'\subset \interior \supp \PP$, which implies either $\eta=1$ or $\eta=0$, respectively, on a set of positive measure.

\Cref{lemma:degenerate_1d_eta_0_1} implies that if $\PP(\eta\in\{0,1\})=0$, any degenerate set for a regular adversarial Bayes classifier $A$ must be contained in $\partial A\cup \ov{(\supp \PP^\e)^C}$. Next, in order to show the fourth bullet of \cref{thm:1d_degenerate}, we show that this set is in fact degenerate. The proof is presented below, with technical details deferred to \cref{app:identify_degenerate_eta_0_1}. First, if $A$ is regular, identifying a degenerate set is straightforward. 
\begin{lemma}\label{lemma:max_degenerate_set}
    Assume that $\PP\ll\mu$ and let $A$ be a regular adversarial Bayes classifier. Then the set $\ov{(\supp \PP^\e)^C}\cup \partial A$ is degenerate for $A$.
\end{lemma}
See \cref{app:identify_degenerate_eta_0_1} for a proof.
Next, we argue that if $\PP(\eta=0\text{ or }1)=0$, the set $\ov{(\supp \PP^\e)^C}\cup \partial A$ is a maximal degenerate set. If $A$ is a regular adversarial Bayes classifier and $D\subset \interior (\supp \PP^\e)$ is a degenerate set which contains two points $x\leq y$ at most $2\e$ apart, then \cref{lemma:enlarge_degenerate} implies that $[x,y]\subset (D^{\e})^{-\e}$ is degenerate, which would contradict \cref{lemma:degenerate_1d_eta_0_1}. Thus $D\cap \interior (\supp \PP^\e)$ must be comprised of points that are strictly more than $2\e$ apart. However, if $x\in D$ is more than $2\e$ from any point in $\partial A$, then one can argue that $R^\e(A-\{x\})-R^\e(A)>0$ if $x\in A$ and $R^\e(A\cup \{x\})-R^\e(A)>0$ if $x\not \in A$. Thus if $D$ is a degenerate set is disjoint from $\ov{(\supp \PP^\e)^C}$, then $D$ must be contained in $\partial A$. 

\begin{lemma}\label{lemma:fourth_bullet_for_regular}
    Assume that $\PP\ll\mu$, $\PP(\eta=0\text{ or }1)=0$, and $\supp \PP$ is an interval. Then if $D$ is a degenerate set for a regular adversarial Bayes classifier $A$, then $D\subset \ov{(\supp \PP^\e)^C}\cup \partial A$.
\end{lemma}
\begin{proof}
        Let $D$ be a degenerate set disjoint from $\ov{(\supp \PP^\e)^C}$, or equivalently $D\subset \interior(\supp \PP^{\e})$. We will show that $D\subset \partial A$.
    First, we use a proof by contradiction to argue that the points in $D\cup \partial A$ are strictly greater than $2\e$ apart. 
    
    If $\partial A$ and $D$ are both degenerate, \cref{lemma:union_degenerate} implies that $D\cup \partial A$ is degenerate as well. For contradiction, assume that $x< y$ are two points in $(D\cup \partial A)\cap \interior(\supp \PP^\e)$ with $y-x\leq 2\e$. Then \cref{lemma:enlarge_degenerate} implies that $[x,y]\subset ((D\cup \partial A)^\e)^{-\e}$ is a degenerate interval contained in $\interior(\supp \PP)$. This statement contradicts \cref{lemma:degenerate_1d_eta_0_1}. Therefore, $D\cup \partial A$ is comprised of points that are more than $2\e$ apart.

   Next, we will show that a degenerate set cannot include any points in $\interior (\supp \PP^\e)$ which are more than $2\e$ from $\partial A$. Let $z$ be any point in $\interior (\supp \PP^\e)$ that is strictly more than $2\e$ from $\partial A$. Assume first that $z\in A$. Then 
    \[R^\e(A-\{z\})-R^\e(A)=\int_{z-\e}^{z+\e}\eta(x) d\PP\]
    However, if $z\in \interior (\supp \PP^\e)$ then $(z-\e,z+\e)$ intersects $\supp \PP$ and thus has positive measure under $\PP$. As $\eta(x)>0$ on $\supp \PP$, one can conclude that $R^\e(A-\{z\})-R^\e(A)>0$.  Similarly, if $z\in A^C$, then one can show that $R^\e(A\cup \{z\})-R^\e(A)>0$. Therefore $z$ cannot be in any degenerate set.

    In summary: $D\cup \partial A$ is comprised of points that are more than $2\e$ apart, but no more than $2\e$ from $\partial A$. Therefore, one can conclude that $D\subset \partial A$.
\end{proof}
A technical argument extends \cref{lemma:fourth_bullet_for_regular} to all adversarial Bayes classifiers by comparing the boundary of a given adversarial Bayes classifier $A$ to the boundary of an equivalent regular adversarial Bayes classifier $A_r$.

\begin{corollary}\label{cor:fourth_bullet}
    Assume that $\PP\ll\mu$, $\PP(\eta=0\text{ or }1)=0$, and $\supp \PP$ is an interval. Then if $D$ is a degenerate set for any adversarial Bayes classifier $A$, then $D\subset \ov{(\supp \PP^\e)^C}\cup \partial A$.
\end{corollary}
See \cref{app:identify_degenerate_eta_0_1} for a proof.

Combining previous results then proves \cref{thm:1d_degenerate}:
    the first bullet of \cref{thm:1d_degenerate} is established in \cref{prop:degenerate_connected_component,lemma:small_components_degenerate_1d}, the second in \cref{cor:degenerate_size}, the third in \cref{lemma:union_degenerate}, and the fourth in \cref{cor:fourth_bullet}.

\Cref{lemma:degenerate_1d_eta_0_1} and \cref{cor:fourth_bullet} are false when $\supp \PP$ is not an interval.

\begin{figure}
    \centering
    \includegraphics[width=0.30\textwidth]{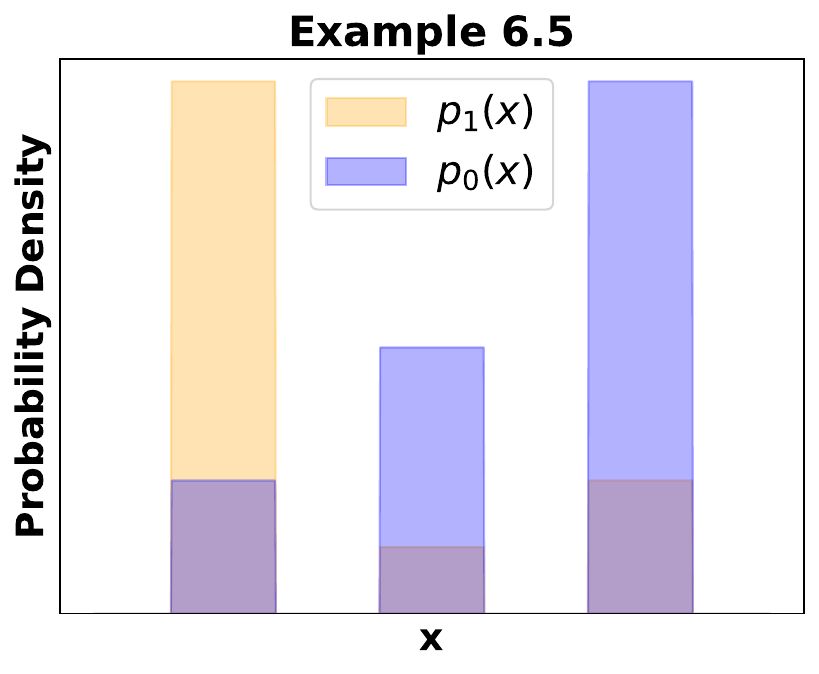}
    \caption{The distribution of \cref{ex:deg_eta_0_1_counterexample}.}
    \label{fig:det_eta_0_1_counterexample}
\end{figure}

\begin{example}\label{ex:deg_eta_0_1_counterexample}
    Consider a probability distribution for which 
    \[p_1(x)=\begin{cases}
        \frac 8 {25 \e} &\text{if }-\frac{5}2 \e\leq x \leq -\frac{3}2\e\\
        \frac 1 {25 \e}&\text{if }-\frac 12 \e \leq x\leq +\frac 12 \e\\
        \frac 2 {25 \e}&\text{if }+\frac 32 \e \leq x\leq +\frac 52 \e\\
        0&\text{otherwise}
    \end{cases}\quad 
    p_0(x)=\begin{cases}
        \frac 2{25 \e}& \text{if } -\frac 52 \e \leq x\leq -\frac 32 \e \\
        \frac 4 {25 \e} &\text{if }-\frac 12 \e\leq x\leq +\frac 12\e \\
        \frac{8}{25\e} &\text{if } +\frac 32 \e \leq x\leq +\frac 52 \e\\
        0&\text{otherwise}
    \end{cases}\]
    See \cref{fig:det_eta_0_1_counterexample} for an illustration. There are no solutions $x$ to the necessary conditions 
    \cref{eq:first_order_necessary} within $\supp \PP^\e$ at which $p_0$ is continuous at 
    $x\pm \e$ and $p_1$ 
    continuous at $x\mp \e$. Thus the only possible values for the $a_i$s and $b_i$s within $\supp \PP^\e$ are $\{-\frac 72\e, -\frac 52\e, -\frac 32 \e, -\frac 12 \e, +\frac 12 \e, +\frac 32 \e, +\frac 52 \e, +\frac 72\e\}$. By comparing the risks of all adversarial Bayes classifiers with endpoints in this set, one can show that $(-\infty,-\frac 12 \e)$ is an adversarial Bayes classifier. At the same time, $R^\e((-\infty,-\frac 12 \e)\cup S)=R^\e((-\infty,-\frac 12\e))$ for any subset $S$ of $[-\frac 12 \e,+\frac 12\e]$. Thus $[-\frac 12 \e, +\frac 12 \e]$ is a degenerate set, but $\eta(x)=\frac 15$ on $[-\frac 12 \e,+\frac 12 \e]$. See \cref{app:deg_eta_0_1_counterexample} for details.

\end{example}

\subsection{Regularity as $\epsilon$ Increases---Proof of \cref{thm:adv_bayes_increasing_e}}\label{sec:increasing_e}

Let $A_1$ and $A_2$ be two regular adversarial Bayes classifiers corresponding to perturbation radiuses $\e_1$ and $\e_2$ respectively.
 Notice that the adversarial classification risk in \cref{eq:adv_classification_risk_set} pays a penalty of $1$ within $\e$ of each $a_i$ and $b_i$. This consideration suggests that as $\e$ increases, there should be fewer transitions between the two classes. 
The key observation is that so long as $A_1$ is non-trivial, no connected component of $A_2$ should contain a connected component of $A_1^C$ and no connected component of $A_2^C$ should contain a connected component of $A_1$.

\begin{lemma}\label{lemma:inclusion_basic}
    Assume that $\PP\ll \mu$ is a measure for which $\supp \PP$ is an interval $I$ and $\PP(\eta(x)=0\text { or } 1)=0$. Let $A_1=\bigcup_{i=m}^M (a_i^1,b_i^1)$ and $A_2=\bigcup_{j=n}^N (a_j^2,b_j^2)$ be two regular adversarial Bayes classifiers corresponding to perturbation sizes $0< \e_1<\e_2$. 
    \begin{itemize}
        \item If both $\Rset$ and $\emptyset$ are adversarial Bayes classifiers for perturbation radius $\e_1$, then both $\Rset$ and $\emptyset$ are adversarial Bayes classifiers for perturbation radius $\e_2$.
        \item Assume that $\Rset$ and $\emptyset$ are not both adversarial Bayes classifiers for perturbation radius $\e_1$. Then for each interval $(a_i^1,b_{i}^1)$, the set $(a_i^1,b_{i}^1)\cap \interior(I^{\e_1})$ cannot contain any non-empty $(b_j^2,a_{j+1}^2)\cap \interior(I^{\e_1})$ and for each interval $(b_i^1,a_{i+1}^1)$, the set $(b_i^1,a_{i+1}^1)\cap \interior(I^{\e_1})$ cannot contain any non-empty $(a_j^2,b_{j}^2)\cap \interior(I^{\e_1})$.
    \end{itemize}

    
\end{lemma}

\Cref{ex:gaussians_equal_variances} demonstrates the exception to the first bullet--- when $\e\geq (\mu_1-\mu_0)/2$, both $\Rset$ and $\emptyset$ are adversarial Bayes classifiers. The intersection with $\interior(I^{\e_1})$ is required because $(\interior(I^{\e_1}))^C$ is a degenerte set for $A_1$. The claim in the second bullet would be contradicted by choosing $A_2=\Rset$, $A_1=\emptyset$, and $\e_2>\e_1\geq (\mu_1-\mu_0)/2$. \Cref{lemma:inclusion_basic} provides a way to pair each component of $A_2 \cap I^{\e_2} $ to a component of $ A_1\cap I^{\e_1}$ and each component of $A_2^C\cap I^{\e_2}$ to a component of $A_1\cap I^{\e_1}$, thereby proving \cref{thm:adv_bayes_increasing_e} when $\e_2>\e_1>0$. The challenge in extending this result to $\e_1=0$ is that there may not exist a regular Bayes classifier, as demonstrated in \cref{ex:counterexample_Bayes}. \Cref{lemma:inclusion_basic} is rephrased in \cref{lemma:inclusion_basic_0} of \cref{app:increasing_e} to circumvent this technicality.

To prove \cref{lemma:inclusion_basic}, notice that if $A_2=\bigcup_{i=1}^M(a_i^2,b_i^2)$ is a regular adversarial Bayes classifier and $(a_j^2,b_j^2)\subset I^{-\e_2}$, then $R^{\e_2}(A_2- (a_j^2,b_j^2))\geq R^{\e_2}(A_2)$ which is equivalent to 
 \[ 0\leq  \int_{a_j^2-\e_2}^{b_j^2+\e_2} p_1dx-\left( \int_{a_j^2-\e_2}^{a_j^2+\e_2} p dx +\int_{a_j^2+\e_2}^{b_j^2-\e_2} p_0 dx+\int_{b_j^2-\e_2}^{b_j^2+\e_2} p dx\right)=\int_{a_j^2+\e_2} ^{b_j^2-\e_2} p_1(x)dx-\int_{a_j^2-\e_2}^{b_j^2+\e_2} p_0(x)dx\]
As $p_0,p_1$ are non-zero on $\supp \PP$, replacing $\e_2$ with $\e_1$ in this equation would strictly increase the first integral and decrease the second, thereby strictly increasing the entire expression.

Thus, if $(a_j^2-\e_1, b_j^2+\e_1)\subset A_1^C\cap I$, this calculation would imply that $R^{\e_1}(A_1 \cup (a_i^2,b_i^2))< R^{\e_1}(A_1)$, which would contradict the fact that $A_1$ is an adversarial Bayes classifier. Similar but more technical calculations performed in \cref{app:increasing_e} show that if $(a_i^2,b_i^2)\subset A_1^C\cap \interior(I^{\e_1})$ then $R^{\e_1}(A_1 \cup (a_i^2,b_i^2))< R^{\e_1}(A_1)$ and so $A_1$ cannot be an adversarial Bayes classifier.

\section{Related works}\label{sec:related_works}

 Prior work analyzes several variations of our setup, such as perturbations in open balls \cite{BungertGarciaMurray2021}, alternative perturbation sets \cite{BhagojiCullinaMittalji2019lower}, attacks using general Wasserstein $p$-metrics \cite{trillosMurray2022,TrillosJacobsKim22}, minimizing $R^\e$ over Lebesgue measurable sets \cite{PydiJog2021}, the multiclass setting \cite{TrillosJacobsKim22}, and randomized classifiers \cite{pydi2023role,TrillosJacobsKim22}. Due to the plethora of attacks present in the literature, this paper contains proofs of intermediate results that appear in prior work. Understanding the uniqueness of the adversarial Bayes classifier in these other contexts remains an open question. Establishing a notion of uniqueness for randomized classifiers in the adversarial context is particularly interesting, as randomized classifiers are essential in calculating the minimal possible error in adversarial multiclass classification \cite{TrillosJacobsKim22} but not binary classification \cite{pydi2023role}. 

 Prior work \cite{AikenBrewerMurray2022, BhagojiCullinaMittalji2019lower, PydiJog2019}
 adopts a different method for identifying adversarial Bayes classifiers for various distributions. To prove a set is an adversarial Bayes classifier, \cite{BhagojiCullinaMittalji2019lower} first show a strong duality result $\inf_A R^\e(A)=\sup_\gamma \tilde D(\gamma)$ for some dual risk $\tilde D$ on the set of couplings between $\PP_0$ and $\PP_1$. Subsequently, \cite{AikenBrewerMurray2022,BhagojiCullinaMittalji2019lower,PydiJog2019} exhibit a set $A$ and a coupling $\gamma$ for which the adversarial risk of $A$ matches the dual risk of $\gamma$, and thus $A$ must minimize the adversarial classification risk. This approach involves solving the first order necessary conditions \cref{eq:first_order_necessary}, and \cite{AikenBrewerMurray2022} relies on a result of \cite{trillosMurray2022} which states that these relations hold for sufficiently small $\e$ under reasonable assumptions.  In contrast, this paper uses equivalence up to degeneracy to show that it suffices to consider sets with enough regularity for the first order necessary conditions to hold; and the solutions to these conditions typically reduce the possibilities for the adversarial Bayes classifier to a finite number of sets.

Prior work on regularity \cite{AwasthiFrankMohri2021, BungertGarciaMurray2021} prove the existence of adversarial Bayes classifiers with one sided tangent balls. \Cref{thm:high_dim_uniqueness regularity} states that each equivalence class under equivalence up to degeneracy has a representative with this type of regularity. Furthermore, results of \cite{AikenBrewerMurray2022} imply that under reasonable assumptions, one can choose adversarial Bayes classifiers $A(\e)$ for which $\comp(A(\e))+\comp(A(\e)^C)$ is always decreasing in $\e$. Specifically, they show that as $\e$ increases, the only possible discontinuous changes in $A(\e)$ are merged components, deleted components, or a endpoint of a component changing discontinuously in $\e$. This statement and \cref{lemma:inclusion_basic} are complementary results; neither one implies the other.


\section{Conclusion}
We defined a new notion of uniqueness for the adversarial Bayes classifier, which we call \emph{uniqueness up to degeneracy}. The concept of uniqueness up to degeneracy produces a method for calculating the adversarial Bayes classifier for a reasonable family of distributions in one dimension, and assists in understanding their regularity properties. We hope that the theoretical insights in this paper will assist in the development of algorithms for robust learning.  
\section*{Acknowledgments}
Natalie Frank was supported in part by the Research Training Group in Modeling and Simulation funded by the National Science Foundation via grant RTG/DMS – 1646339 NSF grant DMS-2210583 and grants DMS-2210583, CCF-1535987, IIS-1618662. Special thanks to Jonathan Niles-Weed, Pranjal Awasthi, and Mehryar Mohri for helpful conversations.

\bibliographystyle{siamplain}
\bibliography{references}
\pagebreak
\appendix

\ifthenelse{\boolean{appendixmode}}{\section*{Contents of the appendix}}{\section*{Contents of the supplement}}
\begin{itemize}[label={}, itemsep=0.5em]
    \item \textbf{\hyperref[app:app_outline]{\Cref*{app:app_outline}: Organization of the \ifthenelse{\boolean{appendixmode}}{appendix}{supplement}}} \dotfill \pageref{app:app_outline}

    \item \textbf{\hyperref[app:background_main_deferred]{\Cref*{app:background_main_deferred}: Deferred arguments from \cref*{sec:Background,sec:main_results}}} \dotfill \pageref{app:background_main_deferred}
    
    \begin{itemize}[label={}]
        \item\hyperref[app:Bayes_equivalencies]{\Cref*{app:Bayes_equivalencies}: Proof of \cref*{prop:Bayes_equivalencies}} \dotfill \pageref{app:Bayes_equivalencies}
        
        \item\hyperref[app:union_intersection_adv_bayes]{\Cref*{app:union_intersection_adv_bayes}: Proof of \cref*{lemma:union_intersection_adv_bayes}} \dotfill \pageref{app:union_intersection_adv_bayes}
        
        \item \hyperref[app:W_infty]{\Cref*{app:W_infty}: Complementary slackness-- proof of \cref*{thm:complementary_slackness_classification}} \dotfill \pageref{app:W_infty}
    \end{itemize}
    
    \item\hyperref[app:counterexample_Bayes]{\textbf{\Cref*{app:counterexample_Bayes}: Deferred arguments from \cref*{sec:main_results}---\tocbreak\Cref*{ex:counterexample_Bayes} details}} \dotfill \pageref{app:counterexample_Bayes}

    \item \textbf{\hyperref[app:abs_cont_equivalencies]{\Cref*{app:abs_cont_equivalencies}: Deferred proofs from \cref*{app:abs_cont_equivalencies_main}}} \dotfill \pageref{app:abs_cont_equivalencies}

    \begin{itemize}[label={}]
        \item \hyperref[app:intersection_union_a.e.]{\Cref*{app:intersection_union_a.e.}: Proof of \cref*{cor:0_1_a.e._equivalence,lemma:intersection_union_a.e.}} \dotfill \pageref{app:intersection_union_a.e.}

        \item \hyperref[app:closure_interior_adversarial_Bayes]{\Cref*{app:closure_interior_adversarial_Bayes}: Proof of \cref*{lemma:closure_interior_adversarial_Bayes,lemma:boundary_difference}} \dotfill \pageref{app:closure_interior_adversarial_Bayes}

        \item \hyperref[app:open_carrot_e]{\Cref*{app:open^e}: Proof of \cref*{lemma:open^e}} \dotfill \pageref{app:open^e}
    \end{itemize}

    \item \textbf{\hyperref[app:TFAE_equiv_proof]{\Cref*{app:TFAE_equiv_proof}: Deferred proofs from \cref*{app:TFAE_equiv}---{\tocbreak}proof of \cref*{lemma:hat_eta_adv_Bayes}}} \dotfill \pageref{app:TFAE_equiv_proof}

    \item \textbf{\hyperref[app:e_-e]{\Cref*{app:e_-e}: More about the $\empty^\e$, $\empty^{-\e}$, and $S_\e$ operations}} \dotfill \pageref{app:e_-e}

    \item \textbf{\hyperref[app:measurability]{\Cref*{app:measurability}: Measurability}} \dotfill \pageref{app:measurability}
    \begin{itemize}[label={}]
         \item \hyperref[app:define_universal_sigma_algebra]{\Cref*{app:define_universal_sigma_algebra}: Defining the universal $\sigma$-algebra} \dotfill \pageref{app:define_universal_sigma_algebra}

         \item \hyperref[app:borel_univ_equivalent]{\Cref*{app:borel_univ_equivalent}: Proof of \cref*{thm:borel_univ_equivalent}} \dotfill \pageref{app:borel_univ_equivalent}

         \item \hyperref[app:univ_borel_deferred_lemmas]{\Cref*{app:univ_borel_deferred_lemmas}: Proofs of \cref*{lemma:borel_to_univ_W_infty,cor:cdl_univ_borel,cor:S_e_and_W_inf_universal}} \dotfill \pageref{app:univ_borel_deferred_lemmas}
    \end{itemize}
    \item \textbf{\hyperref[app:degenerate]{\Cref*{app:degenerate}: Deferred proofs from \cref*{sec:degenerate}}} \dotfill \pageref{app:degenerate}
    \begin{itemize}[label={}]
         \item \hyperref[app:union_degenerate]{\Cref*{app:union_degenerate}: Proof of \cref*{lemma:union_degenerate}} \dotfill \pageref{app:union_degenerate}

         \item \hyperref[app:degenerate_connected_component]{\Cref*{app:degenerate_connected_component}: Proof of \cref*{prop:degenerate_connected_component}} \dotfill \pageref{app:degenerate_connected_component}

         \item \hyperref[app:degenerate_set_-e]{\Cref*{app:degenerate_set_-e}: Proof of \cref*{lemma:degenerate_set_-e}} \dotfill \pageref{app:degenerate_set_-e}
    \end{itemize}
    \item \textbf{\hyperref[app:1d_necessary]{\Cref*{app:1d_necessary}: Deferred proofs from \cref*{sec:regular_adv_bayes}---{\tocbreak}proof of \cref*{lemma:^e_to_order}}} \dotfill \pageref{app:1d_necessary}
    
    \item \textbf{\hyperref[app:degenerate_1d]{\Cref*{app:degenerate_1d}: Deferred Proofs from \cref*{sec:degenerate_1d}}} \dotfill \pageref{app:degenerate_1d}
    \begin{itemize}[label={}]
         \item \hyperref[app:degenerate_1d_eta_0_1]{\Cref*{app:degenerate_1d_eta_0_1}: Proof of \cref*{lemma:degenerate_1d_eta_0_1}} \dotfill \pageref{app:degenerate_1d_eta_0_1}

         \item \hyperref[app:identify_degenerate_eta_0_1]{\Cref*{app:identify_degenerate_eta_0_1}: Proof of \cref*{lemma:max_degenerate_set,cor:fourth_bullet}} \dotfill \pageref{app:identify_degenerate_eta_0_1}
    \end{itemize}

    \item \textbf{\hyperref[app:increasing_e]{\Cref*{app:increasing_e}: Deferred proofs from \cref*{sec:increasing_e}}} \dotfill \pageref{app:increasing_e}
    \item \textbf{\hyperref[app:examples]{\Cref*{app:examples}: Adversarial Bayes classifier examples{\tocbreak}and deferred \cref*{sec:examples} proofs}} \dotfill \pageref{app:examples}
        
    \begin{itemize}[label={}]
         \item \hyperref[app:gaussians_equal_variances]{\Cref*{app:gaussians_equal_variances}: \Cref*{ex:gaussians_equal_variances} details} \dotfill \pageref{app:gaussians_equal_variances}

         \item \hyperref[app:gaussians_equal_means_details]{\Cref*{app:gaussians_equal_means_details}: \Cref*{ex:gaussians_equal_means} details} \dotfill \pageref{app:gaussians_equal_means_details}

        \item \hyperref[app:eps_in_interval]{\Cref*{app:eps_in_interval}: Proof of \cref*{lemma:eps_in_interval}} \dotfill \pageref{app:eps_in_interval}

        \item \hyperref[app:non_uniqueness_all]{\Cref*{app:non_uniqueness_all}: \Cref*{ex:non_uniqueness_all} details} \dotfill \pageref{app:non_uniqueness_all}

        \item \hyperref[app:ex_degenerate_details]{\Cref*{app:ex_degenerate_details}: \Cref*{ex:degenerate} details} \dotfill \pageref{app:ex_degenerate_details}

        \item \hyperref[app:uniform_within_e]{\Cref*{app:uniform_within_e}: Proof of \cref*{prop:uniform_within_e}} \dotfill \pageref{app:uniform_within_e}

        \item \hyperref[app:eta_0_1]{\Cref*{app:eta_0_1}: Proof of \cref*{prop:eta_0_1}} \dotfill \pageref{app:eta_0_1}

        \item \hyperref[app:deg_eta_0_1_counterexample]{\Cref*{app:deg_eta_0_1_counterexample}: \Cref*{ex:deg_eta_0_1_counterexample} details} \dotfill \pageref{app:deg_eta_0_1_counterexample}
    \end{itemize}
    
\end{itemize}

\ifthenelse{\boolean{appendixmode}}{\section{Organization of the appendix}}{\section{Organization of the supplement}} \label{app:app_outline}
The appendix is organized to be read sequentially and generally follows the order of the main sections of the paper. The exception is \cref{app:examples}, which contains detailed computations for all examples of adversarial Bayes classifiers and is presented at the very end.

Other than the following exceptions, each lettered appendix is self-contained:
\begin{itemize}
    \item \Cref{app:identify_degenerate_eta_0_1} relies on \cref{lemma:boundary^e_leb_zero}
    \item \cref{app:increasing_e} relies on \cref{lemma:transition_eta_0_1}
    \item \cref{app:eps_in_interval} relies on \cref{lemma:transition_eta_0_1}
    \item \cref{app:deg_eta_0_1_counterexample} relies on \cref{lemma:transition_eta_0_1}
\end{itemize}

\section{Deferred proofs from \cref{sec:Background}}\label{app:background_main_deferred}
\subsection{Proof of \cref{prop:Bayes_equivalencies}}\label{app:Bayes_equivalencies}

\begin{proof}[Proof of \cref{prop:Bayes_equivalencies}]
    We show the implications \cref{it:Bayes_equivalencies_unique} $\Rightarrow$ \cref{it:Bayes_equivalencies_values} $\Rightarrow$ \cref{it:Bayes_equivalencies_eta} $\Rightarrow$ \cref{it:Bayes_equivalencies_unique}.

\textbf{\cref{it:Bayes_equivalencies_unique} $\Rightarrow$ \cref{it:Bayes_equivalencies_values}:} 

If the Bayes classifier is unique, then $\one_{B_1}=\one_{B_2}$ $\PP$-a.e. Consequently, any two Bayes classifiers satisfy $\PP_0(B_1)=\PP_0(B_2)$.

\textbf{ \cref{it:Bayes_equivalencies_values} $\Rightarrow$ \cref{it:Bayes_equivalencies_eta}:}

Let $R_*$ be the minimal value of the Bayes risk. For a Bayes classifier $B$ 
\[R_*-\PP_0(B)=\PP_1(B^C)\]
and consequently, the value of $\PP_1(B^C)$ is unique over all Bayes classifiers iff the value of $\PP_0(B^C)$ is unique over all Bayes classifiers. 
Consequently, \cref{it:Bayes_equivalencies_values} is equivalent to requiring that both $\PP_1(B^C)$ and $\PP_0(B)$ have unique values across all Bayes classifiers.

Both $\{\eta(\bx)>1/2\}$ and $\{\eta(\bx)\geq 1/2\}$ are Bayes classifiers, as the functions $\one_{\eta>1/2}$, $\one_{\eta\geq 1/2}$ minimize \cref{eq:C_def_a} pointwise. \Cref{it:Bayes_equivalencies_values} implies that $\PP_0(\{\eta(\bx)\geq 1/2\})=\PP_0( \{\eta(\bx)> 1/2\})$ and similarly, $\PP_1(\{\eta(\bx)\geq 1/2\}^C)=\PP_1( \{\eta(\bx)> 1/2\})^C$. These relations result in $\PP_0(\{\eta(\bx)=1/2\})=0$ and $\PP_1(\{\eta(\bx)=1/2\})=0$, or in other words $\PP(\{\eta(\bx)=1/2\})=0$.

\textbf{\cref{it:Bayes_equivalencies_eta} $\Rightarrow$ \cref{it:Bayes_equivalencies_unique}:}

As discussed in \cref{sec:background_bayes_classifiers}, any Bayes classifier must include $\{\eta(\bx)>1/2\}$ and exclude $\{\eta(\bx)<1/2\}$ up to sets of $\PP$-measure zero. The set of points within $\supp \PP$ not contained in either $\{\eta(\bx)>1/2\}$ or $ \{\eta(\bx)<1/2\}$ is 
\[\supp \PP\cap \{\eta(\bx)>1/2\}^C\cap \{\eta(\bx)<1/2\}^C= \supp \PP \cap \{\eta(\bx)=1/2\}.\]
Consequently, if $B_1$ and $B_2$ are two Bayes classifiers, then their symmetric difference $B_1\triangle B_2$ must be contained in $\supp \PP^C \cup \{\eta(\bx)=1/2\}$ up to measure zero sets. Thus if $\PP(\{\eta=1/2\})=0$, then $\PP(B_1\triangle B_2)=0$, and consequently $\one_{B_1}=\one_{B_2}$ $\PP$-a.e. 
\end{proof}

\subsection{Proof of \cref{lemma:union_intersection_adv_bayes}}\label{app:union_intersection_adv_bayes}

First, the $S_\e$ operation satisfies a subadditivity property:

\begin{lemma}\label{lemma:S_intersect_union}
    Let $S_1$ and $S_2$ be two subsets of $\Rset^d$. Then 
    \begin{equation}\label{eq:S_e_cap_cup}
        S_\e(\one_{S_1})+S_\e(\one_{S_2})\geq S_\e(\one_{S_1\cap S_2})+S_\e(\one_{S_1\cup S_2})
    \end{equation}
\end{lemma}
\begin{proof}
    First, notice that 
    \begin{equation}\label{eq:S_e_sum_compute}
    \begin{aligned}
        &S_\e(\one_{S_1})(\bx)+S_\e(\one_{S_2})(\bx)=\begin{cases}
        0&\text{if }\bx \not \in S_1^\e\text{ and }\bx \not \in S_2^\e\\
        1&\text{if } \bx \in S_1^\e\triangle S_2^\e\\
        2&\text{if }\bx\in S_1^\e\cap S_2^\e
    \end{cases}\\
    &=\one_{S_1^\e\cap S_2^\e}(\bx)+\one_{S_1^\e\cup S_2^\e}(\bx)
    \end{aligned}
    \end{equation}
    due to \cref{eq:indicator_supremum}.
    Next, one can always swap the order of two maximums but a min-max is always larger than a max-min. Therefore:
    \begin{equation}\label{eq:compare_S_e_cup_cap}
    \begin{aligned}
        &S_\e(\one_{S_1\cap S_2})+S_\e(\one_{S_1\cup S_2})=S_\e(\min(\one_{S_1},\one_{S_2}))+S_\e(\max(\one_{S_1},\one_{S_2}))\\
        &\leq \min(S_\e(\one_{S_1}),S_\e(\one_{S_2}))+\max(S_\e(\one_{S_1}),S_\e(\one_{S_2}))=\one_{S_1^\e\cap S_2^\e}+\one_{S_1^\e\cup S_2^\e}
    \end{aligned}
    \end{equation}
    Comparing \cref{eq:S_e_sum_compute} and \cref{eq:compare_S_e_cup_cap} results in \cref{eq:S_e_cap_cup}.
\end{proof}
\Cref{lemma:union_intersection_adv_bayes,lemma:S_intersect_union} are not a novel results, a similar statements have appeared in \cite{BungertGarciaMurray2021,PydiJog2021}.

Therefore, the adversarial classification risk is sub-additive.
\begin{corollary}\label{cor:adv_classification_subadditive}
    Let $S_1$ and $S_2$ be any two sets. Then 
    \[R^\e(S_1\cap S_2)+R^\e(S_1\cup S_2)\leq R^\e(S_1)+R^\e(S_2)\]
\end{corollary}

This result then directly implies \cref{lemma:union_intersection_adv_bayes}:
\begin{proof}[Proof of \cref{lemma:union_intersection_adv_bayes}]
    Let $A_1$ and $A_2$ be two adversarial Bayes classifiers, and let $R_*^\e$ be the minimal adversarial Bayes risk. Then \cref{cor:adv_classification_subadditive} implies that 
    \[2R^\e_*\geq R^\e(A_1\cup A_2)+R^\e(A_1\cap A_2)\]
    and hence $A_1\cap A_2$ and $A_1\cup A_2$ must be adversarial Bayes classifiers as well.
\end{proof}

\subsection{Complementary slackness-- proof of \cref{thm:complementary_slackness_classification}}\label{app:duality_proofs}\label{app:W_infty}

The complementary slackness relations of \cref{thm:complementary_slackness_classification} are a consequence of the minimax relation of \cref{thm:minimax_classification} and properties of the $W_\infty$ metric. 

    Integrating the maximum of an indicator function over an $\e$-ball is intimately linked to maximizing an integral over a $W_\infty$ ball of measures: 
        \begin{lemma}\label{lemma:S_e_and_W_inf}
        Let $\QQ$ be a positive measure. Then for any Borel set $A$ 
        \[\int S_\e(\one_A)d\QQ\geq \sup_{\QQ'\in \Wball \e(\QQ)}\int \one_Ad\QQ'\]
    \end{lemma}

Lemma~5.1 of \cite{PydiJog2021} and Lemma~3 of \cite{FrankNilesWeed23minimax} proved slightly different versions of this result, so we include a proof here for completeness.
\begin{proof}
    Let $\QQ'$ be any measure with $W_\infty(\QQ,\QQ')\leq \e$ and let $\gamma$ be any coupling between $\QQ$ and $\QQ'$ for which 
    \[\esssup_{(\bx,\by)\sim \gamma} \|\bx-\by\|=W_\infty(\QQ,\QQ').\]
    Such a coupling exists by Theorem~2.6 of \cite{Jylha15}. 
    Then $S_\e(\one_A)(\bx)\geq \one_A(\by)$ $\gamma$-a.e. and consequently
    \[\int S_\e(\one_A)(\bx)d\QQ(\bx)=\int S_\e(\one_A)(\bx) d\gamma(\bx,\by)\geq \int \one_A(\by) d\gamma(\bx,\by)=\int \one_A(\by) d\QQ'(\by)\]
    Now taking a supremum over all $\QQ'\in \Wball \e(\QQ)$ concludes the proof.
\end{proof}
    One can prove \cref{thm:complementary_slackness_classification} with this result.
\begin{proof}[Proof of \cref{thm:complementary_slackness_classification}]

\mbox{}\\
\textbf{Forward Direction: }

    Let $A$ be a minimizer of $\cprm$ and assume that  $\PP_0^*\in \Wball \e (\PP_0)$, $\PP_1^*\in \Wball \e(\PP_1)$ maximize $\cdl$. Then: 
    \begin{align}
        &\cprm(A)=\int S_\e(\one_{A^C})d\PP_1+\int S_\e(\one_A)d\PP_0\geq \int \one_{A^C}d\PP_1^*+\int \one_A d\PP_0^*\label{eq:comp_slack_proof_1}\\
        &=\int \eta^* \one_{A^C}d\PP_1^*+\int (1-\eta^*) \one_A d\PP_0^*\geq \int C^*(\eta^*)d\PP^*=\cdl(\PP_0^*,\PP_1^*)\label{eq:comp_slack_proof_2}
    \end{align}
    The first inequality follows from \cref{lemma:S_e_and_W_inf} while the second inequality follows from the definition of $C^*$ in \cref{eq:C_def}. By \cref{thm:minimax_classification}, the first expression of \cref{eq:comp_slack_proof_1} and the last expression of \cref{eq:comp_slack_proof_2} are equal. Thus all the inequalities above must in fact be equalities. Thus the fact that the inequality in \cref{eq:comp_slack_proof_2} is an equality implies \cref{eq:complementary_slackness_necessary_classification}. \Cref{lemma:S_e_and_W_inf} and the fact that the inequality in \cref{eq:comp_slack_proof_1} must be an equality implies \cref{eq:sup_comp_slack_classification}.

\textbf{Backward Direction: }

    Assume that $A$, $\PP_0^*$, and $\PP_1^*$ satisfy \cref{eq:sup_comp_slack_classification} and \cref{eq:complementary_slackness_necessary_classification}. We will argue that $A$ must be a minimizer of $\cprm$ and $(\PP_0^*,\PP_1^*)$ must maximize $\cdl$.

    First, notice that \cref{thm:minimax_classification} implies that $R^\e(A')\geq \cdl(\PP_0',\PP_1')$ for \emph{any} Borel $A'$ and \emph{any} $\PP_0'\in \Wball\e(\PP_0), \PP_1'\in \Wball \e (\PP_1)$.
    Thus if one can show
    \begin{equation}\label{eq:comp_slack_target}
        R^\e(A)=\cdl(\PP_0^*,\PP_1^*), 
    \end{equation}
    then $A$ must minimize $\cprm$ because for any other $A'$,
    \[\cprm(A')\geq \cdl(\PP_0^*,\PP_1^*)=\cprm(A).\]
    Similarly, one could conclude that $\PP_0^*,\PP_1^*$ maximize $\cdl$ because for any other $\PP_0'\in \Wball \e(\PP_0)$ and $\PP_1'\in \Wball \e(\PP_1)$,
    \[\cdl(\PP_0',\PP_1')\leq R^\e(A)=\cdl(\PP_0^*,\PP_1^*).\]
    
    Hence it remains to show \cref{eq:comp_slack_target}. Applying \cref{eq:sup_comp_slack_classification} followed by \cref{eq:complementary_slackness_necessary_classification}, one can conclude that 
    \begin{align*}
        R^\e(A)&=\int S_\e(\one_{A^C})d\PP_1+\int S_\e(\one_A)d\PP_0=\int \one_{A^C}d\PP_1^*+\int \one_Ad\PP_0^*&\text{\cref{eq:sup_comp_slack_classification}}\\
        &=\int \eta^*\one_{A^C}+(1-\eta^*)\one_A d\PP^*=\int C^*(\eta^*)d\PP^*=\cdl(\PP_0^*,\PP_1^*)&\text{\cref{eq:complementary_slackness_necessary_classification}}
    \end{align*}
    
\end{proof}
\section{Deferred arguments from \cref{sec:main_results}--- \Cref{ex:counterexample_Bayes} details}\label{app:counterexample_Bayes}
    It remains to verify two claims: 1) The formula \cref{eq:Bayes_soln} equals $\{\eta(x)>1/2\}$, and 2) The Bayes classifier in \cref{eq:Bayes_soln} is not equivalent to any regular set, in the the sense defined by \cref{def:equivalence_Bayes}.
    
    \subsubsection*{1) Verifying \cref{eq:Bayes_soln}}
    As discussed in \cref{sec:background_bayes_classifiers}, the set $\{\eta(x)>1/2\}$ is always a Bayes classifier. Similarly,
    the conditions \cref{eq:Bayes_necessary} are equivalent to $\eta(c_i)=\eta(d_i)=1/2$. The function $\eta$ in \cref{eq:counterexample_Bayes_def} satisfies $\eta(x)=1/2$ when either $x=0$, or
    \vspace{-10pt}\sidebysidesubequations{eq:Bayes_necessary_counterex}{\frac \pi x =2n\pi +\frac \pi 2}{\frac \pi x=(2n+1)\pi+\frac \pi 2}
    \vspace{5pt}
    These equations are solved by

    \vspace{-10pt}\sidebysidesubequations{eq:Bayes_necessary_counterex_soln}{x=\frac 2 {4n+1}}{x=\frac 2 {4n+3}}
    \vspace{5pt}

    respectively with $n\in \mathbb Z$. Observe that $\eta(x)>1/2$ iff $\frac \pi x\in ((2n+1)\pi+\frac \pi 2 , 2(n+1)\pi +\frac \pi 2)$ while $\eta(x)<1/2$ iff $\frac \pi x\in (2n\pi+\frac \pi 2, (2n+1)\pi +\frac \pi 2)$. Consequently, the solutions \cref{eq:Bayes_necessary_counterex_soln_a} are the left endpoints of the intervals comprising $\{\eta(x)>1/2\}$ while the solutions \cref{eq:Bayes_necessary_counterex_soln_b} are form the right endpoints. As $\lim_{n\to \infty} \frac 1 {4n+1}=0$ and $\lim_{n\to -\infty} \frac 1 {4n+3}=0$, the point $x=0$ is not an endpoint of any interval comprising $\{\eta(x)>1/2\}$.

    \subsubsection*{2) Proving \cref{eq:Bayes_soln} is not equivalent to any regular set}
    \begin{proposition}
        The Bayes classifier in \cref{eq:Bayes_soln} is not equivalent to any regular set.
    \end{proposition}
    In this context, \emph{equivalence} is defined according to \cref{def:equivalence_Bayes} while \emph{regularity} refers to \cref{def:regularity} with $\e=0$. Specifically, a set is $S$ is called \emph{regular} if every point in $S$ is contained in an interval of positive length that is entirely contained in $B$, while every point in $B^C$ is contained in an interval $I$ of positive length that is entirely contained in $B^C$. 

    \begin{proof}
        The complement of $B=\{\eta(x)>1/2\}$ in \cref{eq:Bayes_soln} within $\supp \PP$ is
        \[B^C\cap \supp \PP= \bigcup_{n=-\infty}^{-1} \left[\frac 2 {4n+1},\frac 2 {4n-1}\right]\cup \bigcup_{n=1}^\infty \left[\frac 2 {4n+1}, \frac 2 {4n-1}\right]\]
        Similarly, based on the computations in \eqref{eq:Bayes_necessary_counterex_soln}, the set $\{\eta<1/2\}\cap \supp \PP$ is given by 
        \begin{equation}\label{eq:Bayes_counterex_eta<1/2}
            \{\eta<1/2\}\cap \supp \PP=  \bigcup_{n=-\infty}^{-1} \left(\frac 2 {4n+1},\frac 2 {4n-1}\right)\cup \bigcup_{n=1}^\infty \left(\frac 2 {4n+1}, \frac 2 {4n-1}\right)
        \end{equation}        
        Crucially, $B$ is not regular because the point $0$ is an isolated point of $B^C$--- this point is not contained in any interval of positive length that is entirely contained in $B^C$.

        For contradiction, assume that $E$ is a regular Bayes classifier equivalent to $B$. 

        Assume first that $0\in E$. Then there exists an interval $I$ of positive length containing $0$ that is entirely contained in $E$. This interval must include either $(-\delta,0]$ or $[0,\delta)$ for some $\delta>0$, and therefore intersects $B^C$ on a set of positive measure. Recall from \cref{sec:background_bayes_classifiers} that any Bayes classifier must contain $\{\eta(x)>1/2\}$ (computed in \cref{eq:Bayes_soln}) and exclude $\{\eta(x)<1/2\}$ (given in \cref{eq:Bayes_counterex_eta<1/2}) up to measure zero sets. However, both intervals $[0,\delta)$ and $(-\delta,0]$ intersect $\{\eta(x)<1/2\}$ on sets of positive measure, and thus $E$ cannot be a Bayes classifier. Therefore, $0\not \in E$.
        
        An analogous argument shows that $0\not \in E^C$, leading to a contradiction.
    \end{proof}

\section{Deferred proofs from \cref{app:abs_cont_equivalencies_main}}\label{app:abs_cont_equivalencies}

\subsection{Proof of \cref{cor:0_1_a.e._equivalence,lemma:intersection_union_a.e.}}\label{app:intersection_union_a.e.}

\begin{proof}[Proof of \cref{lemma:intersection_union_a.e.}]
    We will assume that $S_\e(\one_{A_1})=S_\e(\one_{A_2})$ $\PP_0$-a.e., the argument for $S_\e(\one_{A_1^C})=S_\e(\one_{A_2^C})$ $\PP_1$-a.e. is analogous.
      If $S_\e(\one_{A_1})=S_\e(\one_{A_2})$ $\PP_0$-a.e., then 
    \[S_\e(\one_{A_1})=S_\e(\one_{A_2})=\max(S_\e(\one_{A_1}),S_\e(\one_{A_2}))=S_\e(\one_{A_1\cup A_2})\quad \PP_0\text{-a.e.}\]
    However, $S_\e(\one_{A_1^C})\geq S_\e(\one_{(A_1\cup A_2)^C})$. If this inequality were strict on a set of positive $\PP_1$-measure, we would have $R^\e(A_1\cup A_2)<R^\e(A_1)$ which would contradict the fact that $A_1$ is an adversarial Bayes classifier. Thus $S_\e(\one_{A_1^C})= S_\e(\one_{(A_1\cup A_2)^C})$ $\PP_1$-a.e. The same argument applied to $A_2$ then shows that $S_\e(\one_{A_1^C})= S_\e(\one_{(A_1\cup A_2)^C})=S_\e(\one_{A_2^C})$ $\PP_1$-a.e.

    Now as $S_\e(\one_{A_1^C})=S_\e(\one_{A_2^C})$ $\PP_1$-a.e., one can conclude that 
    \[S_\e(\one_{A_1^C})=S_\e(\one_{A_2^C})=\max(S_\e(\one_{A_1^C}),S_\e(\one_{A_2^C}))=S_\e(\one_{(A_1\cap A_2)^C})\quad \PP_1\text{-a.e.}\] 
    An analogous argument implies \cref{eq:0_cup_cap}.
\end{proof}

\Cref{cor:0_1_a.e._equivalence} is a direct consequence of this result:
\begin{proof}[Proof of \cref{cor:0_1_a.e._equivalence}]
       \Cref{lemma:intersection_union_a.e.} implies that $S_\e(\one_{A_1})=S_\e(\one_{A_2})$ $\PP_0$-a.e. iff $S_\e(\one_{A_1^C})=S_\e(\one_{A_2^C})$ $\PP_1$-a.e.    
\end{proof}

\subsection{Proof of \cref{lemma:closure_interior_adversarial_Bayes,lemma:boundary_difference}}\label{app:closure_interior_adversarial_Bayes}

The $\empty^\e$ operation on sets interacts particularly nicely with Lebesgue measure.

    \begin{lemma}\label{lemma:boundary^e_leb_zero}
        For any set $A$ and $\e>0$, $\partial (A^\e)$ has Lebesgue measure zero. 
    \end{lemma}

    This result is standard in geometric measure theory, see for instance Lemma~4 in \cite{AwasthiFrankMohri2021} for a proof.
    Next, the closure and $\empty^\e$ operations commute:
    \begin{lemma}\label{lemma:closure_e_commute}
        Let $A$ be any set in $\Rset^d$. Then $\ov{A^\e}=\left(\ov A\right)^\e$.
    \end{lemma}
    \begin{proof}We show the two inclusions $\ov{A^\e}\subset \left(\ov A\right)^\e$ and $\ov{A^\e}\supset \left(\ov A\right)^\e$ separately.
    
    \textbf{Showing $\ov{A^\e}\subset \left(\ov A\right)^\e$:}
        First, because the direct sum of a closed set and a compact set must be closed, $\left(\ov A\right)^\e$ is a closed set that contains $A^\e$. Therefore, because $\ov{A^\e}$ is the smallest closed set containing $A^\e$, the set $\ov{A^\e}$ must be contained in $\left(\ov A\right)^\e$.

    \textbf{Showing $\ov{A^\e}\supset \left(\ov A\right)^\e$:}
        Let $\bx \in \left(\ov A\right)^\e$, we will show that $\bx \in \ov{A^\e}$. If $\bx \in \left(\ov A\right)^\e$, then $\bx=\ba+\bh$ for some $\ba \in \ov A$ and $\bh \in \ov{B_\e(\zero)}$. Let $\ba_i$ be a sequence of points contained in $A$ that converges to $\ba$. Then $\ba_i+\bh \in A^\e$, and $\ba_i+\bh$ converges to $\ba+\bh$. Therefore, $\ba+\bh\in \ov{A^\e}$.    
    \end{proof}

Next, this result implies \cref{lemma:boundary_difference}:

\begin{proof}[Proof of \cref{lemma:boundary_difference}]

First, \cref{lemma:boundary^e_leb_zero,lemma:closure_e_commute} imply that 
\begin{equation}\label{eq:closure_interior_adv_Bayes_main_swap_closure}
    \PP_0(A^\e)= \PP_0(\ov{A^\e})= \PP_0(\ov A^\e) 
\end{equation} 
Furthermore, $\PP_1\left((A^C)^\e\right) \geq  \PP_1\left((\ov A^C)^\e\right)$ and thus $R^\e(A)\geq R^\e(\ov A)$. Consequently, $\ov A$ must be an adversarial Bayes classifier  and 
\begin{equation}\label{eq:closure_interior_other_equals_closure}
    \PP_1\left((A^C)^\e\right) =  \PP_1\left((\ov A^C)^\e\right)
\end{equation}
Furthermore, \cref{lemma:boundary^e_leb_zero,lemma:closure_e_commute} again imply that
\begin{equation}\label{eq:closure_interior_other_equals_interior}
    \PP_1\left((A^C)^\e\right)=\PP_1\left(\ov{(A^C)^\e}\right) =\PP_1\left (\left(\ov{(A^C)}\right)^\e\right)=  \PP_1\left(\left((\interior A)^C\right)^\e\right).
\end{equation}
Next a similar line of reasoning shows that $\interior A$ is also an adversarial Bayes classifier and thus
\begin{equation}\label{eq:closure_interior_adv_Bayes_main_swap_interior}
    \PP_0(A^\e)= \PP_0((\interior A)^\e) .
\end{equation} 
\Crefrange{eq:closure_interior_adv_Bayes_main_swap_closure}{eq:closure_interior_adv_Bayes_main_swap_interior} imply the desired result.
\end{proof}
\cref{lemma:boundary_difference} requires both the assumption of absolute continuity and $\e>0$. \Cref{lemma:boundary^e_leb_zero} is invoked to prove \cref{eq:closure_interior_adv_Bayes_main_swap_closure,eq:closure_interior_other_equals_interior}.

Finally, \cref{lemma:boundary_difference} implies that $\interior A$, $A$ and $\ov A$ are all equivalent up to degeneracy.
\begin{proof}[Proof of \cref{lemma:closure_interior_adversarial_Bayes}]
\Cref{lemma:boundary_difference} implies that if $E$ is any measurable set with $\interior A \subset E\subset \ov A$, then $\PP_0(E^\e)=\PP_0(A^\e)$ and $\PP_1((E^C)^\e)=\PP_1((A^C)^\e)$. Therefore, $E$ must be an adversarial Bayes classifier.

\end{proof}


    \subsection{Proof of \cref{lemma:open^e}}\label{app:open^e}\label{app:open_carrot_e}

    Before proving \cref{lemma:open^e}, we reproduce another useful intermediate result from \cite{AwasthiFrankMohri2021}.
    \begin{lemma}\label{lemma:e_ball_union}
        Let $\ba_n$ be a sequence that approaches $\ba$. Then $B_\e(\ba)\subset \bigcup_{n=1}^\infty B_\e(\ba_n)$. 
    \end{lemma}
    \begin{proof} 
        Let $\by$ be any point in $B_\e(\ba)$ and let $\delta =\|\by-\ba\|$. Pick $n$ large enough so that $\|\ba-\ba_n\|< \e-\delta$. Then
        \[\|\by-\ba_n\|\leq \|\ba -\ba_n\|+\|\by-\ba\|<\e-\delta +\delta =\e\]
        and thus $\by\in B_\e(\ba_n)$.
    \end{proof}

        \begin{proof}[Proof of \cref{lemma:open^e}]
        We will argue that $U^\e=(U\cap \QQ^d)^\e$, the argument for $U\cap (\QQ^d)^C$ is analogous.

        First, $U\cap \QQ^d\subset U$ implies that $(U\cap \QQ^d)^\e \subset U^\e$. 

        For the opposite containment, let $\bu$ be any point in $U$. We will argue that $\ov{B_\e(\bu)}\subset (U\cap \QQ)^\e$. Because $U$ is open, there is a ball $B_r(\bu)$ contained in $U$. Because $\QQ^d$ is dense in $\Rset^d$, for every $\by \in B_r(\bu)$, there is  a sequence $\by_n\in \QQ$ converging to $\by$. Thus \cref{lemma:e_ball_union} implies that 
        \[  \ov{B_\e(\bu)} \subset B_r(\bu)^\e\subset (B_r(\bu)\cap \QQ^d)^\e\subset (U\cap \QQ^d)^\e \]
        Taking a union over all $\bu\in U$ results in $U^\e\subset (U\cap \QQ^d)^\e$.

    \end{proof}

\section{Deferred proofs from \cref{app:TFAE_equiv}---proof of \cref{lemma:hat_eta_adv_Bayes}}\label{app:hat_eta_adv_Bayes}\label{app:TFAE_equiv_proof}\label{app:TFAE_equiv_lemmas}
    In order to conveniently introduce the construction of $\hat \eta$, we define an operation $I_\e$ analogous to $S_\e$: for a measurable function $g$,
    \[I_\e(g)(\bx)=\inf_{\bh \in \ov{B_\e(\bx)}} g(\bx+\bh).\]

    Lemma~24 of \cite{FrankNilesWeed23minimax} shows that there exists a function $\hat \eta$ and maximizers $\PP_0^*,\PP_1^*$ of $\cdl$ for which optimal attacks on $\hat \eta$ are are given by $\PP_0^*$, $\PP_1^*$:

   \begin{proposition}\label{prop:hat_eta_detailed}

     There exists a function $\hat \eta:\Rset^d\to [0,1]$ and measures $\PP_0^*\in \Wball\e (\PP_0)$, $\PP_1^*\in \Wball \e (\PP_1)$ with the following properties:
         \begin{enumerate}
             \item \label{it:hat_eta_1} Let $\PP^*=\PP_0^*+\PP_1^*$ and $\eta^*=d\PP_1^*/d\PP^*$. Then 
             \[\hat \eta(\by)=\eta^*(\by)\quad \PP^*-\text{a.e.}\]
             \item \label{it:hat_eta_2} Let $\gamma_i^*$ be a coupling between $\PP_i$ and $\PP_i^*$ for which $\esssup_{(\bx,\by)\sim \gamma_i^*}\|\bx-\by\|\leq \e$. Then for these $\PP_0^*,\PP_1^*$, the function $\hat \eta$ satisfies 
             \[I_\e(\hat \eta)(\bx)=\hat \eta(\by) \quad \gamma_1^*\text{-a.e.} \quad \text{and} \quad S_\e(\hat \eta)(\bx)=\hat \eta(\by)\quad \gamma_0^*\text{-a.e.}\]
         \end{enumerate}

    \end{proposition}
Recall that Theorem~2.6 of \cite{Jylha15} proves that when $W_\infty(\QQ,\QQ')\leq \e$, there always exists a coupling $\gamma$ between $\QQ$ and $\QQ'$ with $\esssup_{(\bx,\by)\sim \gamma}\|\bx-\by\| \leq \e$.

Next, we prove that one can take $\hat A_1=\{\hat \eta>1/2\}$ and $\hat A_2=\{\hat \eta \geq 1/2\}$ in \cref{lemma:hat_eta_adv_Bayes}.

\begin{proof}[Proof of \cref{lemma:hat_eta_adv_Bayes}]
    Let $\PP_0^*$, $\PP_1^*$, $\gamma_0^*,$ and $\gamma_1^*$ be the measures given by \cref{prop:hat_eta_detailed} and set $\PP^*=\PP_0^*+\PP_1^*$ and $\eta^*=d\PP_1^*/d\PP^*$. Let $\hat \eta$ denote the function described in \cref{prop:hat_eta_detailed}. We will show that the classifier $\hat A_1=\{\hat \eta>1/2\}$ satisfies the required properties by verifying the complementary slackness conditions in \cref{thm:complementary_slackness_classification}, the argument for $\hat A_2$ is analogous. 
    
    \Cref{it:hat_eta_1} of \cref{prop:hat_eta_detailed} implies that $\one_{\{\hat \eta >1/2\}}=\one_{\eta^*>1/2}$ $\PP^*$-a.e. and $\one_{\{\hat \eta >1/2\}^C}=\one_{\{\hat \eta \leq 1/2\}} =\one_{\eta^*\leq 1/2}$ $\PP^*$-a.e.

   Therefore,
    \[\eta^*\one_{\{\hat \eta >1/2\}^C}+(1-\eta^*)\one_{\{\hat \eta >1/2\}} =C^*(\eta^*)\quad \PP^*\text{-a.e.,}\]

    verifying \cref{eq:complementary_slackness_necessary_classification} of \cref{thm:complementary_slackness_classification}.

    Next, \cref{it:hat_eta_2} of \cref{prop:hat_eta_detailed} implies that $\hat \eta$ assumes its maximum over closed $\e$-balls $\PP_0$-a.e. and hence $S_\e(\one_{\hat \eta >1/2})(\bx)=\one_{S_\e(\hat \eta(\bx))>1/2}$ $\PP_0$-a.e. Additionally, \cref{it:hat_eta_2} of \cref{prop:hat_eta_detailed} implies that $\one_{S_\e(\hat \eta)(\bx)>1/2}=\one_{\hat \eta(\by)>1/2}$ $\gamma_0^*$-a.e.  Therefore, one can conclude that 
    \begin{equation}\label{eq:c_slack_verify_0}
      \int S_\e(\one_{\hat \eta>1/2})(\bx)d\PP_0(\bx)=\int \one_{\hat \eta(\by)>1/2} d\gamma_0^*(\bx,\by)=\int \one_{\hat \eta>1/2}d\PP_0^*  
    \end{equation}
    
    Similarly, using the fact that $I_\e(\hat \eta)(\bx)=\hat \eta(\by)$ $\gamma_1^*$-a.e., one can show that $\int S_\e(\one_{\hat \eta \leq 1/2})d\PP_1=\int \one_{\hat \eta\geq 1/2} d\PP_1^*$. This statement together with \cref{eq:c_slack_verify_0} verifies \cref{eq:sup_comp_slack_classification}. 
\end{proof}

\section{More about the $\empty^\e$, $\empty^{-\e}$, and $S_\e$ operations}\label{app:e_-e} 
This appendix provides a unified exposition of several results relating to the $\empty^\e$ and $\empty^{-\e}$ relations---namely \Cref{eq:e_interpretation,eq:-e_interpretation}, \cref{lemma:e_-e_decrease_R^e,lemma:A_e_-e_inclusion,lemma:A_e_-e_containments}. These results have all appeared elsewhere in the literature ---see for instance \cite{AwasthiFrankMohri2021,BungertGarciaMurray2021}.

The characterizations of the $\empty^\e$ and $\empty^{-\e}$ operations provided by \cref{eq:e_interpretation} and \cref{eq:-e_interpretation} are an essential tool for understanding how $\empty^\e$ and $\empty^{-\e}$ interact.
\begin{proof}[Proof of \Cref{eq:e_interpretation}]
    To show \cref{eq:e_interpretation}, notice that $\bx\in A^\e$ iff $\bx\in \ov{B_\e(\ba)}$ for some element $\ba$ of $A$. Thus:
    \begin{equation*}
        \bx\in A^\e \Leftrightarrow \bx \in \ov{B_\e(\ba)}\text{ for some }\ba\in A\Leftrightarrow \ba \in \ov{B_\e(\bx)}\text{ for some }\ba\in A \Leftrightarrow \ov{B_\e(\bx)}\text{ intersects }A 
    \end{equation*}\mbox{}
    \end{proof}
\Cref{eq:-e_interpretation} then follows directly from \Cref{eq:e_interpretation}:
\begin{proof}[Proof of \Cref{eq:-e_interpretation}]
By \Cref{eq:e_interpretation}, 
\[ \bx\in (A^C)^\e \Leftrightarrow \ov{B_\e(\bx)} \text{ intersects }A^C\]
Now $A^{-\e}= ((A^C)^\e)^C$, and so taking compliments of the relation above implies
\[\bx \in A^{-\e}\Leftrightarrow \ov{B_\e(\bx)}\text{ does not intersect } A^C\Leftrightarrow \ov{B_\e(\bx)}\subset A\]
    
\end{proof}

Next, \Cref{eq:e_interpretation} and \Cref{eq:-e_interpretation} immediately imply \cref{lemma:A_e_-e_inclusion}:
\begin{proof}[Proof of \cref{lemma:A_e_-e_inclusion}]
    By \Cref{eq:e_interpretation}, \Cref{eq:-e_interpretation}, $(A^\e)^{-\e}$ is the set of points $\bx$ for which $\ov{B_\e(\bx)}\subset A^\e$. For any point $\ba \in A$, $\ov{B_\e(\ba)}\subset A^\e$ and thus $A\subset (A^\e)^{-\e}$. Applying this statement to the set $A^C$ and then taking compliments results in $(A^{-\e})^\e\subset A$.  
\end{proof}

\Cref{lemma:A_e_-e_inclusion} then immediately implies \cref{lemma:A_e_-e_containments}:

\begin{proof}[Proof of \cref{lemma:A_e_-e_containments}]
First, \cref{lemma:A_e_-e_inclusion} implies that $A\subset (A^\e)^{-\e}$ and thus $A^\e\subset ((A^\e)^{-\e})^\e$. At the same time, \cref{lemma:A_e_-e_inclusion} implies that $((A^\e)^{-\e})^\e=\Big( \Big(A^{\e}\Big)^{-\e} \Big)^\e\subset A^\e$. Therefore, $((A^\e)^{-\e})^\e=A^\e$.
Applying this result to $A^C$ and then taking compliments then results in $((A^{-\e})^\e)^{-\e}=A^{-\e}$.

Next, \cref{lemma:A_e_-e_inclusion} implies that $(A^{-\e})^\e\subset A$ and hence $( (A^{-\e})^\e)^\e\subset A^\e$. Applying this result to $A^C$ and then taking compliments $((A^\e)^{-\e})^{-\e}\supset A^{-\e}$.

\end{proof}

\Cref{lemma:e_-e_decrease_R^e} is then an immediate consequence of \cref{lemma:A_e_-e_containments}.
\section{Measurability}\label{app:measurability}
\subsection{Defining the universal $\sigma$-algebra}\label{app:define_universal_sigma_algebra}

	Let $\cM_+(\Rset^d)$ be the set of finite positive measures on the Borel $\sigma$-algebra $\cB(\Rset^d)$. For a Borel measure $\nu$ in $\cM_+(\Rset^d)$, let $\cL_\nu(\Rset^d)$ be the completion of $\cB(\Rset^d)$ under $\nu$. Then the \emph{universal $\sigma$-algebra} $\sU(\Rset^d)$ is defined as 
	
		\[\sU(\Rset^d)=\bigcap_{\nu\in \cB(\Rset^d)}\cL_\nu(\Rset^d)\]
	
	In other words, $\sU(\Rset^d)$ is the $\sigma$-algebra of sets which are measurable with respect to the completion of \emph{every} finite positive Borel measure $\nu$. See \cite[Chapter 7]{BertsekasShreve96} or \cite{nishiura2010} for more about this construction. 
	
	Due to \cref{thm:univ_meas}, throughout this paper, we adopt the convention that $\int S_\e(\one_A)d\nu$ is the integral of $S_\e(\one_A)$ with respect to the completion of $\nu$.

\subsection{Proof of \cref{thm:borel_univ_equivalent}}\label{app:borel_univ_equivalent}

    First, notice that because every Borel set is universally measurable, $\inf_{A\in \cB(\Rset^d)} R^\e(A)\geq \inf_{A\in \sU(\Rset^d)}R^\e(A)$. The opposite inequality relies on a duality statement similar to \cref{thm:minimax_classification}, but with the primal minimized over universally measurable sets and the dual maximized over measures on $\sU(\Rset^d)$.

      For a Borel measure $\QQ$, there is a canonical extension to the universal $\sigma$-algebra called the \emph{universal completion}.
    \begin{definition}
        The \emph{universal completion} $\wtd \QQ$ of a Borel $\QQ$ is the completion of $\QQ$ restricted to the universal $\sigma$-algebra.
    \end{definition}
     Notice that $\QQ(E)=\wtd \QQ(E)$ for any Borel measure $\QQ$ and Borel set $E$. Furthermore, \cite[Lemma~7.26]{BertsekasShreve96} states that every universally measurable set differs from Borel set by a subset of a null Borel set.

\begin{lemma}\label{lemma:univeral_to_borel_B.S.}
    The set $E$ is universally measurable iff given any Borel measure $\QQ$, there are Borel sets $B_1$, $B_2$ for which $B_1\subset E\subset B_2$ and $\QQ(B_1)=\QQ(B_2)$.
\end{lemma}
     
     As a consequence, 
     \begin{equation}\label{eq:integral_univ_borel_reduce}
         \int g d\QQ=\int gd\wtd \QQ\quad \text{for any Borel function $g$.}
     \end{equation}
    \Cref{lemma:univeral_to_borel_B.S.} provides a useful result for translating statements for $\cB(\Rset^d)$ to $\sU(\Rset^d)$.

     In addition to the $W_\infty$-ball of Borel measures $\Wball \e(\QQ)$ around $\QQ$, one can consider the $W_\infty$ ball of universal completions of measures around $\QQ$, which we will call $\UWball \e(\QQ)$.
     
          Explicitly, for a Borel measure $\QQ$, define
     
	\[\UWball \e(\QQ)=\{\wtd \QQ':\QQ'\in \Wball \e(\QQ)\}.\]
     
     The following result shows that if $\QQ'\in \Wball \e (\QQ)$, then $W_\infty(\wtd \QQ,\wtd \QQ')\leq \e$, and thus $\UWball \e(\QQ)$ contains only measures that are within $\e$ of $\wtd \QQ$ in the $W_\infty$ metric.

     \begin{lemma}\label{lemma:borel_to_univ_W_infty}
         Let $\QQ$ and $\QQ'$ be Borel measures with $W_\infty(\QQ,\QQ')\leq \e$ and let $\wtd \QQ,\wtd \QQ'$ be their universal completions. Then $W_\infty(\wtd \QQ,\wtd\QQ')\leq \e$.
     \end{lemma}

     Next, to compare the values of $\cdl$ on $\Wball \e(\PP_0)\times \Wball \e(\PP_1)$ and $\UWball \e( \PP_0) \times \UWball \e( \PP_1)$, we show:
     \begin{corollary}\label{cor:cdl_univ_borel}
         Let $\PP_0,\PP_1$ be two Borel measures and let $\wtd \PP_0,\wtd \PP_1$ be their universal completions. Then $\cdl(\PP_0,\PP_1)=\cdl(\wtd \PP_0,\wtd\PP_1)$. 
     \end{corollary}
    Thus \cref{lemma:borel_to_univ_W_infty} and \cref{cor:cdl_univ_borel} implies that 
    \begin{equation}\label{eq:cdl_borel_univ_compare}
        \sup_{\substack{\PP_0'\in \Wball \e(\PP_0)\\ \PP_1'\in \Wball \e(\PP_1)}} \cdl(\PP_0',\PP_1')= \sup_{\substack{\wtd \PP_0'\in \UWball \e( \PP_0)\\ \wtd \PP_1'\in \UWball \e(\PP_1)}} \cdl(\wtd\PP_0',\wtd\PP_1')    
    \end{equation}

	Furthermore, \cref{lemma:S_e_and_W_inf} and \cref{eq:integral_univ_borel_reduce} imply: 

	\begin{corollary} \label{cor:S_e_and_W_inf_universal}
		Let $\QQ$ be a finite positive measure on $\sU(\Rset^d)$. Then for any universally measurable set $A$,
        \[\int S_\e(\one_A)d\QQ\geq \sup_{\QQ'\in \UWball \e(\QQ)}\int \one_Ad\QQ'\]
	\end{corollary}
       See \cref{app:univ_borel_deferred_lemmas} for proofs of \cref{lemma:borel_to_univ_W_infty,cor:cdl_univ_borel,cor:S_e_and_W_inf_universal}.

    This result implies a weak duality relation between the primal $\cprm$ minimized over $\sU(\Rset^d)$ and the dual $\cdl$ maximized over $\UWball \e ( \PP_0)\times \UWball\e( \PP_1)$:
    
    \begin{lemma}[Weak Duality]\label{lemma:weak_duality_universal}
        Let $\PP_0,\PP_1$ be two Borel measures. Then 
        \[\inf_{A\in \sU(\Rset^d)} R^\e(A)\geq  \sup_{\substack{\wtd \PP_0'\in \UWball \e(\PP_0)\\ \wtd \PP_1'\in \UWball \e(\PP_1)}} \cdl(\wtd\PP_0',\wtd\PP_1')\]
    \end{lemma}
    \begin{proof}
        Let $A$ be any universally measurable set and let $\wtd \PP_0'$, $\wtd \PP_1'$ be any measures in $\UWball \e (\PP_0)$ and $\UWball \e(\PP_1)$ respectively.

        Then \cref{cor:S_e_and_W_inf_universal} implies that 
        \begin{align*}
            \inf_{A\in \sU(\Rset^d)} R^\e(A)\geq  \inf_{A\in \sU(\Rset^d)} \sup_{\substack{\wtd \PP_0'\in \UWball \e( \PP_0)\\ \wtd \PP_1'\in \UWball \e(\PP_1)}}  \int \one_{A^C} d\wtd \PP_1'+\int \one_A d\wtd \PP_0'
        \end{align*}
        However, because inf-sup is always larger than a sup-inf, one can conclude that 
        \begin{align*}
            \inf_{A\in \sU(\Rset^d)} R^\e(A)\geq  \sup_{\substack{\wtd \PP_0'\in \UWball \e( \PP_0)\\ \wtd \PP_1'\in \UWball \e(\PP_1)}} \inf_{A\in \sU(\Rset^d)}  \int \one_{A^C} d\wtd \PP_1'+\int \one_A d\wtd \PP_0'= \sup_{\substack{\wtd \PP_0'\in \UWball \e( \PP_0)\\ \wtd \PP_1'\in \UWball \e(\PP_1)}}  \cdl(\wtd \PP_0',\wtd \PP_1')
        \end{align*}
    \end{proof}

This observation suffices to prove \cref{thm:borel_univ_equivalent}:
\begin{proof}[Proof of \cref{thm:borel_univ_equivalent}]
    First, because every Borel set is universally measurable, $\inf_{A\in \cB(\Rset^d)}\cprm(A)\geq \inf_{A\in \sU(\Rset^d)} \cprm(A)$. Thus the strong duality result of \cref{thm:minimax_classification} and \cref{eq:cdl_borel_univ_compare} imply that 
    \[\inf_{A\in \sU(\Rset^d)} \cprm(A)\leq \inf_{A\in \cB(\Rset^d)}\cprm(A)= \sup_{\substack{\PP_0'\in \Wball \e(\PP_0)\\ \PP_1'\in \Wball \e(\PP_1)}} \cdl(\PP_0',\PP_1')= \sup_{\substack{\wtd \PP_0'\in \UWball \e(\PP_0)\\ \wtd \PP_1'\in \UWball \e(\PP_1)}} \cdl(\wtd\PP_0',\wtd\PP_1') .\] However, the weak duality statement of \cref{lemma:weak_duality_universal} implies that the inequality above must actually be an equality.
\end{proof}

\subsection{Proofs of \cref{lemma:borel_to_univ_W_infty,cor:cdl_univ_borel,cor:S_e_and_W_inf_universal}}\label{app:univ_borel_deferred_lemmas}

\subsubsection{Proof of \cref{lemma:borel_to_univ_W_infty}}
    Notice that if $\gamma$ is a coupling between two Borel measures, $\esssup_{(\bx,\by)\sim \gamma} \|\bx-\by\|\leq \e$ iff $\gamma(\Delta_\e^C)=0$, where $\Delta_\e$ is the set defined by
         \begin{equation}\label{eq:Delta_e_definition}
             \Delta_\e=\{(\bx,\by)\in \Rset^d\times \Rset^d:\|\bx-\by\|\leq \e\}.
         \end{equation}
    This notation is helpful in the proof of \cref{lemma:borel_to_univ_W_infty}.

 \begin{proof}[Proof of \cref{lemma:borel_to_univ_W_infty}]
         Let $\gamma$ be a Borel coupling between $\QQ$ and $\QQ'$ supported on $\Delta_\e$. Such a coupling exists by Theorem~2.6 of \cite{Jylha15}. Let $\ov \gamma$ be the completion of $\gamma$ restricted to $\sigma(\sU(\Rset^d)\times \sU(\Rset^d))$, the $\sigma$-algebra generated by $\sU(\Rset^d)\times \sU(\Rset^d)$. We will show $\ov \gamma$ is the desired coupling between $\wtd \QQ$ and $\wtd \QQ'$.

         Let $S$ be an arbitrary universally measurable set in $\Rset^d$. Then \cref{lemma:univeral_to_borel_B.S.} states that there are Borel sets $E_1,E_2$ for which $E_1\subset S\subset E_2$ and $\wtd \QQ(E_1)=\wtd \QQ(S)=\wtd \QQ(E_2)$. Then because $\gamma$ and $\ov \gamma$ are equal on Borel sets,
         \[\wtd \QQ(E_1)=\QQ(E_1)=\gamma(E_1\times \Rset^d)=\ov \gamma(E_1\times \Rset^d)\]
         and similarly, 
          \[\wtd \QQ(E_2)=\QQ(E_2)=\gamma(E_2\times \Rset^d)=\ov \gamma(E_2\times \Rset^d)\]
         Therefore, \[\wtd \QQ(S)=\ov \gamma(E_1\times \Rset^d)=\ov \gamma(E_2\times \Rset^d)=\ov \gamma(S\times \Rset^d).\]
         Similarly, one can argue
         \[\wtd \QQ'(S)= \ov \gamma(\Rset^d\times S).\]

         Therefore, $\ov \gamma$ is a coupling between $\wtd \QQ$ and $\wtd \QQ'$.
         Next, recall that $\esssup_{(\bx,\by)\sim \gamma} \|\bx-\by\|\leq \e$ iff $\gamma(\Delta_\e^C)=0$, where $\Delta_\e$  defined by \cref{eq:Delta_e_definition}.

         Therefore, because $\Delta_\e$ is closed (and thus Borel),
         \[\ov\gamma (\Delta_\e^C)=\gamma(\Delta_\e^C)=0\]
         Consequently, $\esssup_{(\bx,\by)\sim \ov \gamma} \|\bx-\by\|\leq \e$ and thus $W_\infty(\wtd \QQ,\wtd \QQ')\leq \e$.
     \end{proof}
     \subsubsection{Proof of \cref{cor:cdl_univ_borel}}
     Next, we will show:
     \begin{lemma}
         Let $\nu,\lambda$ be two Borel measures with $\nu\ll \lambda$, and let $d\nu/d\lambda$ be the Radon-Nikodym derivative. Then $d\wtd \nu /d\wtd \lambda=d\nu/d\lambda$ $\wtd\lambda$-a.e.
     \end{lemma}
    \begin{proof}
        First, if a function $g$ is Borel measurable, ($g^{-1}:(\Rset,\cB(\Rset))\to (\Rset,\cB(\Rset^d))$, then it is necessarily universally measurable ($g^{-1}:(\Rset, \cB(\Rset))\to (\Rset^d,\sU(\Rset^d))$). Thus the Radon-Nikodym derivative $d\nu/d\lambda$ is both Borel measurable and universally measurable.

        Next, if $S\in \sU(\Rset^d)$ then \cref{lemma:univeral_to_borel_B.S.} implies there is a Borel set $E$ and $\lambda$-null sets $N_1,N_2$ for which $S=(E\cup N_1)-N_2$. Because $\nu$ is absolutely continuous with respect to $\lambda$, the sets $N_1$ and $N_2$ are null under $\nu$ as well. Therefore, by the defintion of the Radon-Nikodym derivative $d\nu/d\lambda$ and the fact that $\int gd\lambda=\int gd\wtd \lambda$ for all Borel functions $g$,
        \[\wtd \nu(S)=\nu(E)=\int_E \frac{d\nu}{d\lambda} d\lambda=\int_E \frac{d\nu}{d\lambda} d\wtd \lambda=\int_S \frac{d\nu}{d\lambda} d\wtd\lambda\]
        Because the Radon-Nikodym derivative is unique $\wtd \lambda$-a.e., it follows that 
        $d\wtd \nu/d\wtd \lambda =d \nu/d\lambda $ $\wtd \lambda$-a.e. 
    \end{proof}
         This result together with \cref{eq:integral_univ_borel_reduce} immediately implies \cref{cor:cdl_univ_borel}.

\subsubsection{Proof of \cref{cor:S_e_and_W_inf_universal}}\label{app:S_e_and_W_inf_universal}

\begin{proof}[Proof of \cref{cor:S_e_and_W_inf_universal}]
    Fix a $\QQ'\in \UWball \e(\QQ)$ and let $\QQ' =\tilde \lambda$ for some $\lambda\in \Wball \e(\QQ)$. Then \cref{lemma:univeral_to_borel_B.S.} states that there is a Borel set $B_1\subset A$ for which \[\lambda(B_1)=\tilde \lambda(B_1)=\QQ'(B_1)=\QQ'(A).\]
    Thus \cref{lemma:S_e_and_W_inf} implies that $\int\one_{B_1}d\QQ'\leq \int S_\e(\one_{B_1})d\QQ$. Furthermore, $B_1\subset A$ implies $S_\e(\one_{B_1})\leq S_\e(\one_{A})$ and consequently: \[\int \one_A d\QQ'=\int \one_{B_1} d\QQ' \leq \int S_\e(\one_{B_1}) d\QQ\leq \int S_\e(\one_A)d\QQ\]
    Taking the supremum over all $\QQ'\in \UWball \e(\QQ)$ proves the result.    

\end{proof}



\section{Deferred proofs from \cref{sec:degenerate}}\label{app:degenerate}
\subsection{Proof of \cref{lemma:union_degenerate}}\label{app:union_degenerate}

\begin{proof}[Proof of \cref{lemma:union_degenerate}]
    Let $\{D_i\}_{i=1}^\infty$ be a countable sequence of degenerate sets for an adversarial Bayes classifier $A$. Then by \cref{prop:abs_cont_equivalencies}, one can conclude that $S_\e(\one_{A})=S_\e(\one_{A\cup D_i})=\one_{A^\e\cup D_i^\e}$ $\PP_0$-a.e. and $S_\e(\one_{A^C})=S_\e(\one_{(A- D_i)^C})=\one_{(A^C)^\e\cup D_i^\e}$ $\PP_1$-a.e. for every $i$. Taking a supremum then implies that $S_\e(\one_{A})=\one_{A^\e\cup \bigcup_{i=1}^\infty D_i^\e}=S_\e(\one_{A\cup \bigcup_{i=1}^\infty D_i})$ $\PP_0$-a.e. and $S_\e(\one_{A^C})=\one_{(A^C)^\e\cup \bigcup_{i=1}^\infty D_i^\e}=S_\e(\one_{(A- \bigcup_{i=1}^\infty D_i)^C})$ $\PP_1$-a.e.. Therefore, \cref{prop:abs_cont_equivalencies} implies that $A$, $A\cup \bigcup_{i=1}^\infty D_i$, and $A-\bigcup_{i=1}^\infty D_i$ are all equivalent up to degeneracy. Consequently, $\bigcup_{i=1}^\infty D_i$ is a degenerate set.
\end{proof}

\subsection{Proof of \cref{prop:degenerate_connected_component}}\label{app:degenerate_connected_component}

\begin{lemma}\label{lemma:C^e_A}
        Let $A$ be an adversarial Bayes classifier. If $C$ is a connected component of $A$ with $C^{-\e}=\emptyset$, then 
        \begin{equation}
        \label{eq:C^e_A^C}
        C^\e=\{\by\in A^C: \ov{B_\e(\by)}\text{ intersects }C \}^\e
        \end{equation}
        Similarly, if $C$ is a component of $A^C$ with $C^{-\e}=\emptyset$, then 

        \begin{equation}
        \label{eq:C^e_A}
        C^\e=\{\by\in A: \ov{B_\e(\by)}\text{ intersects }C \}^\e
        \end{equation}
\end{lemma}
\begin{proof}
        We will prove \cref{eq:C^e_A^C}, the argument for \cref{eq:C^e_A} is analogous. 
        Assume that $C$ is a component of $A$, \cref{eq:e_interpretation} implies the containment $\supset$ of \cref{eq:C^e_A^C}.       

    Next, we prove the containment $\subset$ in \cref{eq:C^e_A^C}. Specifically, we will show that for every $\bx \in C^\e$, there is a $\by\in A^C$ for which $\bx\in \ov{B_\e(\by)}$ and $\ov{B_\e(\by)}$ intersects $C$.

    To show the opposite containment, we show that for every $\bx\in C^\e$, there is a $\by\in A^C$ for which $\bx \in \ov{B_\e(\by)}$ and $\ov{B_\e(\by)}$ intersects $C$.

    Let $\bx\in C$. 
    Because $C^{-\e}=\emptyset$, \Cref{eq:-e_interpretation} implies that $\ov{B_\e(\bx)}$ is not entirely contained in $C$. Thus the set $C\cup\ov{B_\e(\bx)}$ is connected and strictly contains $C$. Recall that a connected component of a set $A$ is a maximal connected subset. If $\ov{B_\e(\bx)}$ were entirely contained in $A$, $C\cup\ov{B_\e(\bx)}$ would be a connected subset of $A$ that strictly contains $C$, and then $C$ would not be a maximal connected subset of $A$. Therefore, $\ov{B_\e(\bx)}$ contains a point $\by$ in $A^C$, and $\ov{B_\e(\by)}$ intersects $C$ at the point $\bx$.
    
    Next assume that $\bx \in C^\e$ but $\bx \not \in C$. Then \Cref{eq:e_interpretation} states that the ball $\ov{B_\e(\bx)}$ intersects $C$ at some point $\bz$. Consider the line defined by $\ell:=\{t\bx +(1-t)\bz: 0\leq t\leq 1\}$. Again $\ell$ is a connected set that intersects $C$, so $\ell\cup C$ is connected as well. However, $\ell$ also contains a point not in $C$ and thus if $\ell$ were entirely contained in $A$, then $C\cup \ell$ would be a connected subset of $A$ that strictly contains $C$. As $C$ is a maximal connected subset of $A$, the line $\ell$ is not entirely contained in $A$. Let $\by$ be any point in $A^C\cap \ell$, then $\ov{B_\e(\by)}$ intersects $C$ at the point $\bz$ and contains $\bx$.  
\end{proof}

\begin{proof}[Proof of \cref{prop:degenerate_connected_component}]
    First assume that $C$ is a connected component of $A$ with $C^{-\e}=\emptyset$. We will argue that $C\subset (A^\e)^{-\e} -(A^{-\e})^\e$, and then \cref{cor:A_e_diff_degenerate} will imply that $C$ is a degenerate set for $A$. 
    
    If $C$ is a component of $A$, then $C^\e\subset A^\e$ and thus $C\subset (C^\e)^{-\e}\subset (A^\e)^{-\e}$. Next, \cref{eq:C^e_A^C} of \cref{lemma:C^e_A} implies that $C^\e\subset (A^C)^\e$ and thus $C\subset (C^\e)^{-\e}\subset ((A^C)^\e)^{-\e}= ((A^{-\e})^\e)^C$. Therefore, $C$ is disjoint from $(A^{-\e})^\e$. Consequently, $C$ is contained in $(A^\e)^{-\e}-(A^{-\e})^\e$, which is degenerate by \cref{lemma:A_e_-e_containments}.

    The argument for a connected component of $A^C$ is analogous, with \cref{eq:C^e_A} in place of \cref{eq:C^e_A^C}

    As each connected component of $A$ or $A^C$ is contained in the degenerate set $(A^\e)^{-\e}- (A^{-\e})^\e$, it follows that the set in \eqref{eq:components_union_set} is contained in the degenerate set $(A^\e)^{-\e}- (A^{-\e})^\e$. 
\end{proof}

\subsection{Proof of \cref{lemma:degenerate_set_-e}}\label{app:degenerate_set_-e}
\begin{proof}[Proof of \cref{lemma:degenerate_set_-e}]
    We will show that $\PP_0(D^{-\e})=0$, the argument for $\PP_1$ is analogous. 
    As both $A-D$ and $A\cup D$ are adversarial Bayes classifiers, \cref{prop:abs_cont_equivalencies} implies that $\PP_0((A-D)^\e\cup D^\e)=\PP_0((A-D)^\e)$ and thus $\PP_0( D^\e-(A-D)^\e)=0$. However, \cref{eq:e_interpretation} and \cref{eq:-e_interpretation} imply that 
    \begin{align*}
        &D^\e-(A-D)^\e=\{\bx: \ov{B_\e(\bx)} \text{ intersects } D\text{ but not }A-D\}\\
        &\supset \{\bx: \ov{B_\e(\bx)}\subset D\} =D^{-\e}   
    \end{align*}
    Thus $\PP_0(D^{-\e})=0$.
    
\end{proof}
\section{Deferred proofs from \cref{sec:regular_adv_bayes}--- proof of \cref{lemma:^e_to_order}} \label{app:1d_necessary}
\label{app:exists regular}
In this appendix, we adopt the same notational convention as \cref{sec:necessary}
 regarding the $a_i$s and $b_i$s: Namely, when $A=\bigcup_{i=m}^M (a_i,b_i)$ is a regular adversarial Bayes classifier, $a_{M+1}$ is defined to be $+\infty$ if $M$ is finite and $b_{m-1}$ is defined to be $-\infty$ if $m$ is finite.

\begin{proof}[Proof of \cref{lemma:^e_to_order}]
    Let $A=\bigcup_{i=m}^M (a_i,b_i)$ be a countable union of disjoint intervals, each of length at least $2\e$, so that $b_i-a_i\geq 2\e$. The goal is to re-index the intervals $\{(a_{i_j},b_{i_j})\}_{j=m}^M$ in such a way that $b_{i_j}<a_{i_j}$ for all $j<M$. In the following, we will prove that such an index $\{i_j\}$ exists for which $A=\bigcup_{j=m}^M (a_{i_j},b_{i_j})$. First pick $j_0$ to be the index for which either:
    \begin{itemize}
        \item $0\in (a_{j_0},b_{j_0})$ if $0\in A$
        \item $j_0\in \argmin_i |a_i|$ otherwise
    \end{itemize}
    Notice that the quantity $j_0\in \argmin_i |a_i|$ is well-defined because the condition $b_i-a_i\geq 2\e$ implies that $|a_i-a_k|\geq 2\e$ for any two $a_i$, $a_k$. 

    Consequently, 
    \[A-(a_{j_0},b_{j_0})=((A-(a_{j_0},b_{j_0}))\cap \Rset^{>0}) \cup ((A-(a_{j_0},b_{j_0}))\cap \Rset^{<0})\]
    It remains to express $(A-(a_{j_0},b_{j_0}))\cap \Rset^{>0})$ as $\bigcup_{j=1}^M (a_{i_j},b_{i_j})$ and $(A-(a_{j_0},b_{j_0}))\cap \Rset^{>0})$ as $\bigcup_{j=m}^{-1} (a_{i_j},b_{i_j})$.
    We consider the set $(A-(a_{j_0},b_{j_0}))\cap \Rset^{>0})$ noting that the argument for the other option is analogous. 
    
    If $(A-(a_{j_0},b_{j_0}))\cap \Rset^{>0})$ can in fact be expressed as a finite union of intervals then we're done. Thus we assume this set cannot be expressed as a finite union of intervals, and we define the index sequence $\{i_j\}_{j=0}^\infty$ inductively. Given $i_j$, define $i_{j+1}$ by 
    \[i_{j+1}=\argmin\{a_k: a_k>a_{i_j}\}\]
    Again, this $\argmin$ is well defined because the condition $b_i-a_i\geq 2\e$ implies that $|a_i-a_k|\geq 2\e$ for any two $a_i$, $a_k$. This condition further implies that 
    \[\bigcup_{j=1}^K (a_{i_j},b_{i_j})\supset (A-(a_{j_0},b_{j_0})) \cap (0,2K\e).\]
    Taking $K\to \infty$ yeilds
    \[\bigcup_{j=1}^\infty (a_{i_j},b_{i_j})\supset (A-(a_{j_0},b_{j_0})) \cap \Rset^{>0}.\]

    Since the reverse inclusion also holds, we conclude that 
       \[\bigcup_{j=1}^\infty (a_{i_j},b_{i_j})= (A-(a_{j_0},b_{j_0})) \cap \Rset^{>0}.\]
\end{proof}

\section{Deferred Proofs from \cref{sec:degenerate_1d}}\label{app:degenerate_1d}
\subsection{Proof of \cref{lemma:degenerate_1d_eta_0_1}}\label{app:degenerate_1d_eta_0_1}
First, we show \cref{lemma:degenerate_1d_eta_0_1} for intervals near the boundary of $\supp \PP$.
\begin{lemma}\label{lemma:transition_eta_0_1}
    Assume $\PP\ll \mu$ and let $A=\bigcup_{i=m}^M (a_i,b_i)$ be a regular adversarial Bayes classifier for radius $\e$. Let $y$ represent any of the $a_i$s or $b_i$s. Let $I$ be an interval for which $\supp \PP\subset I$
    \begin{itemize}
        \item Assume that $I=[\ell,\infty)$ or $I=[\ell,r]$.\\ If $y\in (\ell-\e,\ell+\e]$ then $[\ell-\e,y]$ is a degenerate set. If furthermore $\supp \PP=I$, then for some $\delta>0$, either $\eta\equiv 0$ or $\eta\equiv 1$ $\mu$-a.e. on $[\ell,\ell+\delta]$.
        \item Assume that $I=(-\infty,r]$ or $I=[\ell,r]$.\\ If $y\in [r-\e,r+\e)$ then $[y,r+\e]$ is a degenerate set. If furthermore $\supp \PP=I$, then for some $\delta>0$, either $\eta\equiv 0$ or $\eta\equiv 1$ $\mu$-a.e. on $[r-\delta,r]$.
    \end{itemize}
\end{lemma}
\begin{proof}
    We will prove the first bullet; the second bullet follows from the first by considering distributions with densities $\td p_0(x)=p_0(-x)$ and $\td p_1(x)=p_1(-x)$.

    Assume that some $y=a_i$ is in $(\ell-\e, \ell+\e]$, the argument for $y=b_i$ is analogous. Then because $A$ is adversarial Bayes classifier: 
    \begin{equation}\label{eq:eta_1_conclusion}
        0\geq R^\e(A)-R^\e(A\cup [\ell-\e,a_i])=\int_\ell^{a_i+\e }pdx-\int_\ell^{a_i+\e} p_0dx=\int_{\ell}^{a_i+\e} p_1(x)dx.    
    \end{equation}
    Consequently, $\int_\ell^{a_i+\e}p_1(x)dx=0$ and thus the set $A\cup [\ell-\e,a_i]$ must be an adversarial Bayes classifier as well.

    Next, we prove that the interval $[\ell-\e, a_i]$ is a degenerate set. Let $D_1$, $D_2$ be arbitrary measurable subsets of $[\ell -\e, a_i]$. Then 
    \[R^\e((A\cup D_1)-D_2)-R^\e(A\cup [\ell-\e,a_i])\leq \int_\ell^{a_i+\e} pdx-\int_{\ell}^{a_i+\e} p_0 dx=\int_\ell^{a_i+\e}p_1(x)dx \]
    and this quantity must be zero by \cref{eq:eta_1_conclusion}. Therefore, the set $(A\cup D_1)-D_2$ is an adversarial Bayes classifier.

    Next, we will show that if $\supp \PP=I$, then $\eta=0$ $\mu$-a.e. on a set of positive measure.     By assumption, $a_i>\ell-\e$ and thus $\delta=a_i+\e-\ell>0$. As $[\ell, \ell+\delta]\subset \supp \PP$, \Cref{eq:eta_1_conclusion} implies that $\eta\equiv 0$ $\mu$-a.e. on $[\ell, \ell+\delta]$. 
\end{proof}

\begin{proof}[Proof of \cref{lemma:degenerate_1d_eta_0_1}]
    First, \cref{thm:equivalence_up_to_degeneracy} implies that two equivalent adversarial Bayes classifiers have the same degenerate sets. Consequently, \cref{thm:adv_bayes_and_degenerate} implies that we can assume that $A$ is regular without loss of generality. Next, assume that the endpoints of $I$ are $d_1, d_2$, so that $I=[d_1,d_2]$ (\cref{cor:degenerate_size} implies that $|I|<\infty$). Define an interval $J$ via
    \[J=\bigcup_{\substack{I'\supset I:\\I\text{ degenerate interval}}} I'\]
    Because each interval $I'$ includes $I$, the interval $J$ can be expressed as a countable union of intervals of length at least $|I|$ and thus is
    a degenerate set as well by \cref{lemma:union_degenerate}. The interval $J$ must be closed because the boundary of every adversarial Bayes classifier is a degenerate set when $\PP\ll \mu$. If $J\cap (\supp \PP^\e-\interior( \supp \PP^{-\e}))$ is nonempty, \cref{lemma:transition_eta_0_1} implies that $\eta\in \{0,1\}$ on a set of positive measure under $\PP$. It remains to consider the case $J\subset \interior (\supp \PP^{-\e})$, discussed in the main text. \Cref{cor:degenerate_size} implies that $J$ has finite length and so one can express $J$ as $J=[d_3,d_4]$. 
    Now if any point $\{x\}$ in  $[d_3-\e,d_3)$ were a degenerate set, then \cref{lemma:union_degenerate} and \cref{lemma:enlarge_degenerate} would imply that $((J\cup \{x\})^\e)^{-\e}=[x,d_4]$ would be a degenerate interval strictly containing $J$, which would contradict the definition of $J$. Thus $[d_3-\e,d_3)$ cannot contain any degenerate sets. Furthermore, if this interval contains both points in $A$ and $A^C$, then it must contain points in $\partial A$, which was assumed to be regular. However, the boundary of a regular adversarial Bayes classifier is always a degenerate set (see \cref{lemma:closure_interior_adversarial_Bayes}). Thus $[d_3-\e,d_3)$ must be contained entirely in $A$ or $A^C$. Similarly, $(d_4,d_4+\e]$ must be contained entirely in $A$ or $A^C$. 

    We will analyze the two cases $(d_3-\e,d_3]$, $[d_4,d_4+\e)\subset A$ and $(d_3-\e,d_3]\subset A$, $[d_4,d_4+\e)\subset A^C$. The cases 
    $(d_3-\e,d_3]$, $[d_4,d_4+\e)\subset A^C$ and $(d_3-\e,d_3]\subset A^C$, $[d_4,d_4+\e)\subset A$ are analogous.   

    Assume first that $(d_3-\e,d_3]$, $[d_4,d_4+\e)\subset A$. Then because $J\subset \supp \PP$ is degenerate, \cref{cor:degenerate_size} implies that $|J|\leq 2\e$. Hence one can conclude
    \[0=R^\e(A- J)-R^\e(A\cup J)= \int_{d_3-\e}^{d_4+\e} p(x) dx -\int_{d_3-\e}^{d_4+\e} p_0(x)dx= \int_{d_3-\e}^{d_4+\e} p_1(x)dx\geq \int_{d_1-\e}^{d_2+\e} p_1(x)dx.\]

    Thus on the interval $[d_1-\e, d_2+\e]$, one can conclude that $p_1(x)=0$ $\mu$-a.e. As $[d_1,d_2]\subset \interior (\supp \PP^{-\e})$ and $d_2>d_1$, the interval $[d_1-\e,d_2+\e]$ intersects $\supp \PP$ on an open set. Thus $\eta(x)=0$ $\mu$-a.e. on a set of positive measure. 

    Next assume that $(d_3-\e,d_3]\subset A$, $[d_4,d_4+\e)\subset A^C$. Again, \cref{cor:degenerate_size} implies that $|I|\leq 2\e$. Then:
    \begin{align*}
        0=R^\e((A\cup (J\cap \QQ))-(J\cap \QQ^C))-R^\e(A\cup J)&= \int_{d_3-\e}^{d_4+\e}p(x) dx- \left(\int_{d_3-\e}^{d_4-\e} p_0(x)dx+\int_{d_4-\e}^{d_4+\e} p(x)dx\right)\\
        &\geq \int_{d_3-\e}^{d_4+\e} p_1(x)dx\geq \int_{d_1-\e}^{d_2-\e} p_1(x)dx
    \end{align*}
    Thus $p_1(x)=0$ on $[d_1-\e,d_2-\e]$. 
    

    Now $[d_1,d_2]\subset \interior (\supp \PP^{-\e})$ implies that $[d_1-\e,d_2-\e]$, intersects $\supp \PP$ on an open interval. Thus $\eta(x)=0$ on a set of positive measure.
\end{proof}

\subsection{Proof of \cref{lemma:max_degenerate_set,cor:fourth_bullet}}\label{app:identify_degenerate_eta_0_1}
The following lemma implies $(\supp \PP^\e)^C$ is a degenerate set.
\begin{lemma}\label{lemma:e_non_disjoint}
    If $A$ and $B^\e$ are disjoint, then $A^\e$ and $B$ are disjoint.
\end{lemma}
\begin{proof}
    We will show the contrapositive of this statement: if $A^\e$ and $B$ intersect, then $A$ and $B^\e$ intersect.

    If $A^\e$ an $B$ intersect, then there are $\ba\in A$, $\bb\in B$ and $\bh\in \ov{B_\e(\zero)}$ for which $\ba+\bh=\bb$ and thus $\ba=\bb-\bh\in B^\e$. Thus $A$ and $B^\e$ intersect. 
\end{proof}
Next, we argue that the set $\ov{(\supp \PP^\e)^C}\cup \partial A$ is indeed degenerate for any regular adversarial Bayes classifier $A$. The proof of this result relies on \cref{lemma:boundary^e_leb_zero}.

\begin{proof}[Proof of \cref{lemma:max_degenerate_set}]
    First, $\supp \PP^\e$ and $(\supp \PP^\e)^C$ are disjoint, so \cref{lemma:e_non_disjoint} implies that $\supp \PP$ and $((\supp \PP^\e)^C)^\e)$ are disjoint. Thus $\PP(((\supp \PP^\e)^C)^\e)=0$, and so $(\supp \PP^\e)^C$ is a degenerate set. Next, \cref{lemma:closure_interior_adversarial_Bayes} implies that both $\ov{ (\supp \PP^\e)^C}$ and $\partial A$ are degenerate sets. Finally, \cref{lemma:union_degenerate} implies that the union of $\partial A$ and $\partial (\supp \PP^\e)^C$ is a degenerate set.
\end{proof}

\begin{proof}[Proof of \cref{cor:fourth_bullet}]
    Any adversarial Bayes classifier $A$ is equivalent up to degeneracy to a regular adversarial Bayes classifier $A_r$. Because two equivalent adversarial Bayes classifiers have the same degenerate sets,   \cref{lemma:fourth_bullet_for_regular} implies that $A\cap \interior(\supp \PP^\e)$ and $A_r\cap \interior(\supp \PP^\e)$ can differ only at the boundary of $\partial A_r$. Consequently, $\partial A=\partial A_r$ and $A$ must be a regular set as well, so \cref{lemma:fourth_bullet_for_regular} applies also to $A$.
\end{proof}

\section{Deferred proofs from \cref{sec:increasing_e}}
\label{app:increasing_e}
In this appendix, we adopt the same notational convention as \cref{sec:necessary}
 regarding the $a_i$s and $b_i$s: Namely, when $A=\bigcup_{i=m}^M (a_i,b_i)$ is a regular adversarial Bayes classifier, $a_{M+1}$ is defined to be $+\infty$ if $M$ is finite and $b_{m-1}$ is defined to be $-\infty$ if $m$ is finite.

    The following observation will assist in proving the first bullet of \cref{lemma:inclusion_basic}.
    \begin{lemma}\label{lemma:R_e_constant}
        Let $\e_2>\e_1\geq 0$. If $\Rset$ minimizes $R^{\e_2}$ but $\emptyset$ minimizes $R^{\e_1}$, then the sets $\Rset$ and $\emptyset$ minimize both $R^{\e_1}$ and $R^{\e_2}$. 

        Similarly, if $\emptyset$ minimizes $R^{\e_2}$ but $\Rset$ minimizes $R^{\e_1}$, then the sets $\Rset$ and $\emptyset$ minimize both $R^{\e_1}$ and $R^{\e_2}$. 
    \end{lemma}
    \begin{proof}
        
        First, assume that $\Rset$ minimizes $R^{\e_2}$ and $\emptyset$ minimizes $R^{\e_1}$.  
        The quantities  
        \[R^\e(\Rset)=\int_{\Rset} d\PP_0 \quad R^\e(\emptyset)=\int_{\Rset} d\PP_1\]
        are independent of the value of $\e$. Next, notice that $R^{\e_2}(A)\geq R^{\e_1}(A)$ for an set $A$. Therefore,
        \[R^{\e_2}_*\geq R^{\e_1}_*=R^{\e_1}(\emptyset) = R^{\e_2}(\emptyset),\] 
        where $R_*^\e=\inf_A R^\e(A)$. Thus $\emptyset$ also minimizes $R^{\e_2}$. As a result, the sets $\Rset$ and $\emptyset$ achieve the same $R^{\e_2}$ risk, and so 
        \[R^{\e_1}(\Rset)=R^{\e_2}(\Rset)=R^{\e_2}(\emptyset)=R^{\e_1}(\emptyset).\]
        Consequently, $\Rset$ is also a minimizer of $R^{\e_1}$.

        Next, swapping the roles of $\PP_0$ and $\PP_1$ shows that if $\emptyset$ minimizes $R^{\e_2}$ and $\Rset$ minimizes $R^{\e_1}$ then $\Rset, \emptyset$ minimize both $R^{\e_1}$ and $R^{\e_2}$
    \end{proof}
    
    Next, recall that \cref{lemma:transition_eta_0_1} implies that if the an endpoint of an adversarial Bayes classifier is too close to the boundary of $\supp \PP$, then that endpoint must be in the boundary of a degenerate interval.
As a result:
\begin{corollary}\label{cor:transition_eta_0_1}
    Assume $\PP\ll \mu$ is a measure for which $\supp \PP$ is an interval $I$, and $\PP(\eta=0\text{ or }1)=0$. Then if $A$ is an adversarial Bayes classifier for radius $\e$, then $\partial A$ is disjoint from $\interior(I^\e)-\interior (I^{-\e})$.
\end{corollary}

This result implies that one only need consider $a_i^1,b_i^1,a_j^2,b_j^2$ contained in $\interior(I^{-\e_2})$ in the proof of \cref{lemma:inclusion_basic}.

Next, we restate \cref{lemma:inclusion_basic} in a form that avoids a regularity assumption on $A_1$, so that this result can apply to perturbation radius $\e_1$ as well.

\begin{lemma}\label{lemma:inclusion_basic_0}
    Assume that $\PP\ll \mu$ is a measure for which $\supp \PP$ is an interval $I$ and $\PP(\eta(x)=0\text { or } 1)=0$. Let $A_1$ and $A_2$ be adversarial Bayes classifier corresponding to perturbation sizes $\e_2>\e_1\geq 0$ respectively, and assume $A_2$ regular. Let $C^i$ be a connected component of $A_i$ and $K^i$ a connected component of $A_i^C$. 
    \begin{itemize}
        \item If both $\Rset$ and $\emptyset$ are adversarial Bayes classifiers for perturbation radius $\e_1$, then both $\Rset$ and $\emptyset$ are adversarial Bayes classifiers for perturbation radius $\e_2$.
        \item Assume that $\Rset$ and $\emptyset$ are not both adversarial Bayes classifiers for perturbation radius $\e_1$. Then for each connected component $C^1$ of $A_1$, the set $C^1\cap \interior(I^{\e_1})$ cannot contain any non-empty $\interior(K^2)\cap \interior(I^{\e_1})$ and for each connected component $K^1$ of $A_1^C$, the set $\interior(K^1)\cap \interior(I^{\e_1})$ cannot contain any non-empty $C^2\cap \interior(I^{\e_1})$.
    \end{itemize}
    
\end{lemma}
     The assumption that $A_2$ is regular is made for notational convenience in the proof; it is not essential for the validity of the claim. The second bullet involves $C^1$ but $\interior(K^2)$. This asymmetry prevents cases where the two sets intersect only at an endpoint. For example, if $C^1=[a_i,b_i]$ and $K^2=[b_i,a_{i+1}]$, then $C^1\cap K^2={b_i}$ while $C^1\cap \interior(K^2)=\emptyset$. When $\e_1=0$, we cannot assume that $A_1$ is regular, or that is can be conveniently expressed as $A_1=\bigcup_{i=m}^M (a_i^1,b_i^1)$ with $a_i<b_i<a_{i+1}$, see \cref{ex:counterexample_Bayes} for a counterexample.
\begin{proof}[Proof of \cref{lemma:inclusion_basic_0}]
    We will show that $K^1\cap \interior(I^{\e_1})$ does not include any non-empty $\interior(C^2)\cap \interior(I^{\e_1})$, the argument for $C^1\cap \interior(I^{\e_1})$ and $(\interior(K^2)\cap \interior(I^{\e_1})$ is analogous.

    As $A_2$ is regular, \cref{lemma:^e_to_order} implies that the interior of this set can be expressed as
    \[\interior(A_2)=\bigcup_{j=n}^N (a_j^2,b_j^2)\]
    with $-\infty\leq n\leq N-1\leq +\infty$ and $a_j^2<b_j^2<a_{j+1}^2$. Consequently, every connected component $C^2$ of $\interior(A_2)$ equals $(a_j^2,b_j^2)$ for some $j$ while every connected component $K^2$ of $\interior(A_2^C)$ equals $(b_{j}^2, a_{j+1}^2)$ for some $j$. Next, let the left and right endpoints of $K^1$ be $b^1$ and $a^1$ respectively so that
    \[\interior(K^1)=(b^1,a^1).\]
    
    Fix an interval $(a_j^2,b_j^2)$ and assume that $(a_j^2,b_j^2)\cap \interior(I^{\e_1})\neq \emptyset $, $(a_j^2,b_j^2)\cap \interior(I^{\e_1})\subset \interior(K^1)\cap \interior(I^{\e_1})$.

    First, notice that the assumption $\eta\neq 0,1$ implies that none of the $a_j^2$s, $b_j^2$s are in $\interior(I^{\e_2})-\interior (I^{-\e_2})$ due to \cref{cor:transition_eta_0_1}. Thus because the intersection $(a_j^2,b_j^2)\cap \interior(I^{\e_1})$ is non-empty, then either $\interior(I^{\e_2})\subset (a_j^2,b_j^2)$ or at least one endpoint of $(a_j^2,b_j^2)$ is in $\interior (I^{-\e_2})$. 

    If in fact $(a_j^2,b_j^2)\supset \interior(I^{\e_2})$, then $\interior(K^1)\cap \interior(I^{\e_1})\supset (a_j^2,b_j^2)\cap \interior(I^{\e_1})$ must include $\interior(I^{\e_1})$. Thus 
    $R^{\e_1}(A_1)=R^{\e_1}(\emptyset)$ while $R^{\e_2}(A_2)=R^{\e_2}(\Rset)$. \Cref{lemma:R_e_constant} then implies that $\Rset$, $\emptyset$ are both adversarial Bayes classifiers for both perturbation sizes $\e_1$ and $\e_2$, which implies the first bullet of \cref{lemma:inclusion_basic}.
    
    Thus, to show the second bullet of \cref{lemma:inclusion_basic}, it remains to consider $(a_j^2,b_j^2)\not \supset \interior(I^{\e_2})$.  
    We will show $(a_j^2,b_j^2)\not \supset \interior(I^{\e_2})$ 
    but $(a_j^2,b_j^2)\cap \interior(I^{\e_1})\neq \emptyset$ and $(a_j^2,b_j^2)\cap \interior(I^{\e_1})\subset \interior(K^1)\cap \interior(I^{\e_1})$ results in a contradiction.

     Because $A_2$ is regular the inclusion $\interior(K^1)\cap \interior(I^{\e_1})\supset (a_j^2,b_j^2)\cap \interior(I^{\e_1})$ implies that $b^1-a^1>2\e_2$. As $b_j^2-a_j^2>2\e_2$ and the interval $(a_j^2,b_j^2)$ is included in the adversarial Bayes classifier $A_2$, it follows that $R^\e(A_2)\leq R^\e(A_2-(a_j^2,b_j^2))$ which implies 
    \[ \int_{a_j^2-\e_2}^{a_j^2+\e_2} p dx +\int_{a_j^2+\e_2}^{b_j^2-\e_2} p_0 dx+\int_{b_j^2-\e_2}^{b_j^2+\e_2} p dx\leq \int_{a_j^2-\e_2}^{b_j^2+\e_2} p_1dx\]
    and consequently 
    \begin{equation}\label{eq:p_comp_minimal_forward}
        \int_{a_j^2-\e_2}^{b_j^2+\e_2} p_0dx\leq \int_{a_j^2+\e_2}^{b_j^2-\e_2} p_1 dx. 
    \end{equation}

     Next, $b_j^2-a_j^2>2\e_2$ implies $(b_j^2-(\e_2-\e_1))-(a_j^2+(\e_2-\e_1))>2\e_1$. Notice that 
     \[(a_j^2+\e_2-\e_1,b_j^2-(\e_2-\e_1))\cap \interior(I^{\e_1})\subset (a_j^2,b_j^2)\cap \interior(I^{\e_1})\subset K^1\cap \interior(I^{\e_1})\] 
     is then a connected component of $\left(A_1\cup (a_j^2+(\e_2-\e_1),b_j^2-(\e_2-\e_1))\right)\cap \interior(I^{\e_1})$. Therefore,
     \begin{equation}\label{eq:adv_risk_modified_classifier}
     \begin{aligned}
         R^{\e_1}(A_1)-R^{\e_1}\left(A_1\cup \left(a_j^2+(\e_2-\e_1),b_j^2-(\e_2-\e_1)\right)\right) =\\
         \int_{c_j}^{d_j} p_1dx - \left( \int_{c_j}^{a_j^2+\e_2} p dx+\int_{a_j^2+\e_2}^{b_j^2-\e_2} p_0dx+ \int_{b_j^2-\e_2}^{d_j} pdx \right)
    \end{aligned}
     \end{equation}
    where $c_j= \max(b^1+\e_1,a_{j}^2+\e_2-2\e_1)$ and $d_j= \min(a^1+\e_1,b_j^2-\e_2+2\e_1)$. We will now argue that this quantity is positive, contradicting the fact that $A_1$ is an adversarial Bayes classifier.  
    
    Adding 
    \[\int_{c_j}^{a_j^2+\e_2} p_1 dx+ \int_{b_j^2-\e_2}^{d_j} p_1dx\] to both sides of \cref{eq:p_comp_minimal_forward} implies that
    \begin{align}
        &\int_{c_j}^{d_j}  p_1dx
        \geq \int_{c_j}^{a_j^2+\e_2} p dx+ \int_{b_j^2-\e_2}^{d_j} pdx+\int_{a_j^2+\e_2}^{b_j^2-\e_2} p_0dx +\int_{a_j^2-\e_2}^{c_j} p_0dx +\int_{d_j}^{b_j^2+\e_2}p_0dx\nonumber\\
        &>\int_{c_j}^{a_j^2+\e_2} p dx+ \int_{b_j^2-\e_2}^{d_j} pdx+\int_{a_j^2+\e_2}^{b_j^2-\e_2} p_0dx \label{eq:inequality_to_components}
    \end{align}
     We now prove that this last inequality is in fact strict. First, recall that the interval $(a_j^2,b_j^2)$ does not contain $\interior(I^{\e_2})$ and thus \cref{cor:transition_eta_0_1} implies that at least one of $a_j^2$, $b_j^2$ must be in $\interior(I^{-\e_2})$. Consequently, $\interior \supp \PP$ must contain at least one of $a_j^2-\e_2$ and $b_j^2+\e_2$. Lastly, $c_j-(a_j^2-\e_2)\geq 2(\e_2-\e_1)>0$ and $b_j^2+\e-d_j\geq 2(\e_2-\e_1)>0$ and thus at least one of the intervals $[a_j^2-\e,c_j]$, $[d_j,b_j^2+\e_2]$ must have positive $\PP$-measure. The assumption $\PP(\eta=0\text{ or }1)=0$ implies $\supp \PP_0=\supp \PP_1$ and consequently one of these intervals must have positive $\PP_0$-measure. 
    
    The strict inequality in \cref{eq:inequality_to_components} together with \cref{eq:adv_risk_modified_classifier} implies $R^{\e_1}(A_1)-R^{\e_1}(A_1 \cup (a_j^2+\e_2-\e_1,b_j^2-(\e_2-\e_1))>0$, which contradicts the fact that $A_1$ is an adversarial Bayes classifier. 
\end{proof}

The following result reduces general $A_2$ to regular $A_2$, allowing the application of \cref{lemma:inclusion_basic_0}:
\begin{lemma}\label{lemma:regular_match}
    Assume $d=1$, $\PP\ll\mu$, $\e>0$, $\supp \PP$ is an interval, and $\PP(\eta=0\text{ or }1)=0$. Then for any adversarial Bayes classifier $A$, there is a regular adversarial Bayes classifier $A_r$ for which $A\cap \interior(\supp \PP^\e)= A_r\cap  \interior(\supp \PP^\e)$
\end{lemma}
    \begin{proof}
        Let $A_r$ be the regular adversarial Bayes classifier equivalent to $A$. The fourth bullet of \cref{thm:1d_degenerate} implies that $(A\triangle A_r) \cap \interior (\supp \PP)^\e\subset \partial A_r$. However, if $D_1,D_2\subset \partial A_r$, then $A_r\cup D_1-D_2$ is also a regular set. Taking $D_1=A-A_r$ and $D_2=A_r-A$ establishes the claim.    
    \end{proof}

\Cref{thm:adv_bayes_increasing_e} then directly follows from \cref{lemma:inclusion_basic_0}.

\begin{proof}[Proof of \cref{thm:adv_bayes_increasing_e}]
    We will show the statement for $A_1\cap \interior(I^{\e_1})$ and $A_2\cap \interior(I^{\e_1})$, the argument for $A_1^C\cap \interior(I^{\e_1})$ and $A_2^C\cap \interior(I^{\e_1})$ is analogous. Namely, we will prove that either $\comp(A_1\cap \interior(I^{\e_1}))\geq \comp(A_2\cap \interior(I^{\e_1}))$ or $\Rset,\emptyset$ minimize both $R^{\e_1}$ and $R^{\e_2}$.

    If $\comp(A_1\cap \interior (I^{\e_1}))=+\infty$ then there is nothing to prove. Thus, we assume $A_1\cap \interior (I^{\e_1})$ has finitely many components.

    Next, as $\e_2>0$, \cref{lemma:regular_match} implies that it suffices to consider regular $A_2$. Furthermore, the fact that $A_2$ is regular and $\interior(I^{\e_1})$ is open implies that $\comp(A_2\cap \interior(I^{\e_1}))=\comp(\interior(A_2)\cap \interior(I^{\e_1}))$.
    
     As $A_2$ is regular we can write
    \[\interior(A_2)=\bigcup_{j=n}^N (a_j^2,b_j^2).\]
    with $a_i^2<b_i^2<a_{i+1}^2$, $n$ and $N$ finite with $n\leq N-1$, and $-\infty\leq n\leq N-1\leq +\infty$. 
    
     The first bullet of \cref{lemma:inclusion_basic_0} implies that that if both $\emptyset$, $\Rset$ are adversarial Bayes classifiers for perturbation radius $\e_1$, then both are adversarial Bayes classifiers for perturbation radius $\e_2$.

    Next, assume that for perturbation size $\e_1$, the sets $\Rset$, $\emptyset$ are not both adversarial Bayes classifiers.
    \Cref{cor:transition_eta_0_1} implies that there are no $a_j^2,b_j^2\in \interior(I^{\e_2})-\interior(I^{-\e_2})$. As $\interior(I^{-\e_2})\subset \interior(I^{\e_1})\subset \interior(I^{\e_2})$ are all intervals which are connected sets, one can conclude that $\comp(\interior(A_2)\cap \interior(I^{\e_2}))=\comp(\interior(A_2)\cap \interior(I^{\e_1}))$. Therefore, it remains to show that $\comp(A_1\cap \interior(I^{\e_1}))\geq \comp(\interior(A_2)\cap \interior(I^{\e_1}))$.

    Let $C^1$ be any component of $A_1\cap \interior I^{\e_1}$. Because $\interior(I^{\e_1})$ is an interval, the intersections $C^1\cap \interior(I^{\e_1})$, $(b_j^2,a_{j+1}^2)\cap \interior(I^{\e_1})$ 
    are intervals for $j\in [n,N]$. If the interval $C^1\cap \interior(I^{\e_1})$ intersects both the intervals $(a_j^2,b_j^2)\cap \interior(I^{\e_1})$ 
    and $(a_{j+1}^2,b_{j+1}^2)\cap \interior(I^{\e_1})$ for some $j$, then $C^1\cap \interior(I^{\e_1})$ must contain $(b_j^2,a_{j+1}^2)\cap \interior(I^{\e_1})$ for some $j$, 
    which contradicts \cref{lemma:inclusion_basic_0}. Thus there is at most one interval $(a_j^2,b_j^2)\cap \interior(I^{\e_1})$ for each component $C^1$ of $A_1\cap \interior(I^{\e_1})$, which implies that $\comp(A_1\cap \interior(I^{\e_1}))\geq \comp(\interior(A_2)\cap \interior(I^{\e_1}))=\comp(A_2\cap \interior(I^{\e_2}))$.
\end{proof}
\section{Adversarial Bayes classifier examples and deferred \cref{sec:examples} proofs}\label{app:examples}
The following lemma is helpful for verifying the second order necessary conditions for gaussian mixtures.
\begin{lemma}\label{lemma:gaussian_derivative}
    Let $g(x)=\frac{t}{\sqrt{2\pi}\sigma}e^{-\frac{(x-\mu)^2}{2\sigma^2}}$. Then $g'(x)=-\frac{x-\mu}{\sigma^2}g(x)$.
\end{lemma}
\begin{proof}
    The chain rule implies that 
    \[g'(x)=-\frac{x-\mu}{\sigma^2} \cdot \frac{t}{\sqrt{2\pi}\sigma}e^{-\frac{(x-\mu)^2}{2\sigma^2}}= -\frac{x-\mu}{\sigma^2}g(x)\]
\end{proof}

\subsection{\Cref{ex:gaussians_equal_variances} details}\label{app:gaussians_equal_variances}

It remains to verify two of the claims made in \cref{ex:gaussians_equal_variances}--- namely, 1) that $b(\e)$ does not satisfy the second order necessary condition \cref{eq:second_order_necessary_b}, and 2) Comparing the adversarial risks of $\Rset$, $\emptyset$, $(a(\e),+\infty)$ proves that $(a(\e),+\infty)$ is an adversarial Bayes classifier iff $\e\leq \frac{\mu_1-\mu_0} 2$ and $\Rset$, $\emptyset$ are adversarial Bayes classifiers iff $\e\geq \frac{\mu_1-\mu_0} 2$.

\subsubsection*{1) Showing $b(\e)$ doesn't satisfy the second order necessary condition \cref{eq:second_order_necessary_b}}
Due to \cref{lemma:gaussian_derivative} the equation \cref{eq:second_order_necessary_b} reduces to 
\[p_0'(b(\e)+\e)-p_1'(b(\e)-\e)=-\frac{b(\e)+\e-\mu_0}{\sigma^2}p_0(b(\e)-\e)+\frac{b(\e)-\e-\mu_1}{\sigma^2}p_1(b(\e)+\e)\]
Furthermore, the first order necessary condition $p_0(b(\e)-\e)-p_1(b(\e)+\e)=0$ implies that 
\[p_0'(b(\e)+\e)-p_1'(b(\e)-\e)=\frac{p_1(b+\e)}{\sigma^2}\left( -(b(\e)+\e-\mu_0)+(b(\e)-\e-\mu_1) \right)=\frac{p_1(b+\e)}{\sigma^2}(\mu_0-\mu_1-2\e)\]
This quantity is negative due to the assumption $\mu_1>\mu_0$.

\subsubsection*{2) Comparing the adversarial risks of $\Rset$, $\emptyset$, and $(a(\e),+\infty)$}

First, notice that $R^\e(\emptyset)=R^\e(\Rset)=\frac 12$. 

Thus it suffices to compare the risks of $(a(\e),+\infty)$ and $\Rset$. Let 
\[\Phi(x)=\int_{-\infty}^x \frac 1 {\sqrt{2\pi}} e^{-\frac{t^2}2} dt\]
be the cdf of a standard gaussian. Then $R^\e( (a(\e),+\infty))\leq R^\e(\Rset)$ iff 
\[\int_{-\infty}^{a(\e)+\e} p_1(x)dx +\int_{a(\e)-\e}^{+\infty} p_0(x)dx\leq \int_{-\infty}^{+\infty} p_0(x)dx.\]
This equation is equivalent to 
\[\int_{-\infty}^{a(\e)+\e} \frac 1{\sqrt{2\pi} \sigma} e^{-\frac{(x-\mu_1)^2}{2\sigma^2}} dx\leq \int_{-\infty}^{a(\e)-\e} \frac 1 {\sqrt{2\pi}\sigma}e^{-\frac{(x-\mu_0)^2}{2\sigma^2}}\]
which is also equivalent to 
$\Phi\left( \frac{a(\e)+\e-\mu_1}\sigma\right)\leq \Phi \left( \frac{a(\e)-\e-\mu_0}{\sigma}\right)$.
As the function $\Phi$ is strictly increasing, this relation is equivalent to the inequality
\[ \frac{a(\e)+\e-\mu_1}\sigma\leq \frac{a(\e)-\e-\mu_0}{\sigma}\] which simplifies as $\e\leq \frac{\mu_1-\mu_0}2$. Therefore, $(-\infty,a(\e))$ is an adversarial Bayes classifier iff  $\e\leq \frac{\mu_1-\mu_0}2$ and $\Rset,\emptyset$ are adversarial Bayes classifiers iff  $\e\geq \frac{\mu_1-\mu_0}2$.

\subsection{\Cref{ex:gaussians_equal_means} details}\label{app:gaussians_equal_means_details}
The constant $k=\ln \frac{\lambda \sigma_0}{(1-\lambda)\sigma_1}$ will feature prominently in subsequent calculations, notice that the assumption $\frac \lambda{\sigma_1}>\frac{1-\lambda}{\sigma_0}$ implies that $k>0$. The equation \cref{eq:first_order_necessary_b} requires solving $\frac{1-\lambda}{\sigma_0} e^{-(b+\e)^2/2\sigma_0^2}=\frac{\lambda}{\sigma_1} e^{-(b-\e)^2/2\sigma_1^2}$, with solutions \cref{eq:b(e)_def} and 
\begin{equation}\label{eq:y(e)_equal_means_gaussians}
    y(\e)=\frac{\e\left( \frac 1 {\sigma_1^2}+\frac 1 {\sigma_0^2}\right)-\sqrt{\frac{4\e^2}{\sigma_0^2\sigma_1^2}+2\left(\frac 1 {\sigma_1^2}-\frac 1 {\sigma_0^2} \right)k}}{\frac 1 {\sigma_1^2}-\frac 1 {\sigma_0^2}}.
\end{equation}
    The discriminant is positive as $k>0$ and $\sigma_0>\sigma_1$. However, one can show that $y(\e)$ does not satisfy the second order necessary condition \cref{eq:second_order_necessary_b} (see below).  Similarly, the only solution to the necessary conditions \cref{eq:first_order_necessary_a} and \cref{eq:second_order_necessary_a} is $a(\e)=-b(\e)$.

    Thus there are five candidate sets for the adversarial Bayes classifier: $\emptyset$, $\Rset$, $(-\infty,b(\e))$, $(a(\e),+\infty)$ and $(a(\e),b(\e))$. \Cref{thm:1d_degenerate} implies that none of these sets could be equivalent up to degeneracy. By comparing the adversarial classification risks, one can show that $(a(\e),b(\e))$ has the strictly smallest adversarial classification risk from these five options (see below). Therefore, $(a(\e),b(\e))$ is the adversarial Bayes classifier for all $\e$.

It remains to verify two of the claims above--- namely, 1) that $y(\e)$ does not satisfy the second order necessary condition \cref{eq:second_order_necessary_b}, and 2) proving that $(a(\e),b(\e))$ is always the adversarial Bayes classifier by comparing the risks of $(a(\e),b(\e))$, $\Rset$, $\emptyset$, $(a(\e),\infty)$, and $(-\infty,b(\e))$.

\subsubsection*{1) The point $y(\e)$ does not satisfy the second order necessary condition \cref{eq:second_order_necessary_b}}

First, notice that 
    \begin{equation}\label{eq:y_e_bound}
        y(\e)\leq \frac{\e\left(\frac 1 {\sigma_1^2}+\frac 1 {\sigma_0^2}\right)-\sqrt{\frac{4\e^2}{\sigma_0^2\sigma_1^2}}}
        {\frac 1 {\sigma_1^2}-\frac 1 {\sigma_0^2}}=\frac{\e\left(\frac 1 {\sigma_1^2}+\frac 1 {\sigma_0^2}\right)-\frac{2\e}{\sigma_0\sigma_1}}
        {\frac 1 {\sigma_1^2}-\frac 1 {\sigma_0^2}}
    \end{equation}
    This bound will show that $y(\e)$ fails to satisfy the second order necessary condition \cref{eq:second_order_necessary_b}. Next, we compute the derivative $p_i'$ in terms of $p_i$ using \Cref{lemma:gaussian_derivative}. Specifically, $p_i'(x)=\frac{-x}{\sigma_i^2}p_i(x)$
    and therefore
    \[p_0'(y(\e)+\e)-p_1'(y(\e)-\e)=-\frac{y(\e)+\e}{\sigma_0^2} p_0(y(\e)+\e)+\frac{y(\e)-\e}{\sigma_1^2}p_1(y(\e)-\e)\]
    The first order condition $p_0(y(\e)+\e)-p_1(y(\e)-\e)=0$ implies
    \[p_0'(y(\e)+\e)-p_1'(y(\e)-\e)={p_0(y(\e)+\e)}\left( y(\e)\left(\frac 1 {\sigma_1^2}-\frac 1 {\sigma_0^2}\right)-\e\left(\frac 1{\sigma_1^2}+\frac 1 {\sigma_0^2}\right)\right)\]

    However, \cref{eq:y_e_bound} implies that
    \[{p_0(y(\e)+\e)}\left( y(\e)\left(\frac 1 {\sigma_1^2}-\frac 1 {\sigma_0^2}\right)-\e\left(\frac 1{\sigma_1^2}+\frac 1 {\sigma_0^2}\right)\right)\leq {p_0(y(\e)+\e)}\cdot \frac{-2\e}{\sigma_0\sigma_1}<0\]
    
    Thus, the only solution to first \cref{eq:first_order_necessary_b} and \cref{eq:second_order_necessary_b} is $b(\e)$.

\subsubsection*{2) Comparing the risks of $(a(\e),b(\e))$, $\Rset$, $\emptyset$, $(a(\e),\infty)$, and $(-\infty,b(\e))$}

        First, we argue that $R^\e\left((a(\e),\infty)\right)>R^\e\left((a(\e),b(\e)\right)$: 
    \begin{equation}\label{eq:integrals_subtraction}
    \begin{aligned}
R^\e\big((a(\e),\infty)\big)-R^\e\big((a(\e),b(\e))\big)&=\int^{+\infty}_{b(\e)+\e} p_0(x)-p_1(x)dx-\int_{b(\e)-\e}^{b(\e)+\e} p_1(x)dx\\
&=\int_{b(\e)}^\infty p_0(x+\e)-p_1(x-\e)dx
\end{aligned}
    \end{equation}
   The same calculation that solves for $b(\e)$ in \cref{eq:b(e)_def} and $y(\e)$ in \cref{eq:y(e)_equal_means_gaussians} then shows that $p_0(x+\e)-p_1(x-\e)$ is strictly positive when $x>b(\e)$. 

    Additionally, $R^\e\big((a(\e),+\infty)\big)=R^\e\big((-\infty,b(\e))\big)$ because $a(\e)=-b(\e)$ and  $p_0$, $p_1$ are symmetric around zero. Consequently, writing out the integrals as in the first line of \cref{eq:integrals_subtraction}, results in $R^\e(\Rset)-R^\e\big( (-\infty,b(\e))\big) =R^\e\big( (a(\e), \infty)\big)-R^\e\big( (a(\e),b(\e))\big)$. Thus
    \[R^\e(\Rset)-R^\e\big((a(\e),b(\e))\big)=2 \Big(R^\e\big((a(\e),\infty)\big)-R^\e\big( (a(\e),b(\e) \big)\Big)>0\]
    and hence one can conclude that $R^\e\big( (a(\e),b(\e))\big)<R^\e(\Rset)$ and $R^\e\big( (a(\e),b(\e))\big)<R^\e\big((-\infty,b(\e))\big)$. Furthermore, the assumption $\lambda\geq 1/2$ implies $R^\e(\Rset)\leq R^\e(\emptyset)$ and consequently $R^\e(\emptyset)>R^\e \big( (a(\e),b(\e))\big)$. Accordingly, the set $( a(\e), b(\e))$ must be the adversarial Bayes classifier.

    \subsection{Proof of \cref{lemma:eps_in_interval}}\label{app:eps_in_interval}
     \Cref{lemma:transition_eta_0_1,lemma:max_degenerate_set} of \cref{app:degenerate_1d_eta_0_1} are used in the proof of \cref{lemma:eps_in_interval}.

    \begin{proof}[Proof of \cref{lemma:eps_in_interval}]
    Without loss of generality one can assume that $A$ open and regular by \cref{thm:adv_bayes_and_degenerate,lemma:closure_interior_adversarial_Bayes}. Concretely, the set $A$ is a union of open intervals due to \cref{lemma:closure_interior_adversarial_Bayes}.
    
    There is nothing to show if $\supp \PP=\Rset$.
    
We now consider smaller support--- for concreteness, we will assume that $\supp \PP=[\ell,\infty)$, the cases $\supp \PP=[\ell,r]$, $\supp \PP=(-\infty,r]$ are analogous.

Let 
\[i^*=\argmin_{a_i\geq \ell} a_i-\ell\]
\[j^*=\argmin_{b_i\geq \ell} b_i-\ell\]

We will now consider four cases:
\begin{enumerate}[label=\Roman*)]
    \item\label{it:a_i_low_ex} $\ell-a_{i^*}\leq \ell -b_{j^*}$ and $a_{i^*}>\ell+\e$; in which case  $A'=(a_{i^*},+\infty)\cap A$ is the desired adversarial Bayes classifier
    \item \label{it:a_i_low_ex_deg}
    $\ell-a_{i^*}\leq \ell -b_{j^*}$ and $a_{i^*}\leq \ell+\e$; in which case $A'=(-\infty, a_{i^*}]\cup A$ is the desired adversarial Bayes classifier
    \item \label{it:b_i_low_ex} $\ell-a_{i^*}> \ell -b_{j^*}$ and $b_{j^*}>\ell+\e$; in which case $A'=(-\infty, b_{j^*}) \cup A$ is the desired adversarial Bayes classifier
    \item \label{it:b_i_low_ex_deg} $\ell-a_{i^*}> \ell -b_{j^*}$ and $b_{j^*}\leq \ell +\e$; in which case $A'=( b_{j^*},\infty) \cap A$ is the desired adversarial Bayes classifier
\end{enumerate}
We will show \cref{it:a_i_low_ex,it:a_i_low_ex_deg}, the arguments for \cref{it:b_i_low_ex,it:b_i_low_ex_deg} is analogous.

\textbf{\cref{it:a_i_low_ex}:} First, we argue that $A$ and $A'$ are equivalent. \Cref{lemma:max_degenerate_set} implies that $(-\infty, \ell-\e]$ is a degenerate set for $A$. Next, there can be at most one point of $\partial A$ in $[\ell-\e,\ell]$ because $A$ is regular.  By the definition of $i_*$, if there is some point of $\partial A$ in $[\ell-\e,\ell]$, that point must be $b_{i^*-1}$. If $b_{i^*-1}\in [\ell-\e,\ell]$, then \cref{lemma:transition_eta_0_1} implies that $ [\ell-\e,b_{i^*-1}]$ is a degenerate set, and consequently $A$ is equivalent $A-((-\infty, \ell-\e] \cup[\ell-\e,b_{i^*-1}])$. As $(b_{i^*-1},a_{i^*}]$ is disjoint from $A$, it follows that $A-((-\infty, \ell-\e] \cup[\ell-\e,b_{i^*-1}])=A'$. On the other hand, if $b_{i^*-1}\not \in [\ell-\e,\ell]$, then $A$ is equivalent to $A-(-\infty,\ell-\e]$ and this set is disjoint from $(\ell-\e,a_{i*})$. Consequently, $A'=A-(\infty,\ell-\e]$ and thus $A$ and $A'$ are equivalent.

Next, we show that $A':= (a_{i^*},\infty) \cap A$ is a regular set. Because $A$ is regular, the point $a_{i^*}$ is more than $2\e$ from any other boundary point of $\partial A$. As $\partial ((a_{i^*},+\infty)\cap A)\subset\partial (-\infty, a_{i^*}) \cup \partial A=\partial A$, the point $a_i^*$ must 
be more than $2\e$ from any other boundary point of $ ( a_{i^*},+\infty)\cap A$. Therefore, $A'$ is regular.

The assumption $a_{i^*}>\ell+\e$ implies that $A'\subset (\ell+\e+\delta,+\infty)$ for some $\delta>0$ and consequently $A'$ can only have boundary points in $(\ell+\e, +\infty)=\interior (\supp \PP^{-\e})$.

Finally, $A'$ is open as it is the intersection of open sets.

\textbf{ \cref{it:a_i_low_ex_deg}:}
First, we argue that $A$ and $A'$ are equivalent. \Cref{lemma:max_degenerate_set,lemma:transition_eta_0_1} imply that the sets $[\ell-\e,a_{i^*}]$ and $(-\infty, \ell-\e]$ are degenerate sets for $A$.  Therefore, $A$ and $A'$ are equivalent up to degeneracy.

Next, the same argument as \cref{it:a_i_low_ex} shows that $A'= A\cup (-\infty, a_{i^*}]$ is a regular set:  $\partial (A\cup (-\infty,a_{i^*}])\subset \partial A\cup \partial (-\infty, a_{i^*}]=\partial A$. Thus the boundary points of $A'$ must be at least $2\e$ apart because $A$ is regular. Therefore, the classifier $A'$ is regular.

Further, the set $A'$ is open because $(-\infty, a_{i^*}]\cup (a_{i^*},b_{i^*})=(-\infty, b_{i^*})$ and consequently, $A'= (-\infty, b_{i^*})\cup A$.

Finally, to show that $\partial A' \subset \interior (\supp \PP^{-\e})$, we argue that $A'$ has no boundary points in $(-\infty,\ell+\e]=(\interior (\supp \PP^{-\e}))^C$. As $(-\infty, b_{i^*})\subset A'$, the set $A'$ has no boundary points in $(-\infty, b_{i^*}]$. However, the interval $(-\infty, b_{i^*}]$ contains $(-\infty,\ell+\e]$ as $b_{i^*}-a_{i^*}>2\e$ because $A$ is regular and $a_{i*}\geq \ell$.
\end{proof}

\subsection{\Cref{ex:non_uniqueness_all} details}\label{app:non_uniqueness_all}
    \Cref{thm:exists_regular} implies that when $\e<1/2$ the candidate solutions for the $a_i$, $b_i$ are $[-\e,\e]\cup \{-1-\e,-1+\e, 1-\e$,$ 1+\e\}$. However, \cref{lemma:eps_in_interval} implies that one only needs to consider points $a_i$, $b_i$ in $[-\e,\e]$ when identifying adversarial Bayes classifiers under equivalence up to degeneracy. However, $R^\e((y,\infty))<R^\e((-\infty,y))$ for any $y\in [-\e,\e]$ because $p_1(x)>p_0(x)$ for $x>\e$ while $p_1(x)<p_0(x)$ for any $x<-\e$. Thus, the candidate sets for the adversarial Bayes classifier are $\Rset$, $\emptyset$, and $(y,\infty)$ for any $y\in [-\e,\e]$. Next, any point $y\in[-\e,\e]$ achieves the same risk: $R^\e((y,\infty))=\e+\frac 14(1-\e)$ while $R^\e(\Rset)=R^\e(\emptyset)=1/2$. Thus $\emptyset,\Rset$ are adversarial Bayes classifiers when $\e\in[1/3,1/2)$ and $(y,\infty)$ is an adversarial Bayes classifier when $\e\leq 1/3$. \Cref{thm:adv_bayes_increasing_e} extends this result result to larger values of $\e$: specifically, this theorem implies that $(y,\infty)$ is an adversarial Bayes classifier for any $y \in [-\e,\e]$ iff $\e\leq 1/3$ while $\Rset$, $\emptyset$ are adversarial Bayes classifiers iff $\e\geq 1/3$.

\subsection{\Cref{ex:degenerate} details}\label{app:ex_degenerate_details}
It remains to compare the adversarial risks of all sets whose boundary is included in $\{-1/4\pm \e,1/4\pm\e \}$ for all $\e>0$. As points in the boundary of a regular set must be strictly more than $2\e$ apart, the boundary of a regular adversarial Bayes classifier can include at most one of $\{-\frac 14-\e, -\frac 14+\e\}$ and at most one of $\{\frac 14-\e, \frac 14+\e\}$. Let $\cS$ be the set of open sets with at most one boundary point in $\{-\frac 14-\e, -\frac 14+\e\}$, at most one boundary point in $\{\frac 14-\e, \frac 14+\e\}$, and no other boundary points.

Instead of explicitly computing the adversarial risk of each classifier in $\cS$, we will rule out most combinations by understanding properties of such sets, and then comparing to the adversarial risk of $\Rset$, for which $R^\e(\Rset)=1/10$ for all possible $\e$. We consider three separate cases:

\textbf{When $\e>1/4$:} 
If a set $A$ includes at least one endpoint in $\interior (\supp \PP^{-\e})$, then 
\[R^\e(A)\geq 2\e \inf_{x\in \supp \PP} p(x)\geq \frac {2\e}5>\frac 1 {10}=R^\e(\Rset)\]
The only two sets in $\cS$ that have no endpoints in $\interior (\supp \PP^{-\e})$ are $\Rset$ and $\emptyset$, but $R^\e(\emptyset)=9/10$. Thus if $\e>1/4$, then $\Rset$ is an adversarial Bayes classifier, and this classifier is unique up to degeneracy.

\textbf{When $1/8\leq \e\leq 1/4$:}
If either $1/4+\e,-1/4-\e$ are in the boundary of a set $A$, then 
\[R^\e(A)\geq \int_{y-\e}^{y+\e} p(x)dx= \frac 35\cdot 2\e\geq \frac 3 {20} >R^\e(\Rset).\]
(The value $y$ above is either $1/4+\e$ or $-1/4-\e$.)
Consequently, for these values of $\e$, only sets in $\cS$ with at most one endpoint in $\{-1/4+\e\}$ and at most one endpoint in $\{1/4-\e\}$ can be adversarial Bayes classifiers. 

Next, if a set $A$ in $\cS$ excludes either $(-\infty,-1/4)$ or $(1/4,\infty)$, then 
\[R^\e(A)\geq \int_S p_1(x)dx\geq \frac 3 5\cdot \frac 34>R^\e(\Rset).\]
(The set $S$ above represents either $(-\infty,-1/4)$ or $(1/4,\infty)$.)
As a result, such a set cannot be an adversarial Bayes classifier.

However, $\Rset$ and $(-\infty,-1/4+\e)\cup (1/4-\e,\infty)$ are the only two sets in $\cS$ that include $(-\infty,-1/4)\cup (1/4,\infty)$ with at most one endpoint in $\{-1/4+\e\}$ and at most one endpoint in $\{1/4-\e\}$. The set $(-\infty,-1/4+\e)\cup (1/4-\e,\infty)$ is not a regular set when $\e\geq 1/8$. Consequently, When $\e \in (1/8,1/4]$, the set $\Rset$ is an adversarial Bayes classifier, and this classifier is unique up to degeneracy.

\textbf{When $\e<1/8:$}
First, if $A$ excludes $(-\infty,-1/4-\e)$ or $(1/4+\e,+\infty)$, then 
\[R^\e(A)\geq \frac 35 \cdot (\frac 34-\e)\geq \frac 35 \cdot (\frac 34-\frac 18)=\frac 38>R^\e(\Rset).\] There are only five sets in $\cS$ that include both $(-\infty,-1/4-\e)$, $(1/4+\e,+\infty)$: $A_1=(-\infty, -1/4+\e)\cup (1/4-\e,\infty)$, $A_2=(-\infty, -1/4-\e)\cup ( 1/4-\e,\infty)$, $A_3=(-\infty, -1/4+\e)\cup (1/4+\e,\infty)$, $A_4=(-\infty,-1/4-\e)\cup (1/4+\e,\infty)$, and $A_5=\Rset$. All of these sets are regular when $\e<1/8$. One can compute:
\[R^\e(A_1)=\frac {4\e}5, R^\e(A_2)=R^\e(A_3)=\frac{8\e}5,R^\e(A_4)=\frac {12\e}5,\text{ and }R^\e(\Rset)=\frac 1 {10}\]
Of these five alternatives, the set $A_1$ has the strictly smallest risk when $\e\in (0,1/8)$. Consequently, when $\e\in (0,1/8)$, the set $A_1$ is the adversarial Bayes classifier and is unique up to degeneracy.

\subsection{Proof of \cref{prop:uniform_within_e}}\label{app:uniform_within_e}
\begin{proof}[Proof of \cref{prop:uniform_within_e}]
     Due to \cref{thm:adv_bayes_and_degenerate} and \cref{lemma:eps_in_interval}, any adversarial Bayes classifier is equivalent up to degeneracy to a regular adversarial Bayes classifier $A=\bigcup_{i=m}^M (a_i,b_i)$ for which all the finite $a_i$ and $b_i$ are contained in $\interior (\supp \PP^{-\e})$. 

    For every point $x$ in $\interior (\supp \PP^{-\e})$, the densities $p_0$ and $p_1$ are both continuous at $x-\e$ and $x+\e$. Consequently, the necessary conditions \cref{eq:first_order_necessary} 
    reduce to 



    on this set. If $a_i$ is more than $\e$ away 
    from a point $z$ satisfying $\eta(z)=1/2$, the continuity of $\eta$ implies that $\eta(a_i+\e), \eta(a_i-\e)$ are 
    either both strictly larger than $1/2$ or strictly smaller than $1/2$, and thus $a_i$ would not satisfy \cref{eq:uniform_necessary_reduce_2_a}. As a result, every $a_i$ must be within $\e$ of a solution to $\eta(z)=1/2$. An analogous argument shows that the same holds for solutions to \cref{eq:uniform_necessary_reduce_2_b}.
\end{proof}
\subsection{Proof of \cref{prop:eta_0_1}}\label{app:eta_0_1}
\begin{proof}[Proof of \cref{prop:eta_0_1}]
    Due to \cref{thm:adv_bayes_and_degenerate} and \cref{lemma:eps_in_interval}, any adversarial Bayes classifier is equivalent up to degeneracy to a regular adversarial Bayes classifier $A=\bigcup_{i=m}^M (a_i,b_i)$ for which all the finite $a_i$ and $b_i$ are contained in $\interior (\supp \PP^{-\e})$. Consequently, if there is some $a_i$ or $b_i$ in $\interior (\supp \PP^\e)$, then $\e<|\supp \PP|/2$.
    
    For contradiction, assume that $a_i$ is not within $\e$ of any point in $\partial\{\eta=1\}$. Then for some $r>0$, $\eta$ is either identically 1 or identically 0 on $(a_i-\e-r,a_i+\e+r)$ and thus $p_1=p\eta$ and $p_0=p(1-\eta)$ are continuous on this set. Furthermore, one of $p_1(a_i+\e)$ and $p_0(a_i-\e)$ is strictly positive and the other is zero. Consequently, $a_i$ cannot satisfy the necessary condition \cref{eq:first_order_necessary_a}, thus contradicting \cref{thm:exists_regular}. The argument for the $b_i$ is analogous
\end{proof}

\subsection{\Cref{ex:deg_eta_0_1_counterexample} details}\label{app:deg_eta_0_1_counterexample}
%
    It remains to compare the risks of all regular sets with endpoints in $\{-\frac 72\e, -\frac 52\e, -\frac 32 \e, -\frac 12 \e, +\frac 12 \e, +\frac 32 \e, +\frac 52 \e, +\frac 72\e\}$, and show that $\Rset$ is indeed an adversarial Bayes classifier. Rather than explicitly writing out all such sets and computing their adversarial risks, we show that one need not consider certain sets in $\cS$ because if they were adversarial Bayes classifiers, they would be equivalent up to degeneracy to other sets in $\cS$. 
    
    First, \cref{lemma:transition_eta_0_1} with $I=[-\frac 52 \e, +\frac 52 \e]$ implies that if $A$ is a regular adversarial Bayes classifier and $y\in \{-\frac 72 \e,-\frac 52 \e,-\frac 32 \e\}$ is in $\partial A$, then $[-\frac 72 \e,y]$ is a degenerate set. This observation rules out sets with endpoints in $\{-\frac 72\e,-\frac 52 \e, -\frac 32 \e\}$ when identifying all possible adversarial Bayes classifiers under equivalence up to degeneracy. 
    Similarly, \cref{lemma:transition_eta_0_1} also implies that there is no need to consider $\{+\frac 32 \e,+\frac 52 \e,+\frac 72 \e\}$ as possible values of the $a_i$s or $b_i$s. Thus it remains to compare the risks of regular sets whose boundary is contained in $\{-\frac 12 \e,-\frac 12 \e\}$. As points in the boundary of a regular set are at least $2\e$ apart, one can rule out sets with more than one boundary point in $\{-\frac 12 \e,+\frac 12 \e\}$.

    Consequently, it remains to compare the adversarial risks of six sets: $\Rset$, $\emptyset$, $(-\frac 12 \e,+\infty)$, $(-\infty, -\frac 12 \e)$, $(+\frac 12 \e, +\infty)$, and $(-\infty, +\frac 12 \e)$. The adversarial risks of these sets are: 
    \[R^\e(\Rset)=\frac {14} {25}\quad\quad R^\e(\emptyset)=\frac {11}{25}\]
    \[R^\e\left((-\frac 12 \e,+\infty)\right)=R^\e\left((+\frac 12 \e,+\infty)\right)=\frac {21} {25}\]
    \[R^\e\left((-\infty,-\frac 12 \e)\right)=R^\e\left((-\infty, +\frac 12 \e)\right)=\frac 9{25}\]
    Therefore, the set $(-\infty,-\frac 12 \e)$ is an adversarial Bayes classifier.

\end{document}